\documentclass[a4paper, 10pt]{article}
\usepackage[a4paper, top=1.3in, bottom=1.30in, left=0.90in, right=0.90in]{geometry}
\usepackage{amssymb}
\usepackage{amsmath}
\usepackage{tgtermes}
\usepackage{setspace}
\usepackage{amsthm} 
\theoremstyle{plain}
\newtheorem{theorem}{Theorem}
\newtheorem{corollary}{Corollary}
\newtheorem{definition}{Definition}
\newtheorem{assumption}{Assumption}
\newtheorem{proposition}{Proposition}
\newtheorem{lemma}{Lemma}
\newtheorem{remark}{Remark}
\usepackage{algorithm}
\usepackage{algorithmic}

\usepackage{cancel} 
\usepackage{graphicx}
\usepackage{listings}
\usepackage{mathrsfs}
\usepackage{titlesec}
\usepackage{authblk}
\usepackage{natbib}
\usepackage{xcolor}
\usepackage{hyperref}

\newcommand{\zedt}{\bar{\zeta}_{t}^{\lambda,n}}
\newcommand{\zeds}{\bar{\zeta}_{s}^{\lambda,n}}

\makeatletter\AtBeginDocument{\check@mathfonts}\makeatother

\makeatletter
\renewcommand\subsubsection{\@startsection{subsubsection}{3}{\z@}%
  {-3.25ex\@plus -1ex \@minus -.2ex}%
  {1.5ex \@plus .2ex}%
  {\normalfont\small\itshape\mdseries}}
\makeatother
\titleformat{\section}[block]
  {\normalfont\normalsize\bfseries}{\thesection.\ \ }{0em}{}
\titleformat{\subsection}[block]
  {\normalfont\normalsize\itshape}{\thesubsection.\ \ }{0em}{}
\titleformat{\subsubsubsection}[block]
  {\normalfont\small\itshape\mdseries}{\thesubsubsubsection.}{0em}{}
\title{\Large The Tamed Subgradient Unadjusted Langevin Algorithm beyond Convexity}
\author[1,2,3]{Iosif Lytras}
\author[1]{Nikolaos Makras}
\author[1,2,3]{Sotirios Sabanis}
\affil[1]{School of Mathematics, University of Edinburgh, UK}
\affil[2]{National Technical University of Athens, Greece}
\affil[3]{Athena/Archimedes Research Centre, Greece}
\date{\small \today}

\newcommand{\dpar}{\bar{d}}                      
\newcommand{\dm}{d_{\mathrm{m}}}                 
\newcommand{\dhead}{d_{\mathrm{h}}}              
\newcommand{\dff}{d_{\mathrm{f}}}                
\newcommand{\Nrm}{\mathrm{N}}                    
\newcommand{\smx}{\operatorname{sm}}             
\newcommand{\CE}{\operatorname{CE}}              
\newcommand{\Tcl}{\mathcal{T}}                   
\newcommand{\Idm}{\mathrm{I}_{\dpar}}            
\newcommand{\Fn}[1]{\lVert #1\rVert_{\mathrm{F}}}
\newcommand{\opn}[1]{\lVert #1\rVert}
\newcommand{\Var}{\operatorname{Var}}

\usepackage{capt-of} 
\newenvironment{keywords}{\begin{quote}\small\textbf{Keywords:}}{\end{quote}} 
\begin{document}

\maketitle
{\renewcommand{\thefootnote}{}\footnotetext{This work was supported by Innovate UK (UKRI) through the UK--South Korea Collaborative R\&D (Innovate UK--KIAT) grant 10081810 and by the Centre for Investing Innovation at the University of Edinburgh.}}

\begin{abstract}%
We study the problem of sampling from target distributions whose potentials are simultaneously non-smooth, subject to superlinear gradient growth, and non-convex. We introduce the Subgradient Tamed Unadjusted Langevin Algorithm (SG-TULA), a discretisation of the Langevin diffusion that operates directly on subgradients, without relying on computationally demanding smoothing procedures. To handle the superlinear regime, taming techniques are employed to produce a stable, explicit scheme. We derive non-asymptotic convergence bounds in Wasserstein-2 distance, with all constants tracked explicitly in terms of dimension and inverse temperature, improving upon the currently known rates for subgradient-based Langevin algorithms. We further provide excess risk estimates for the associated optimisation problem. We verify the assumptions, with explicit constants, for the regularized pretraining potential of a LLM in the GPT-2 lineage and the boosted coordinate-wise variant of SG-TULA pretrains the former competitively against  finetuned AdamW and Muon, for
which no comparable non-asymptotic guarantees are presently available.
\end{abstract}


\section{Introduction, Purpose and Scope}
The hypothesis under which optimization algorithms are usually analyzed, namely a
globally Lipschitz gradient together with convexity or strong convexity, are
rarely met by the problems they are applied to. Consider the case of a transformer which is the
composition of layers rendering the gradient polynomial in the weights, so that
it is not Lipschitz, is rendered discontinuous by ReLU, and the loss is not
convex. Each failure has its own practical remedy, clipping for the first and
subgradients for the second, and each has been analyzed in isolation. We seek non-asymptotic guarantees when non-smoothness, superlinear gradient growth, and non-convexity occur together.

For a potential $u:\mathbb{R}^{d}\to\mathbb{R}$, under mild assumptions, we
consider the problem
\begin{equation}\label{eq:opt}
  \theta^{\star}\in\operatorname*{arg\,min}_{\theta\in\mathbb{R}^{d}}u(\theta),
\end{equation}
which we approach by sampling from the Gibbs measure attached to $u$,
$\pi_{\beta}(\mathrm{d}x)\propto e^{-\beta u(x)}\,\mathrm{d}x$ on
$\mathbb{R}^{d}$. The connection between the two is the concentration of
$\pi_{\beta}$ around the global minimizers of $u$ for large $\beta$, a fact
going back to \citet{Hwang1980}. In that sense, an algorithm which samples
accurately from $\pi_{\beta}$ also solves the associated
expected excess risk problem for $u$. We therefore bound the excess risk
$\mathbb{E}[u(\theta_{n}^{\lambda})]-u(\theta^{\star})$ of the iterates
$(\theta_{n}^{\lambda})_{n\in\mathbb{N}}$ by the sampling error and the temperature gap quantifying the bias attributable to $\beta$.
 
The algorithm producing these iterates, namely the subgradient tamed unadjusted
Langevin algorithm (SG-TULA), is an explicit Euler--Maruyama discretisation of
the associated Langevin diffusion which operates on a measurable selection
$h\in\partial u$ and stabilizes the drift by taming,
\[
  \theta_{n+1}^{\lambda}
  =\theta_n^{\lambda}-\lambda h_{\lambda}(\theta_n^{\lambda})
   +\sqrt{2\lambda\beta^{-1}}\,\xi_{n+1},
  \qquad
  h_{\lambda}(x)=\frac{h(x)-\mu x}{1+\lambda^{1/2}\lvert h(x)-\mu x\rvert}+\mu x ,
\]
with $(\xi_n)_{n\ge1}$ i.i.d.\ standard Gaussian, alongside which a coordinate-wise variant is
analysed. Taming acts on the subgradient after it has been computed, so the cost of evaluating 
$h$ is unaffected and the overhead is significantly smaller than one matrix multiplication. Written this way the scheme is visibly an alternative method to gradient clipping, the latter being introduced for recurrent networks in \cite{Pascanu_RNN_2013} and is now the standard
response to exploding gradients in deep learning. However clipping
requires a threshold to be fixed in advance and, once active, displaces the drift by
an amount that does not diminish as the stepsize is refined. The tamed drift
requires no threshold, rescales continuously with the subgradient magnitude,
and satisfies $h_\lambda\to h$ as $\lambda\to0$, so that the scheme remains a
discretisation of the intended dynamics.
 
Accordingly, we assume that $u$ is semi-convex \hyperlink{A1}{A1}, strongly
convex outside a compact set \hyperlink{A2}{A2}, and every subgradient
satisfies $\lvert h(x)\rvert\le m+K\lvert x\rvert^p$ for some $p\ge1$
\hyperlink{A3}{A3}. Semi-convexity is a weak departure from convexity, in that subdifferential calculus is available and permits gradient discontinuities on a set of
measure zero, which is the irregularity encountered in the analysis of neural networks. Additionally, it provides the one-sided Lipschitz bound on which the discretisation estimates rest. Polynomial growth is what taming controls by avoiding exploding gradients. Strong convexity at infinity yields the key property of dissipativity, which
together with semi-convexity gives the contraction of the diffusion. It also
highlights what the method can address, since a dissipative drift forces
$\pi_\beta$ to have at most Gaussian tails. Potentials falling outside this regime must be
regularized by higher order terms, namely weight decay in the context of ML training.
 
Under these assumptions alone, and with no further regularity imposed on $u$,
our main sampling result takes the following form:
 
\medskip
\noindent\textbf{Theorem~\ref{theorem_main} (informal).}
\emph{For every stepsize $\lambda<1/(4\mu)$ and every $n\in\mathbb{N}$,}
\[
  W_2\bigl(\pi_{\beta},\mathcal{L}(\theta_n^{\lambda})\bigr)
  \;\le\;
  \underbrace{2\,C_{E}\,\lambda^{1/4}}_{\textnormal{discretisation}}
  \;+\;
  \underbrace{e\,C_{D}\,e^{-n\lambda/T_0}}_{\textnormal{contraction}} .
\]
\medskip
 
All three constants are given in closed form. The dimension enters only polynomially through the ratio $d/\beta$, with $C_{E}=\mathcal{O}\bigl((d/\beta)^{p/2}\bigr)$ and
$C_{D}=\mathcal{O}\bigl((d/\beta)^{1/2}\bigr)$, while the restart horizon
$T_0$, and thus the rate of the contraction term, is independent of the dimension. The dependence on the inverse temperature is exponential in
general, since it is inherited from
the log-Sobolev constant of $\pi_\beta$. This largely reflects
the proof strategy, which routes the optimization guarantee through a sampling
one. In particular, the exponential factor is the cost of mixing between the modes of
$\pi_\beta$.
In the optimization regime $\pi_\beta$ has concentrated on the global
minimizers since $\beta$ is large. There the governing
constant is a local one and no barrier needs to be crossed, as discussed in
Remark~\ref{expo_lsi}. Formally, we record in
Remark~\ref{remark_beta_poly} the additional Polyak--{\L}ojasiewicz type condition, under which this dependence becomes polynomial. Given a tolerance $\epsilon>0$, Corollary~\ref{corollary1}
provides the complexity of the algorithm, the stepsize being of order $\lambda=\mathcal{O}(\epsilon^4)$
and $n=\mathcal{O}(\epsilon^{-4}\log(1/\epsilon))$ for iterations respectively.
 
The sampling estimate transfers to the associated optimization problem. Writing
$\theta^{\star}$ for a global minimiser of $u$, one obtains the following.

\medskip
\noindent\textbf{Theorem~\ref{theorem2} (informal).}
\emph{Under the latter conditions and for every $\beta\geq4/\mu$,}
\[
  \mathbb{E}[u(\theta_n^{\lambda})]-u(\theta^{\star})
  \;\le\;
  \underbrace{C_{\mathcal{T}_1}W_2\bigl(\pi_\beta,\mathcal{L}(\theta_n^{\lambda})\bigr)}_{\textnormal{sampling}}
  \;+\;
  \underbrace{C_{\mathcal{T}_2}}_{\textnormal{temperature gap}} .
\]
\noindent Both constants are explicit, with
$C_{\mathcal{T}_1}=\mathcal{O}\bigl((d/\beta)^{p/2}\bigr)$ and
$C_{\mathcal{T}_2}=\mathcal{O}\bigl((d/\beta)\log(d\beta)\bigr)$.
\medskip
 
The complexity $\mathcal{O}(\epsilon^{-4})$ is inherited entirely from
Lemma~\ref{lemma_err_W2}. Since the contraction term carries no power of the
stepsize, the order of the strong approximation estimate on the finite restart
horizon is the order of the algorithm. Regarding this approximation rate, two lines of work on numerical approximation of SDEs are of particular relevance. The first, developed in
\cite{DareiotisGerencser2020, DareiotisGerencserLe2023} and the references
therein, given arbitrary
$\delta>0$,  obtains order $\mathcal{O}(\lambda^{1/2-\delta})$ for the Euler--Maruyama scheme without imposing any continuity on
the drift, which is almost optimal in view of the corresponding lower error
bounds. However, these results are obtained for globally bounded drift coefficients.
Growth is addressed by a second line of work. \cite{MullerGronbachYaroslavtseva2020} establish order
$\mathcal{O}(\lambda^{1/2})$ for piecewise Lipschitz drift in one dimension. \cite{MullerGronbachSabanisYaroslavtseva2026} extended this one-dimensional result to coefficients of
superlinear growth and finitely many drift discontinuities; see also \cite{MullerGronbachRauhoggerYaroslavtseva2026} for the
multidimensional case under piecewise Lipschitz drift. By contrast, our
setting is multidimensional, superlinear and non-convex, and the set of
discontinuities is permitted to be infinite, subject only to being of
Lebesgue measure zero. For comparison, the benchmark under regularity is $\mathcal{O}(\lambda^{1/2})$
for Lipschitz drift and $\mathcal{O}(\lambda)$ when the Hessian is also Lipschitz. The latter essentially corresponds to the Milstein rate, since under a constant diffusion coefficient, the Milstein and Euler-Maruyama schemes coincide.
 
These developments are mirrored in the sampling literature. For $p=1$, our
assumptions reduce to those of \cite{SGULA}, where SG-ULA is analyzed under
semi-convexity, strong convexity at infinity and at most
linear subgradient growth. We recover this case, improving the
$W_2$ order from $\mathcal{O}(\lambda^{1/8})$ to $\mathcal{O}(\lambda^{1/4})$.
Table~\ref{tab:comparison} places the result among the alternatives. To the
best of our knowledge, no other existing result, attains $\mathcal{O}(\lambda^{1/4})$ in $W_2$ while allowing non-smoothness, superlinear growth
and non-convexity simultaneously. Those achieving comparable or better rates assume additional smoothness or strong convexity.
\begin{table}[tbp]
\centering
{\small
\begin{tabular}{lcccc}
\hline
\textbf{Work} & \textbf{Smooth} & \textbf{Growth} & \textbf{Convexity} & \textbf{$W_2$ rate}\\
\hline
\cite{DurmusMoulines}                       & yes            & linear      & strong       & $\lambda^{1/2}$\\
\cite{TULA}                                 & yes            & superlinear & strong       & $\lambda^{1/2}$\\
\cite{Lovas_TUSLA_2023}                     & yes            & superlinear & non-convex   & $\lambda^{1/4}$\\
\cite{DurmusMoulinesPereyra}                & composite      & linear      & convex  & $\lambda^{1/2}$ (TV)\\
\cite{Lehec_LMC_2022}                       & no (subgrad.)  & linear      & convex       & $\lambda^{1/4}$\\
\cite{Fruehwirth_Ergodicity_2024}           & no (subgrad.)  & linear      & strong       & $\lambda^{1/2}$\\
\cite{SGULA}                                & no (subgrad.)  & linear      & at infinity  & $\lambda^{1/8}$\\
\hline
\textbf{SG-TULA (this work)}                & no (subgrad.)  & superlinear & at infinity  & $\boldsymbol{\lambda^{1/4}}$\\
\hline
\end{tabular}}
\caption{Hypotheses and $W_2$ rates for related Langevin schemes.}
\label{tab:comparison}
\end{table}
 
The algorithm is also competitive in practice. In
Section~\ref{sec:nc-exp}, we pretrain \emph{nanochat} \cite{Karpathy_nanochat_2025} using a boosted variant
whose taming map includes the multiplicative
factor (boosting) of \cite{Lim_Sabanis_2024}. This perturbation is uniformly
bounded above and below, so the growth, closeness and dissipativity estimates
of Appendix~\ref{sec_taming} hold with altered constants and the analysis
retains its form. At depth $12$ it attains the best validation
bits-per-byte and CORE score among the three optimizers tested, ahead of finetuned
AdamW and Muon. Despite the lack of scaling laws, our algorithm remains competitive at depth $24$, performing on par with AdamW within evaluation noise and remaining close to Muon.

In Appendix~\ref{sec:nanochat} we carry out that
verification for \emph{nanochat}, a compact transformer language model in the
GPT-2 lineage, and obtain the curvature bound
\[
  \nabla^2_{\theta}u(\theta)\ \succeq\ -\,c^{\star}\bigl(1+\lvert\theta\rvert\bigr)^{12\mathsf{L}+2}\,\Idm ,
\]
with $c^\star$ explicit in the depth $\mathsf{L}$, the width and the vocabulary
size, from which \hyperlink{A1}{A1}--\hyperlink{A3}{A3} follow for the
regularized potential. The exponent $12\mathsf{L}+2$ reflects the fact that superlinear growth in deep architectures is a consequence of depth and
not an artifact of the analysis. Our assumptions also cover the sparse phase-retrieval model studied in Section~\ref{sec:pr}, where a non-convex quartic energy gives cubic gradient growth ($p=3$), and an $\ell_1$ penalty introduces gradient discontinuities on the coordinate hyperplanes.

\section{Literature Review}

Although this is a broad and active research area, including Metropolis-adjusted variants, underdamped Langevin dynamics, L\'evy driven dynamics, interacting particles methods, and algorithms designed to handle constraints, we focus on approaches that extend the Unadjusted Langevin Algorithm (ULA) beyond the standard assumptions of linear growth, strong convexity, and smoothness. 

\subsection{Langevin-Based Algorithms}

A systematic study of the ergodicity of the Unadjusted Langevin Algorithm
(ULA) was undertaken in \cite{Roberts_Tweedie_1996}, which also formally introduced gradient clipping under the MALTA (Metropolis-Adjusted Langevin Truncated Algorithm) framework, however, the latter algorithm was not further investigated. Since, theoretical analyses have primarily focused on smooth, convex potentials $u$. Under these assumptions, i.e., global Lipschitz continuity of the gradient and strong convexity, ULA is known to enjoy geometric ergodicity and admits non-asymptotic convergence guarantees in both Wasserstein and total variation distances. Such results were first established by \cite{Dalalyan} and \cite{DurmusMoulinesNonasymptotic}, with refined bounds later provided in \cite{DurmusMoulines}. In the latter work, the authors obtain a global error bound in Wasserstein distance, of order $\mathcal{O}(\lambda^{1/2})$ under the standard smooth and strongly convex setting, which can be improved to $\mathcal{O}(\lambda)$ when additional smoothness of the Hessian is assumed. While this improvement formally corresponds to employing the Milstein discretization scheme, the connection may not be immediately evident to the reader since the diffusion coefficient is constant in this case. More broadly, many results that invoke higher order smoothness assumptions, often phrased in terms of Hessian regularity or beyond, implicitly inherit improved convergence rates that, from the SDE numerical analysis perspective, stem from analogous higher order discretization effects.

Considerable effort has since gone into weakening the convexity assumption.
A prominent line of work trades convexity-type conditions for a functional inequality
on the target measure, while retaining global smoothness of the corresponding gradient. Beginning
with \citet{vemp2019}, a substantial literature provides estimates in KL and R\'enyi
divergences under log-Sobolev and Poincar\'e inequalities
\citep[e.g.,][]{erd2021, erd2022, ChewiPoincareLogSob}. In \cite{MousaviHosseini_2023} the authors work under weak Poincar\'e inequalities to cover heavy-tailed, non-convex targets, including Cauchy-type distributions for which a Poincar\'e inequality fails. Under $s$-H\"older regularity of the gradient, $s\in(0,1]$, they obtain
non-asymptotic R\'enyi divergence guarantees for ULA, providing both iteration
and time-horizon complexities. Their results imply a discretization rate of
$\mathcal{O}(\lambda^{s/2})$ in TV distance.

\subsection{Algorithms for Superlinear Drift}

It is well established that ULA exhibits numerical instability when the target distribution has a superlinearly growing potential, with moments of the discrete approximations diverging in finite time \citep{Hutzenthaler_divergence_2010}. In the deterministic gradient case, SGLD reduces to ULA, and as such can be seen as a standard Euler--Maruyama discretisation of the Langevin diffusion, which inherits the same limitations due to the growth of the drift coefficient. This instability manifests in practical application whenever gradients accumulate too rapidly, as is well documented in the context of recurrent neural networks \citep{Pascanu_RNN_2013}, motivating the development of stabilization methods.

One idea for handling superlinear coefficients, is the taming technique, originally developed for stabilizing explicit Euler schemes of SDEs with exploding coefficients in \cite{HutzenthalerJentzenKloeden} and subsequently in \cite{Sabanis}. The Tamed Unadjusted Langevin Algorithm (TULA), introduced in \cite{TULA}, applies this idea to ULA. There, $u$ is assumed to be continuously differentiable and strongly convex but allowed $\nabla u$ to grow polynomially. Under these assumptions, they proved convergence of TULA in Wasserstein-2 distance to the target distribution with order $\mathcal{O}(\lambda^{1/2})$. Building on TULA, a line of research has adapted taming to various stochastic and non-convex gradient settings.

In the context of optimization, the stochastic gradient setting is understood as mini-batching, whereby only noisy estimates of the full gradient are accessible at each iteration, we refer to \cite{Dalalyan_Karagulyan_2019} for an early treatment, and to \cite{Barkhagen_SGLD_2021,Chau_SGLD_2021}  for an extension to the non-i.i.d. case. Within this regime, \cite{Lovas_TUSLA_2023} introduced TUSLA for ANN, extending taming to non-convex potentials for which the original density fails to satisfy a dissipativity condition, a higher-order regularization term is added to restore this property, yielding a convergence rate of order $\mathcal{O}(\lambda^{1/4})$ in Wasserstein-2 distance. TheoPouLA \citep{Lim_Sabanis_2024} further develops this approach with coordinate-wise taming and a boosting factor, where the latter adaptively scales gradients upward to counteract vanishing gradients. It retains the former rate and still requires higher-order regularization for dissipativity, yet performs competitively against established adaptive gradient methods on deep learning tasks. Using a different taming function, \cite{Neufeld_Langevin_2025} obtain a $O(\lambda^{1/2})$ rate in Wasserstein-2 distance under a convexity-at-infinity type condition. However, their constants scale exponentially with dimension and their analysis requires stronger smoothness assumptions on the potential's Hessian.

In recent work, \cite{YandWang2025} analyze the Unadjusted Langevin Algorithm (ULA) and a projected variant (PLMC) under the log-Sobolev assumption, extending beyond the convex regime. They assume global Lipschitz continuity of $\nabla u$ for ULA, and for PLMC they relax this to a polynomial local Lipschitz condition. In principle PLMC can target potentials with super linear growth, however their analysis imposes a growth bound on $\nabla(\Delta u)$ which significantly limits the class of superlinear potentials that can be targeted, see Remark~A2 in \cite{Li_MeanSquare_2022}. Their error bounds in $W_2$ achieve order $\mathcal{O}(\lambda)$, matching the best known rate of the Milstein scheme in the fully smooth and strongly convex case.

Further results address the superlinear and non-convex regime in the
deterministic gradient case. Under a functional inequality for the target
distribution together with a locally Lipschitz gradient, which permits superlinear
growth, estimates in KL divergence were obtained in
\cite{Lytras_Sabanis_2025, Lytras_Mertikopoulos_2025} and later sharpened in
\cite{Lytras_kTULA_2025, Lytras_Ntousis_2026} by exploiting additional local
regularity of the Hessian. This last requirement was in turn removed through tamed
randomized midpoint schemes, which attain $\tilde{\mathcal{O}}(\epsilon^{-1})$
complexity in Wasserstein and total variation distances
\citep{Langevin_Randomization_2025, Lytras_Ntousis_2026}.
\subsection{Algorithms for Non-Smooth Potentials}
It is worth distinguishing here between two forms of irregularity.
In \cite{Sabanis_Zhang_CVaR} the stochastic gradient is discontinuous in
$\theta$, as occurs in quantile and CVaR estimation, where indicators of
the form $\mathbf{1}_{\{x<\theta\}}$ appear, with $x$ denoting a data observation. The analysis proceeds under a
Lipschitz in average condition (Assumption~3 therein), together with a decomposition of the
stochastic gradient into a Lipschitz part and a part bounded in $\theta$.
Averaging against the law of the data then restores regularity, and the drift
$h=\nabla U$ of the associated Langevin SDE is globally Lipschitz. The irregularity
treated in the present work is of a different nature. The potential $u$ is
itself non-smooth, no averaging is available to recover regularity, and the
drift is a measurable selection $h\in\partial u$ which is genuinely
discontinuous. We identify two broad strategies that extend Langevin algorithms to the latter non-smooth setting.

\subsubsection{Smoothing}

\emph{Proximal Maps.} One approach to handling non-smoothness is to modify the target potential to recover regularity. A standard technique is Moreau–Yosida regularization, which replaces the potential $u$ with its envelope $u_\gamma$, a smoothed approximation with a Lipschitz continuous gradient defined via the proximal operator. This leads to the Moreau–Yosida Unadjusted Langevin Algorithm (MYULA) for which non-asymptotic guarantees were first studied in \cite{BrosseDurmusMoulinesPereyra} for potentials convex on compact bodies, proving convergence in total variation (TV) distance and establishing an $\mathcal{O}(\lambda^{1/2})$ rate in Wasserstein-2 distance, provided the smooth component is strongly convex. In \cite{DurmusMoulinesPereyra}, the authors provided non-asymptotic TV guarantees for potentials that decompose into smooth and non-smooth convex components, achieving a rate of $\mathcal{O}(\lambda^{1/2})$ up to logarithmic factors.

Despite these theoretical strengths, evaluating the Moreau–Yosida envelope typically requires a costly proximal solve at each iteration, since closed form solutions are uncommon. Furthermore, these methods generally necessitate a composite structure, where the potential is a sum of smooth and non-smooth functions, and assume efficiently computable proximal maps. This requirement limits their applicability to potentials that lack such a decomposition. To mitigate the cost of proximal steps, \cite{EhrhardtKugerSchonlieb} proposed using inexact proximal mappings, though their guarantees are restricted to convex targets.

More recently, \cite{Habring_Diffusion_2025} introduced Diffusion at Absolute Zero (DAZ), which extends MYULA by incorporating an annealing schedule for the regularization parameter. By assuming the potential is semi-convex and convex outside a ball, DAZ relaxes the global convexity requirement of MYULA and achieves an $\mathcal{O}(\lambda^{1/2})$ rate in TV distance. Under similar composite and convexity assumptions, \cite{Renaud_Stability_2025} established an $\mathcal{O}(\lambda^{1/(2p)})$ rate in Wasserstein-$p$ distances for the Proximal Stochastic Gradient Langevin Algorithm (PSGLA). Very recently, \cite{Luu_DCLA_2026} replaced the weak convexity requirement of \cite{Renaud_Stability_2025} to a difference-of-convex structure on the non-smooth term, extending the non-asymptotic guarantees in Wasserstein-$p$ distances with the same convergence rate. \medskip

\emph{Gaussian Smoothing.} \cite{GaussSmoothing} introduce the Perturbed Langevin Monte Carlo (P-LMC) algorithm, which addresses the non-smoothness of $u$ by injecting noise into the argument of the gradient. They consider a composite potential of the form $\bar{u}=u+g$, where $u$ is convex with an a-H\"older continuous gradient for $a\in[0,1]$, and $g$ is smooth and strongly convex. In Wasserstein-2 distance they recover the optimal convergence rate when $a=1$, and obtain a rate of order $\mathcal{O}(\lambda^{1/4})$ in the non-smooth case $a=0$ (Remark 3.5 therein). This approach was later generalized to $p$-Gaussian smoothing in \cite{GaussWeakly}, addressing $a-$mixture weakly smooth potentials under the log-Sobolev inequality. In this setting, the method achieves a convergence rate in Wasserstein-2 distance of order $\mathcal{O}(\lambda^{a/2})$ for $a\in(0,1]$.

In a different approach, \cite{Laumont_2022} assume a composite structure similar to the proximal ULA literature and smooth the irregular term via Gaussian convolution. The gradient of the smoothed term is approximated through Tweedie's formula, relating it to a trained denoiser, which leads to the Plug-and-Play ULA (PnP-ULA). The drift is thus determined by a learned network rather than by the potential itself, and their assumptions are correspondingly stated in terms of the denoiser. They establish a bias of order $\mathcal{O}(\lambda^{1/2})$ in $V$-total variation between the stationary distribution of the scheme and that of the smoothed target, rendering their result not directly comparable to the rates in the present discussion. Moreover, in the absence of any dissipativity yielding condition, their constants depend on the radius of the convex projection set required to recover contraction.

The anchored Langevin algorithm of \cite{Gurbuzbalaban_Anchored_2025}
addresses a different gap, primarily targeting heavy-tailed distributions.
The authors introduce a reference potential and define a modified SDE whose
drift and diffusion are constructed so as to preserve the original invariant
measure. For non-smooth targets they consider a composite structure, a smooth
and strongly convex component together with an irregular term that is weakly
convex and Lipschitz, and hence admits bounded gradients almost everywhere,
the reference potential being obtained by replacing the latter with a
Gaussian-smoothed approximation. They establish non-asymptotic convergence
guarantees in Wasserstein-2 distance. Their final estimate (Theorem~30
therein), however, carries a third term arising from the Monte Carlo
approximation of the smoothed reference potential, which is increasing in
both the stepsize and the number of iterations and is thus not uniform in
time. The allowable stepsize is moreover restricted to $\mathcal{O}(d^{-2})$.

\subsubsection{Direct Subgradient Methods}

\citet{DurmusSub} introduced a subgradient variant of the stochastic unadjusted Langevin algorithm (SSGLD) for sampling from non-smooth, log-concave target distributions. Under the assumption that the potential is Lipschitz and convex, they derive non-asymptotic bounds in KL divergence at rate $\mathcal{O}(\lambda)$. However, their convergence guarantees apply only to the weighted average of the iterates' distributions, leaving last-iterate convergence unaddressed. In a different direction, \cite{Lehec_LMC_2022} studies ULA under a subgradient formulation for globally Lipschitz convex potentials and establishes an $\mathcal{O}(\lambda^{1/4})$ rate in Wasserstein-2 distance under a log-Sobolev inequality, with the same rate extending to unconstrained settings when the subgradient has at most linear growth. \cite{Fruehwirth_Ergodicity_2024} consider an explicit subgradient Langevin discretization under strong convexity and linear growth of the subgradient. By separately establishing geometric convergence of the iterates to the scheme's stationary (discrete) measure and bounding the stationary bias $W_2(\pi_\lambda, \pi)$ via an energy-entropy functional argument, they recover a rate of order $\mathcal{O}(\lambda^{1/2})$ in $W_2$, matching the smooth case rate.

\section{Notation}
We introduce some basic notation. For $u,v\in\mathbb{R}^d$, define the scalar product $\langle u,v\rangle=\sum_{i=1}^d u_i v_i$ and the Euclidean norm $|u|={\langle u,u\rangle}^{1/2}$. For all continuously differentiable functions $f:\mathbb{R}^d\to\mathbb{R}$, $\nabla f$ denotes the gradient. For symmetric matrices $A$ and $B$ we write $A\succeq B$ when $A-B$ is positive semi-definite, and $\mathrm{I}_d$ denotes the identity on $\mathbb{R}^d$. For semi-convex $u$ we write $\nabla^2u$ for the Alexandrov Hessian, which exists Lebesgue-almost everywhere, and any matrix inequality involving it is understood to hold almost everywhere. Since the singular part of the distributional Hessian of a convex function is non-negative, such a bound holds distributionally as well, and the two notions agree with the classical one when $u\in C^2$. The integer part of a real number $x$ is denoted by $\lfloor x \rfloor$. For an $\mathbb{R}^d$-valued random variable $Z$, its law on $\mathcal{B}(\mathbb{R}^d)$, i.e., the Borel sigma-algebra of $\mathbb{R}^d$, is denoted by $\mathcal{L}(Z)$. We denote by $\mathcal{P}(\mathbb{R}^d)$ the set of probability measures on $\mathcal{B}(\mathbb{R}^d)$ and for any $p\in\mathbb{N}$, $\mathcal{P}_p(\mathbb{R}^d)=\{\pi\in\mathcal{P}(\mathbb{R}^d):\int_{\mathbb{R}^d}|x|^p d\pi(x)<\infty\}$ denotes the set of all probability measures over $\mathcal{B}(\mathbb{R}^d)$ with finite $p$-th moment. For any two Borel probability measures $\mu$ and $\nu$, we define the Wasserstein distance of order $p\geq 1$ as
\begin{align*}
W_p(\mu,\nu)=\left(\inf_{\zeta\in\Pi(\mu,\nu)}\int_{\mathbb{R}^d\times\mathbb{R}^d}|x-y|^p d\zeta(x,y)\right)^{1/p},\end{align*}
where $\Pi(\mu,\nu)$ is the set of all transference plans of $\mu$ and $\nu$. Moreover, for any $\mu,\ \nu\in\mathcal{P}_p(\mathbb{R}^d) $, there exists a transference plan $\zeta^*\in\Pi(\mu,\nu)$ such that for any coupling $(X,Y)$ distributed according to $\zeta^*$, $W_p(\mu,\nu)=\mathbb{E}^{1/p}\left[\left|X-Y\right|^p\right]$. For $\mu$ and $\nu$ in $\mathcal{P}(\mathbb{R}^d)$ with $\mu=f\nu$, we write $H(\mu\,|\,\nu)=\int f\log f\,d\nu$ for the relative entropy of $\mu$ with respect to $\nu$, and $H(\mu\,|\,\nu)=+\infty$ otherwise.

\section{The Algorithm }
Consider the overdamped Langevin SDE $(Z_t)_{t\in\mathbb{R}_{+}}$ given by
\begin{align}
    dZ_t=-h(Z_t)dt+\sqrt{2\beta^{-1}}dB_t,\ t\geq 0,\label{SDE_1}
\end{align}
with $Z_0=\theta_0\in\mathbb{R}^d$, where $h\in \partial u$, $\beta$ is the inverse temperature parameter and $(B_t)_{t\geq0}$ is a standard $d$-dimensional Brownian motion. To address the superlinear behavior of the subgradient $h(x)=\nabla u(x)$, we introduce a parameterized family of drift functions $(h_{\lambda})_{\lambda\geq 0}$, where $h_{\lambda}:\mathbb{R}^d\to\mathbb{R}^d$ depends on the stepsize $\lambda$. These drift functions, which we will refer to as taming functions, serve as close approximations of $\nabla u$. We then define, the tamed subgradient unadjusted Langevin algorithm (SG-TULA), $(\theta_n^{\lambda})_{n\geq0}$, by
\begin{flalign}
\mbox{(\textbf{SG-TULA}):} \qquad \qquad \quad  \qquad \theta_{n+1}^{\lambda}=\theta_n^{\lambda}-\lambda h_{\lambda}(\theta_n^{\lambda})+\sqrt{2\lambda\beta^{-1}}{\xi}_{n+1},\ \theta_0^{\lambda}=\theta_0,\ n\in\mathbb{N}\label{SGTULA}, &&
\end{flalign}
where $\lambda>0$ is the stepsize and $(\xi_n)_{n\geq1}$ is a sequence of i.i.d. standard Gaussian random variables on $\mathbb{R}^d$. We propose two distinct explicit choices for the family $(h_{\lambda})_{\lambda\geq 0}$, which are inspired by previous studies on tamed Langevin based algorithms \citep{KineticLangevin, TamingIPLA}. Define for all $x\in\mathbb{R}^d$ and $\lambda>0$:\begin{align}
    h_{\lambda,u}(x)&:=\dfrac{h(x)-\mu x}{1+\lambda^{1/2}|h(x)-\mu x|}+\mu x, \text{ and } \nonumber\\
    h_{\lambda,c}&:=\left(\dfrac{h^{(i)}(x)-\mu x^{(i)}}{1+\lambda^{1/2}|h^{(i)}(x)-\mu x^{(i)}|}+\mu x^{(i)}\right)_{i\in\{1,\ldots,d\}}. \label{tame_unif}
\end{align} The Euler scheme \eqref{SGTULA} with $h_{\lambda}=h_{\lambda,u}$ is referred to as the uniform-wise subgradient tamed unadjusted Langevin algorithm (SG-TULAu), and the choice $h_{\lambda}=h_{\lambda,c}$, is referred to as the coordinate-wise subgradient tamed unadjusted Langevin algorithm (SG-TULAc). The corresponding iterates of the algorithms will be denoted by $\theta_n^{\lambda,u}$ and $\theta_n^{\lambda,c}$ respectively. For brevity, we shall suppress this distinction and conduct the analysis for the uniform-wise scheme. The main theorems distinguish the variants where relevant.

\section{Assumptions \& Consequences}
\hypertarget{A1}{}\begin{assumption}
The potential $u$ is $L$-semi-convex for some $L\ge0$, meaning that $u+\frac{L}{2}|\cdot|^2$ is convex. Equivalently, every $h\in\partial u$ satisfies the one-sided Lipschitz condition
\begin{align}
    \langle h(x)-h(y),\,x-y\rangle \ge -L|x-y|^2,\qquad \forall x,y\in\mathbb{R}^d. \label{eqA1}
\end{align}
\end{assumption}

\begin{remark}
Under \hyperlink{A1}{A1}, $u$ is locally Lipschitz, hence by Rademacher's
theorem $\nabla u$ exists Lebesgue-almost everywhere, and by Alexandrov's theorem $u$ is
in fact twice differentiable a.e. Consequently the subdifferential
$\partial u(x)$ is nonempty and compact
for every $x\in\mathbb{R}^d$ and reduces to $\{\nabla u(x)\}$ almost everywhere. In this sense semi-convexity is a natural condition past smooth settings. It admits gradient discontinuities on a set of measure zero, as in $\ell_1$,
total-variation and SCAD penalties and ReLU activation functions.
\end{remark}

\hypertarget{A2}{}\begin{assumption}
The potential $u$ is strongly convex at infinity (outside a compact set). There exist $\mu>0$ and $R\ge0$ such that every $h\in\partial u$ satisfies
\begin{align} \langle h(x)-h(y),\,x-y\rangle &\ge \mu|x-y|^2,\qquad \forall x,y\in\mathbb{R}^d \text{ with } |x-y|\ge R. \label{eqA2} \end{align}
\end{assumption}

\hypertarget{A3}{}\begin{assumption}
There exist $m,K>0$ and $p\ge1$ such that every $h\in\partial u$ satisfies the growth bound
\begin{align}
    |h(x)|\le m+K|x|^p,\qquad\forall x\in\mathbb{R}^d. \label{eqA3}
\end{align}
\end{assumption}
\begin{remark} \label{remark_diss}
Under Assumptions \hyperlink{A2}{A2} and \ \hyperlink{A3}{A3}, $h$ is globally dissipative. There exist $\mu,b>0$, such that\begin{align}
\langle x,h(x)\rangle\geq \dfrac{\mu}{2}|x|^2-b,\ \forall x\in\mathbb{R}^d,\label{eqR1}
\end{align}
where $b=\max\left\{m^2/(2\mu),mR+(\mu/2)R^2+KR^{p+1}\right\}$.\\ Under the addition of \hyperlink{A1}{A1}, $b$ is further refined to $b=\max\left\{m^2/(2\mu),mR+(L+\mu/2)R^2\right\}$.
\begin{proof}
    The proof is postponed to Appendix \hyperlink{proof_remark_diss}{\ref*{sec:postponed}}.
\end{proof}
\end{remark}
\hypertarget{A4}{}\begin{assumption} The initial condition of the algorithm is such that
\begin{align}
    \mathbb{E}|\theta_0|^{2p}<\infty\label{eqA4}.
\end{align}
\end{assumption}
\hypertarget{A5}{}\begin{assumption}
The potential is coordinate-wise dissipative. There exist $\mu,b>0$ such that
\begin{align}
    x^{(i)}h^{(i)}(x)\geq \tfrac{\mu}{2}|x^{(i)}|^2-b,\qquad \forall x\in\mathbb{R}^d \text{ and } \forall i\in\{1,\dots,d\}. \label{eqA5}
\end{align}
\end{assumption}
\begin{lemma}\label{lem:pr-chord}
Let $u:\mathbb{R}^d\to\mathbb{R}$ be $L$-semi-convex \emph{(\hyperlink{A1}{A1})}, and suppose there exist
$\mu_0,\rho>0$ such that the Alexandrov Hessian satisfies $\nabla^2u(z)\succeq\mu_0 I$ a.e.\ for any $z$ with $|z|\ge\rho$. Then, every $h\in\partial u$ satisfies
\[
  \langle h(x)-h(y),\,x-y\rangle\ \ge\ \tfrac{\mu_0}{2}\,|x-y|^2,
  \qquad \forall x,y\in\mathbb{R}^d\text{ with } |x-y|\ \ge\ \tfrac{4\rho(\mu_0+L)}{\mu_0}.
\]
\end{lemma}
\begin{proof}
The proof is postponed to Appendix~\hyperlink{proof_lem_pr_chord}{\ref*{sec:postponed}}.
\end{proof}
\section{Main Results}
\begin{theorem}
\label{theorem_main}Let Assumptions \hyperlink{A1}{A1}--\hyperlink{A4}{A4} hold and $\lambda_0\in(0,1/(4\mu))$. Then, for every $\lambda\in(0,\lambda_0)$ and $n\in\mathbb{N}_0$, the subgradient tamed unadjusted Langevin algorithm (SG-TULAu) given in \eqref{SGTULA} satisfies
        \[\begin{aligned}W_2(\pi_{\beta},\mathcal{L}(\theta^{\lambda,u}_n))\leq  2C_{E}\lambda^{1/4}+eC_{D}\, e^{-n\lambda/T_0},
        \end{aligned}\]
        where $C_{E}$ is given in Lemma~\ref{lemma_err_W2}, $C_{D}$ in
        Lemma~\ref{lemma_crude} and $T_0$ in \eqref{def_horizon}. If additionally Assumption
        \hyperlink{A5}{A5} holds, the same bound holds for
        $W_2(\pi_{\beta},\mathcal{L}(\theta^{\lambda,c}_n))$, with the constants
        $C_{B_1},C_{B_2},C_{B_4}$ replaced by $C_{B_1,c},C_{B_2,c},C_{B_4,c}$ of
        Table~\ref{tab:constants}.
\end{theorem}
\begin{proof}
    The proof is postponed to Appendix \hyperlink{proof_theorem_main}{\ref*{sec:postponed}}.
\end{proof}
\begin{corollary}
\label{corollary1}
Let the assumptions of Theorem~\ref{theorem_main} hold. Then, for any $\epsilon>0$,
\[
\lambda\leq\min\left\{\lambda_{0},\,\epsilon^4/(4C_{E})^4\right\}\quad\text{and}\quad
N\geq \dfrac{T_0}{\lambda}\log\left(\dfrac{2eC_{D}}{\epsilon}\right)\ \text{iterations},
\]
one achieves $W_2(\pi_\beta,\mathcal{L}({\theta}^{\lambda,u}_{N}))\leq \epsilon$.
Overall one chooses $\lambda=\mathcal{O}(\epsilon^4)$ and
$N=\mathcal{O}(\epsilon^{-4}\log(1/\epsilon))$. In terms of the dimension-to-temperature ratio,  Table~\ref{tab:constants} yields
$\lambda=\mathcal{O}\left((d/\beta)^{-2p}\right)$ and
$N=\mathcal{O}\left((d/\beta)^{2p}\right)$.
\end{corollary}
The dependence of $\lambda$ and $N$ on the dimension is a feature of the
assumption regime rather than of the scheme. The exponent $2p$ is inherited from
the order $\lambda^{1/4}$ of the strong error, which is what the present
hypotheses on $u$ permit. Under enough regularity on the Hessian, the strong
error would be of order $\lambda$ and the same argument would give exponent $p/2$,
so that for $p=1$ one matches the optimal dimension dependence $(d/\beta)^{1/2}$
of a linearly growing gradient, as one has for a Gaussian target.
\begin{theorem}
\label{theorem2}
Let Assumptions \hyperlink{A1}{A1}--\hyperlink{A4}{A4} hold and $\lambda_0\in(0,1/(4\mu))$. Then, for any $\lambda\in(0,\lambda_0)$, $n\in\mathbb{N}$ and $\beta\geq \max\{4/\mu,\,1/J\}$, the following bound holds
\begin{align*}
    \mathbb{E}[u(\theta_n^{\lambda})]-u(\theta^{\star})&\leq  C_{\mathcal{T}_1}W_2\left(\pi_\beta,\mathcal{L}(\theta_n^{\lambda})\right) +\dfrac{d}{2\beta}\log\left(\dfrac{2(b+d/\beta)\beta^2J^2}{\mu }\right)+\dfrac{1}{2\beta}\log\left(\pi d\right)+\dfrac{13}{6\beta},
  \end{align*}
  where $ J=m+K2^{2p-2}(1+(2b/\mu)^{p/2})+K2^{p-1}/(p+1) $.
\end{theorem}
\begin{proof}
    Follows directly by invoking Lemmas \ref{lemmaER1} and \ref{lemmaER2}.
\end{proof}
Including a non-smooth PL condition, we obtain a result that relies purely on optimization arguments.
    \begin{theorem}\label{theorem-main-opt}
    Let Assumptions \hyperlink{A1}{A1}--\hyperlink{A4}{A4} hold and $\lambda_0\in(0,1/(4\mu))$.
    In addition, let $u$ satisfy the following non-smooth PL condition, i.e.,
    \begin{equation}\label{eq-nonsmoothPL}
        u(\theta)-u(\theta^{\star})\leq \frac{1}{2C_{PL}}(\operatorname*{dist}(0,\partial u))^2.
    \end{equation}
    Then, for any $\lambda\in(0,\lambda_0)$, $n\in\mathbb{N}$ and $\beta\geq K^2$, the following bound holds
    \[\begin{aligned}\mathbb{E} [u(\theta^\lambda_{n} )]-u(\theta^{\star})&\leq \left(1-\lambda \frac{C_{PL}}{1+\sqrt{\beta^{-1}}C_{PL}}\right)^n \left(\mathbb{E}[ u_\epsilon(\theta^\lambda_{0} )]-u(\theta^{\star})\right)\\&+ C\left(C_{PL}^{-1}+\beta^{-1/2} +1\right)\left( \lambda \beta^{1/2} + \beta^{-1/2} + \sqrt{\lambda}\right),\end{aligned}\]
    where $C\leq \mathcal{O}(d+(d/\beta)^{3p})$ with respect to the dimension.
\end{theorem}
\begin{proof}
    The proof is postponed to Appendix \hyperlink{theorem-main-opt}{\ref*{sec:postponed}}.
\end{proof}
\section{Proof Strategy}
Our proof strategy can be summarized in the following steps:
\begin{enumerate}
\item We begin our work by ensuring the well-posedness of the Langevin SDE and the uniqueness and existence of the invariant measure under our assumptions.
\item An important element of our proof strategy is to provide contraction rates for the SDE towards the invariant measure. The first step of our work, is to prove a Log-Sobolev inequality for the invariant measure under our assumptions, which relies on classical arguments with some delicate steps to cater for the presence of discontinuities.
A combination of the LSI and the semi-convexity assumption yields an exponential contraction from the SDE towards the invariant measure (see Proposition \ref{prop_contraction1}), i.e., \[W_2\left(\nu P_t,\pi_\beta\right)\leq C_we^{-C_rt}W_2\left(\nu,\pi_\beta\right).\]
\item For the discretization analysis we provide uniform in the number of iterations moment bounds for the process generated by the algorithm.
The taming design of the scheme and the dissipativity condition are key in this proof. In addition we prove uniform moment bounds for an auxiliary process which is a Langevin SDE started at the law of the algorithm.
The details can be found in Appendix~\ref{se-mom}.
\item The growth conditions, semi-convexity of the potential and uniform moment bounds yield finite time error estimates between the algorithm and the SDE started at a previous point of the algorithm for any $n_0<n$, denoting $\hat{P}$ the ones step-Markov kernel of the algorithm (see Lemma \ref{lemma_err_W2}), i.e.,
\[W_2(\hat{P}_{n-n_0}\hat{P}_{n_0},\hat{P}_{n-n_0} P_{\lambda n_0})\leq C e^{c\lambda n_0} \lambda^{1/4}. \]
\item Combining the discretization and contraction results \[\begin{aligned}
    W_2(\mathcal{L}(\theta_n^\lambda),\pi_\beta)&\leq W_2(\hat{P}_{n-n_0}\hat{P}_{n_0},\hat{P}_{n-n_0} P_{\lambda n_0}) + W_2(\hat{P}_{n-n_0} P_{\lambda n_0},\pi_\beta)
    \\&\leq C e^{c\lambda n_0} \lambda^{1/4} +  C_w e^{-\dot{c}\lambda n_0} W_2(\hat{P}_{n-n_0},\pi_\beta)
    \\&= C e^{c\lambda n_0} \lambda^{1/4} +  C_w e^{-\dot{c}\lambda n_0} W_2(\mathcal{L}(\theta^\lambda_{n-n_0}),\pi_\beta).
\end{aligned}\]
By picking a suitable $n_0$ so that the prefactor in the second term is smaller than $e^{-1}$, one achieves 
\[W_2(\mathcal{L}(\theta_n^\lambda),\pi_\beta))\leq C^* \lambda^{1/4} + e^{-1} W_2(\mathcal{L}(\theta^\lambda_{n-n_0}),\pi_\beta), \]
which yields Theorem \ref{theorem_main} by applying the discrete Gr{\"o}nwall's inequality.
\item The proof of Theorem \ref{theorem2} is based on decomposing \[\mathbb{E}[u(\theta_n^{\lambda})]-u(\theta^{\star})=\underbrace{\mathbb{E}[u(\theta_n^{\lambda})]-\mathbb{E}[u(\theta_{\infty})]}_{T_1}+\underbrace{ \mathbb{E}[u(\theta_{\infty})]-u(\theta^{\star})}_{T_2}.  \] 
The term $T_1$ refers to the sampling problem and can be bounded by $W_2(\mathcal{L}(\theta_n^\lambda),\pi_\beta)$ in Lemma \ref{lemmaER1} while the second term refers to the concentration of the target measure around the minimum and can be controlled by a decaying function of $\beta$ in Lemma \ref{lemmaER2}.
\item The proof of Theorem \ref{theorem-main-opt} relies on the properties of the Moreau--Yosida regularizer with fixed regularization parameter $\varepsilon=\sqrt{\beta^{-1}}$.
Applying the descent Lemma to $u_\epsilon$, together with the algorithm's moment bounds and the taming-error estimate, gives
 \begin{equation}\label{eq-opt-nablabound}
     \frac{\lambda}{2} \mathbb{E} |\nabla u_\epsilon(\theta^\lambda_n )|^2\leq \mathbb{E} [  u_\epsilon(\theta^\lambda_{n})]-\mathbb{E} [u_\epsilon(\theta^\lambda_{n+1} )]  + C'\left( \lambda^2 \sqrt{\beta} + \lambda \sqrt{\beta^{-1}} + \lambda \sqrt{\lambda}\right).
 \end{equation}
The non-smooth PL condition on $u$ also yields a PL inequality for its envelope, i.e.,  \[u_\epsilon(\theta^\lambda_n)-u(\theta^{\star})\leq c |\nabla u_\epsilon(\theta_n^\lambda)|^2.\]
Combining this inequality with \eqref{eq-opt-nablabound} leads to \begin{align*}\mathbb{E} [u_\epsilon(\theta^\lambda_{n+1} )]-u(\theta^{\star})&\leq \left(1-\lambda \frac{C_{PL}}{1+\sqrt{\beta^{-1}}C_{PL}}\right)^n \left(\mathbb{E} [u_\epsilon(\theta^\lambda_{0} )]-u(\theta^{\star})\right)\\&+ \frac{1+\sqrt{\beta^{-1}}C_{PL}}{C_{PL}}C'\left( \lambda \sqrt{\beta} + \sqrt{\beta^{-1}} + \sqrt{\lambda}\right).\end{align*}
Finally, the Moreau-Yosida approximation properties and uniform moment bounds control $\sup_n\mathbb{E}|u(\theta_n^\lambda)-u_\epsilon(\theta^\lambda_n)|$. This transfers the bound from $u_\epsilon$ to $u$ and completes the proof.
\end{enumerate}
\section{Examples}
We consider two principal examples: the pretraining of the \emph{nanochat} language model in
Section~\ref{sec:nc-exp}, and the generalized phase retrieval problem in Section~\ref{sec:pr}.

\subsection{Nanochat Pretraining}\label{sec:nc-exp}

We pretrain the nanochat language model of \citep{Karpathy_nanochat_2025} at two
depths, driving the matrix-valued parameters with the boosted coordinate-wise
map of \citet{Lim_Sabanis_2024} evaluated at the subgradient selection $h$. The
architecture, the potential $u$ and the associated bounds are those of
Appendix~\ref{sec:nanochat}. We refer to the resulting scheme as
SG-TheoPouLA, and it attains the lowest validation bits per byte of the three
optimisers at $\mathsf{L}=12$, where every configuration was tuned directly due to the lack of available scaling laws,
and remains competitive at $\mathsf{L}=24$.

\begin{algorithm}[h]
\caption{SG-TheoPouLA }
\label{alg:nc-tula}
\begin{algorithmic}[1]
\REQUIRE stepsize schedule $(\lambda_n)$, threshold schedule $(\varepsilon_n)$,
  averaging coefficient $\beta_1$, weight decay $w$, initial $\theta_0$ and
  $m_0=0$
\FOR{$n=0,1,2,\dots$}
  \STATE draw a mini-batch and form the subgradient $H_{n+1}\in\partial
    u(\theta_n)$
  \STATE $\psi_n\leftarrow(1-\lambda_n w)\,\theta_n$
  \STATE $m_{n+1}\leftarrow\beta_1 m_n+(1-\beta_1)H_{n+1}$
  \STATE $\theta_{n+1}\leftarrow\psi_n-\dfrac{\lambda_n}{1-\beta_1^{\,n+1}}\,
    \mathrm{T}_{\lambda_n,\varepsilon_n}(m_{n+1})$
\ENDFOR
\end{algorithmic}
\end{algorithm}

\emph{The boosted tamed map.}
Writing $\lambda$ for the stepsize and $\varepsilon>0$ for the boosting
threshold, the map acts on $v\in\mathbb{R}^{\dpar}$ coordinate by coordinate as
\begin{align}\label{eq:nc-boostmap}
  \bigl(\mathrm{T}_{\lambda,\varepsilon}(v)\bigr)^{(i)}
  =\Bigl(1+\frac{\lambda^{1/2}}{\varepsilon+|v^{(i)}|}\Bigr)\,
   \frac{v^{(i)}}{1+\lambda^{1/2}\,|v^{(i)}|},
  \qquad i\in\{1,\dots,\dpar\},
\end{align}
where the right factor is the coordinate-wise tamed drift $h_{\lambda,c}$ of
\eqref{tame_unif} taken with $\mu=0$, whereas the left factor lifts those
coordinates that taming leaves small. The constant $\varepsilon$ acts as a
threshold rather than as a strength, since it fixes the magnitude below which a
coordinate is treated as small and receives the lift, and it caps that lift at
$1+\lambda^{1/2}/\varepsilon$. The device was introduced in
\citet{Lim_Sabanis_2024} to address vanishing gradients, which is exactly the
regime in which taming alone overshrinks a coordinate and the iterate
stalls. This is key to make the taming/boosting map competitive on deep architectures.

\emph{From \eqref{SGTULA} to Algorithm~\ref{alg:nc-tula}.}
Three choices specialise the analysed scheme to pretraining at this scale. The
first is the coordinate-wise boosted map itself. Following the construction of \citet{Lim_Sabanis_2024} in their
equation~(11), we use decoupled weight decay, motivated by the
higher-order regularisation in \eqref{eq:nc-ureg}. The second is setting the temperature at zero. Notice that since the
inverse temperatures relevant to pretraining are of order $\beta\sim10^{12}$, the
injected noise term of \eqref{SGTULA} would contribute fluctuations orders of
magnitude below those already supplied by mini-batch sampling. Drawing it would
consume computation to no measurable effect. The
third and most consequential is the exponential moving average, which we form
on the raw subgradients. We find this
ordering substantially stronger at that scale than its converse, namely when each
subgradient is tamed before entering the average.

Two further details conclude the description of the implementation. The stepsize
entering \eqref{eq:nc-boostmap} is the scheduled learning rate, following a linear decay to a fixed fraction of its peak. The $\varepsilon$-threshold decreases on a $\log_{10}$ scale from its starting value to
that in Table~\ref{tab:nc-hyper}, essentially controlling the effect of the boosting multiplier applied to near-zero coordinates over the warmup steps while the
average is still forming.

\begin{table}[t]
\centering
\begin{tabular}{llrr}
\hline
Optimiser & Tuning & Val bpb $\downarrow$ & CORE $\uparrow$\\
\hline
Muon                  & Karpathy & 0.8253 & 0.1664\\
Muon                  & ours     & 0.8199 & 0.1706\\
AdamW                 & Karpathy & 0.8579 & 0.1143\\
AdamW                 & ours     & 0.8254 & 0.1523\\
SG-TheoPouLA \ (i)    & ours     & \textbf{0.8163} & 0.1638\\
SG-TheoPouLA \ (ii)   & ours     & 0.8199 & \textbf{0.1715}\\
\hline
\end{tabular}
\caption{Pretraining at $\mathsf{L}=12$, where lower is better for validation
bits per byte and higher is better for CORE, with the best of each column in
bold. AdamW drives the non-matrix groups in every row.}
\label{tab:nc-d12}
\end{table}

\emph{Experiment details.}
Muon \citep{Jordan_Muon_2024} is defined only for matrix-valued parameters and therefore requires a
second optimiser for the embeddings, the unembedding and the scalars present in the architecture. We follow \citet{Karpathy_nanochat_2025} in taking
AdamW for those groups and hold its hyperparameters fixed throughout, so that
every difference recorded below is attributable to the matrix optimiser alone
and the three are compared on equal terms. Moreover, the pipeline equips Muon with three devices that are not an integral part of the
algorithm but plug-ins, namely a second-moment normalisation of the
step, a momentum schedule and a weight decay decaying to zero across training. We
disable all three so that the comparison isolates the matrix update rule rather
than the engineering accumulated around it. Both models take $V=65\,536$,
context length $T=2048$ and full causal attention at every layer, trained on the
ClimbMix-400B corpus of \citet{Diao_ClimbMix_2025}. The shallower model has
$\dm=768$ and $H=6$, whereas the deeper one has $\dm=1536$ and $H=12$. The
shallower runs cover $5000$ steps at a batch of $524\,288$ tokens, whereas the
deeper runs cover $7812$ steps at a batch of $1\,048\,576$ tokens, and in both
cases the learning rate follows a warmup over the tabulated number of steps and
a cosine decay thereafter. We report the validation bits per byte recorded by
the training loop at the final step, together with CORE, the unweighted mean of
the centred accuracies over the $22$ tasks of the DCLM suite of
\citet{Li_DCLM_2024}. Repeating a configuration under various seeds returns validation
bits per byte within $0.0008$, so differences below that figure are read as
level throughout. All runs were carried out on a cluster of four NVIDIA A100
80GB devices.

\begin{figure}[t]
  \centering
  \includegraphics[width=0.67\linewidth,height=0.74\textheight,keepaspectratio]{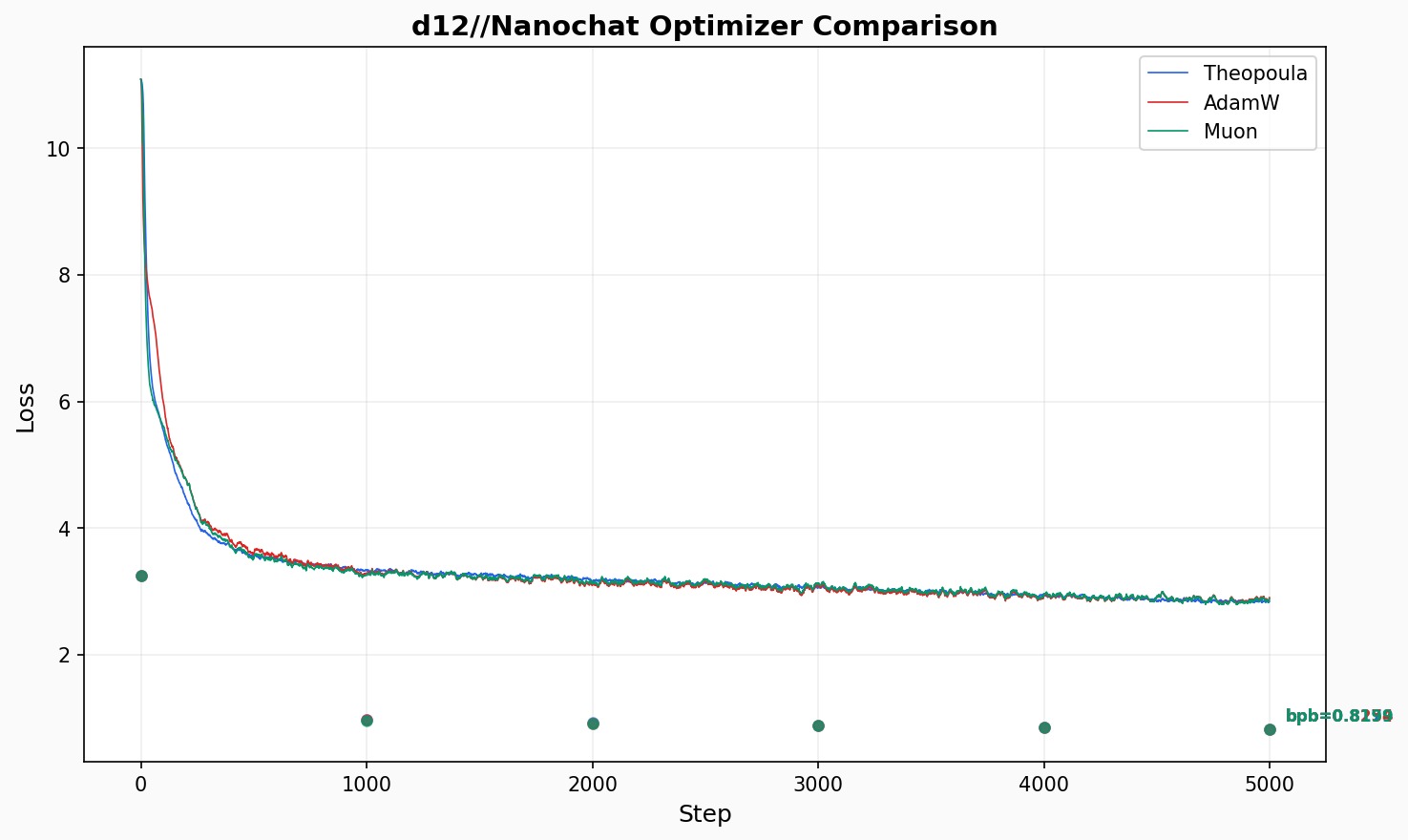}
  \caption{Pretraining at $\mathsf{L}=12$ under the tuned configurations of Table~\ref{tab:nc-hyper}.}
  \label{fig:nc-d12}
\end{figure}

\emph{Hyperparameter transfer.}
Depth $\mathsf{L}$ is the single scaling knob in the pipeline, setting the width as
$\dm=64\mathsf{L}$ rounded up to a multiple of the head dimension $128$, giving the number of heads $H=\dm/128$. Moreover the horizon is fixed at
$D=r\,N_{\mathrm{scaling}}$, for a tokens-per-parameter ratio $r$, where
$N_{\mathrm{scaling}}$ counts the transformer matrices including the
unembedding. The batch size follows $B=B_{\mathrm{ref}}(D/D_{\mathrm{ref}})^{0.383}$ rounded
to the nearest power of two, and every learning rate but the hardcoded smear gate is then
corrected by $(B/B_{\mathrm{ref}})^{1/2}$. The table-shaped groups comprising the
unembedding and the two embedding tables carry the further width correction
$(\dm/768)^{-1/2}$, from which the matrix and scalar rates are exempt, whereas
the weight decay is set by
$w=w_{\mathrm{ref}}(B/B_{\mathrm{ref}})^{1/2}D_{\mathrm{ref}}/D$ to keep $B/(\eta\lambda D)$ constant. Since doubling the depth doubles both
the batch and the width, the two corrections cancel on the table-shaped groups
and their rates hold at either scale, whereas the matrix and scalar rates retain
the batch correction alone. The rules are calibrated for AdamW and Muon, which therefore
move to the deeper model automatically and arrive already close to their best.
No comparable study exists for the boosted tamed map, so its configuration at
the greater depth was found by manual finetuning and its own scaling law remains to
be quantified in future research.

Figure~\ref{fig:nc-d12} and Table~\ref{tab:nc-d12} record the shallower
comparison, in which SG-TheoPouLA reaches the lowest validation bits per byte
of the six configurations under its first setting and the highest CORE under
its second, the two differing in warmup, in weight decay and in the final
learning rate fraction. Retuning improved both Muon and AdamW over their published defaults, making them more comparable to SG-TheoPouLA.

At the greater depth (Figure~\ref{fig:nc-d24} and
Table~\ref{tab:nc-d24}), Muon leads and SG-TheoPouLA follows it closely on
validation bits per byte while remaining next to AdamW on CORE. All three
improve on the available baseline.

\begin{table}[t]
\centering
\begin{tabular}{lrr}
\hline
Model or optimiser & Val bpb $\downarrow$ & CORE $\uparrow$\\
\hline
OpenAI GPT-2 \citep{Radford_GPT2_2019}           & ---    & 0.2565\\
Karpathy baseline \citep{Karpathy_nanochat_2025} & 0.7186 & 0.2571\\
\hline
Muon           & \textbf{0.7067} & \textbf{0.2637}\\
AdamW          & 0.7167 & 0.2577\\
SG-TheoPouLA   & 0.7133 & 0.2574\\
\hline
\end{tabular}
\caption{Pretraining at $\mathsf{L}=24$, where the first two rows are published
reference figures rather than runs of ours and the best among our runs is in
bold. AdamW drives the non-matrix groups in every row.}
\label{tab:nc-d24}
\end{table}

\begin{figure}[t]
  \centering
  \includegraphics[width=0.67\linewidth,height=0.74\textheight,keepaspectratio]{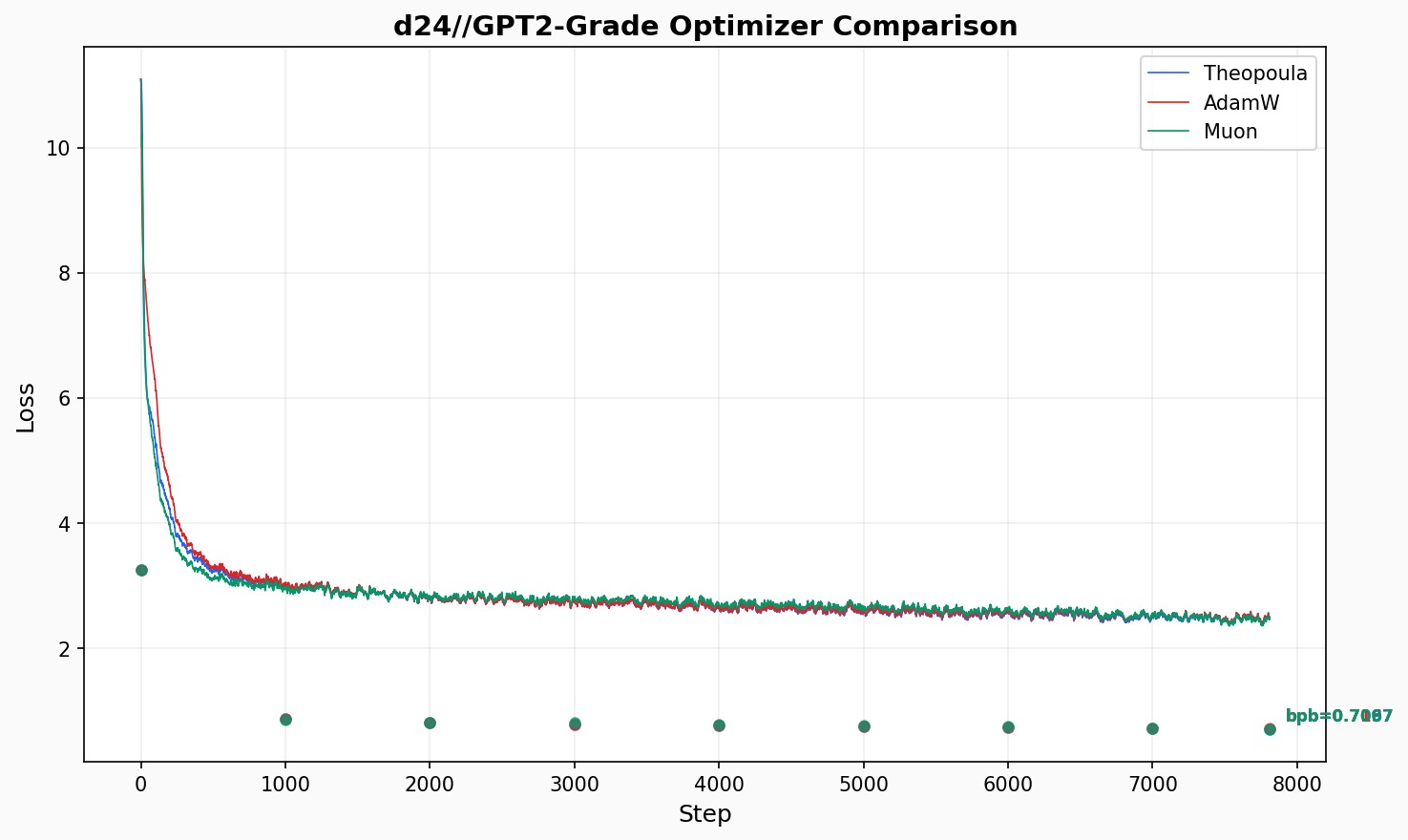}
  \caption{Pretraining at $\mathsf{L}=24$ under the tuned configurations of Table~\ref{tab:nc-hyper}.}
  \label{fig:nc-d24}
\end{figure}

\begin{table}[h]
\centering
\small
\setlength{\tabcolsep}{3.5pt}
\begin{tabular}{llrrrrll}
\hline
Optimiser & $\mathsf{L}$ & Matrix lr & Warmup & Weight decay & Final lr
  & EMA & Boosting\\
\hline
Muon (Karpathy)   & 12 & $0.02$             & 40  & 0.15 & 0.05 & $0.90$ & ---\\
Muon (ours)       & 12 & $0.008$            & 40  & 0.15 & 0.05 & $0.90$ & ---\\
AdamW (Karpathy)  & 12 & $3\times10^{-4}$   & 40  & 0.15 & 0.05 & $0.80$ & ---\\
AdamW (ours)      & 12 & $2\times10^{-3}$   & 40  & 0.15 & 0.05 & $0.80$ & ---\\
SG-TheoPouLA (i)  & 12 & $0.01$             & 120 & 0.10 & 0.05 & $0.82$
  & $3\times10^{-4}\to10^{-11}$\\
SG-TheoPouLA (ii) & 12 & $0.01$             & 80  & 0.15 & 0.10 & $0.82$
  & $3\times10^{-4}\to10^{-11}$\\
\hline
Muon              & 24 & $0.02$             & 40  & 0.15 & 0.05 & $0.90$ & ---\\
AdamW             & 24 & $1.8\times10^{-3}$ & 50  & 0.15 & 0.05 & $0.80$ & ---\\
SG-TheoPouLA      & 24 & $7\times10^{-3}$   & 120 & 0.18 & 0.10 & $0.82$
  & $3\times10^{-4}\to10^{-11}$\\
\hline
\end{tabular}
\caption{Nominal hyperparameters. The matrix rate and the weight decay are the
values supplied to the pipeline, which rescales them with depth, so the rate at
$\mathsf{L}=24$ is not directly comparable with the one above it. The learning rates for the embeddings, the unembedding
and the scalars being $0.3$, $0.008$ and $0.5$. Warmup is measured in steps and
the final learning rate is a fraction of the peak.}
\label{tab:nc-hyper}
\end{table}
\subsection{Generalized Phase Retrieval}\label{sec:pr}

Phase retrieval reconstructs a signal from the magnitudes of its linear
measurements. It is the standard model for X-ray crystallography, coherent
diffraction imaging and astronomical imaging, where detectors record intensities but not phase
\citep{SunQuWright2018}. One fixes an unknown signal $z_\star\in\mathbb{C}^n$ and known sensing vectors
$a_1,\dots,a_N\in\mathbb{C}^n$, here drawn i.i.d.\ complex Gaussian,
$a_k=\tfrac{1}{\sqrt2}(X_k+iY_k)$ with $X_k,Y_k\sim\mathcal{N}(0,I_n)$ independent, and observes only the
magnitudes $y_k=|a_k^{*}z_\star|$, $k=1,\dots,N$. Since $z_\star$ and $z_\star e^{i\phi}$ produce the
same magnitudes, the signal can be recovered only up to a global phase $z_\star\mapsto z_\star e^{i\phi}$.
Writing $b_k:=y_k^2$, a standard way to fit the measurements is to minimise the quartic least-squares
energy
\begin{align*}
  f(z)=\frac{1}{2N}\sum_{k=1}^{N}\bigl(|a_k^{*}z|^2-b_k\bigr)^2,\qquad z\in\mathbb{C}^n ,
\end{align*}
whose associated target is $\pi_\beta\propto e^{-\beta f}$.

To read $f$ as a potential on $\mathbb{R}^d$ we use the canonical identification
$\mathbb{C}^n\cong\mathbb{R}^{d}$, $d=2n$, sending $z$ to $x=(\Re z;\Im z)$. This map is an isometry:
the real part of the Hermitian product becomes the Euclidean inner product and $|z|=|x|$, so lengths
are preserved. Each squared magnitude is a Hermitian quadratic form, $|a_k^{*}z|^2=z^{*}(a_ka_k^{*})z$,
and under the identification the Hermitian matrix $a_ka_k^{*}$ is represented by the real symmetric
positive semi-definite matrix
\begin{align*}
  A_k=\begin{pmatrix} p_kp_k^{\top}+q_kq_k^{\top} & p_kq_k^{\top}-q_kp_k^{\top}\\[2pt]
                      q_kp_k^{\top}-p_kq_k^{\top} & p_kp_k^{\top}+q_kq_k^{\top}\end{pmatrix},
  \qquad a_k=p_k+iq_k,\ \ p_k,q_k\in\mathbb{R}^n,
\end{align*}
the unique matrix with $x^{\top}A_kx=|a_k^{*}z|^2$ for every $x$; it has rank two and largest
eigenvalue $|a_k|^2$. Consequently $f$, written in the real variable $x$, is exactly the potential
\begin{align}\label{eq:pr-potential}
  u(x)=\frac{1}{2N}\sum_{k=1}^{N}\bigl(x^{\top}A_kx-b_k\bigr)^2,\qquad x\in\mathbb{R}^d,
\end{align}
that is, $u(x)=f(z)$ whenever $x$ is the real form of $z$. Verifying \hyperlink{A1}{A1}--\hyperlink{A3}{A3}
for $u$ is therefore legitimate and concerns the very function we sample. Being a polynomial of degree
four, $u$ is smooth, with gradient
\begin{align}\label{eq:pr-grad}
  h(x)=\nabla u(x)=\frac{2}{N}\sum_{k=1}^{N}\bigl(x^{\top}A_kx-b_k\bigr)A_kx,
\end{align}
and Hessian
\begin{align}\label{eq:pr-hess}
  \nabla^2u(x)=u_1(x)+u_2(x)-u_3,
\end{align}
where $u_1(x)=\tfrac{4}{N}\sum_{k}(A_kx)(A_kx)^{\top}$,
$u_2(x)=\tfrac{2}{N}\sum_{k}(x^{\top}A_kx)A_k$ and $u_3=\tfrac{2}{N}\sum_{k}b_kA_k$.

\medskip
\noindent Assumptions \hyperlink{A1}{A1} and \hyperlink{A3}{A3} hold deterministically, for every
realisation of the measurements.

\begin{remark}[\hyperlink{A1}{A1} holds]\label{prop:pr-semiconvex}
For every fixed $\{(a_k,y_k)\}_{k=1}^N$, the potential $u$ is $L$-semi-convex with $L=\lambda_{\max}(u_3)$.
\end{remark}
\begin{proof}
In view of \eqref{eq:pr-hess}, the matrices $u_1(x)$ and $u_2(x)$ are positive semi-definite. In particular, $u_1(x)$
is a sum of the rank-one terms $(A_kx)(A_kx)^{\top}$, while $u_2(x)$ combines the positive
semi-definite $A_k$ with the nonnegative weights $x^{\top}A_kx=|a_k^{*}z|^2\ge0$. Hence
$\nabla^2u(x)\succeq-u_3\succeq-\lambda_{\max}(u_3)I=-LI$ for all $x$. Since $u$ is smooth, integrating
this bound along the segment $\gamma(t)=y+t(x-y)$ gives
$$\langle h(x)-h(y),x-y\rangle=\int_0^1(x-y)^{\top}\nabla^2u(\gamma(t))(x-y)\,dt\ge-L|x-y|^2.$$
\end{proof}

\begin{remark}[\hyperlink{A3}{A3} holds]\label{prop:pr-growth}
For every fixed $\{(a_k,y_k)\}_{k=1}^N$, $h=\nabla u$ satisfies the polynomial growth \eqref{eqA3} with $p=3$,
$m=\tfrac{2}{N}\sum_{k=1}^N y_k^2|a_k|^2$ and $K=\tfrac{2}{N}\sum_{k=1}^N|a_k|^4+m$.
\end{remark}
\begin{proof}
By Cauchy--Schwarz, $|a_k^{*}z|\le|a_k||z|=|a_k||x|$, so that $x^{\top}A_kx=|a_k^{*}z|^2\le|a_k|^2|x|^2$
and $|A_kx|=|a_k^{*}z|\,|a_k|\le|a_k|^2|x|$. In view of \eqref{eq:pr-grad}, one calculates
\begin{align*}
  |h(x)|\le\frac{2}{N}\sum_{k=1}^N\bigl(x^{\top}A_kx+b_k\bigr)|A_kx|
  \le\Bigl(\frac{2}{N}\sum_{k=1}^N|a_k|^4\Bigr)|x|^3+\Bigl(\frac{2}{N}\sum_{k=1}^N y_k^2|a_k|^2\Bigr)|x|.
\end{align*}
Since $|x|\le1+|x|^3$, the linear term is bounded by $m(1+|x|^3)$, and one obtains $|h(x)|\le m+K|x|^3$.
\end{proof}
\noindent This example involves two independent
probabilistic mechanisms that must not be conflated.
\begin{itemize}
\item[(M)] \textit{The measurement model.} The sensing vectors $a_1,\dots,a_N$ are drawn once,
i.i.d.\ complex Gaussian, and then fixed. This draw is part of the inverse problem, not
of our algorithm; the relevant probabilities are those of \cite{SunQuWright2018}.
\item[(A)] \textit{The algorithm.} The injected noise and initialization of SG-TULA. All
convergence guarantees of this paper are taken in expectation over (A) alone.
\end{itemize}
The statement ``holds with high probability'' in Remark~\ref{prop:pr-sci} below is a probability over
\textbf{(M) only}, a property of the measurement ensemble, certified in \cite{SunQuWright2018} and
established \emph{before} our algorithm runs. We fix one such instance, on it $u$ is a deterministic
function satisfying \hyperlink{A1}{A1}--\hyperlink{A3}{A3}, and every result of this paper then applies
to it deterministically. No probability over (M) ever enters a guarantee over (A).

\begin{remark}[\hyperlink{A2}{A2} holds for a generic measurement ensemble]\label{prop:pr-sci}
Let $a_1,\dots,a_N$ be i.i.d.\ $\mathcal{CN}(0,I_n)$, and write $L=\lambda_{\max}(u_3)$ as in
Remark~\textnormal{\ref{prop:pr-semiconvex}}. For $N$ large enough, with high probability over the draw
of $\{a_k\}$ the fixed potential $u$ satisfies \hyperlink{A2}{A2}: for every $\mu>0$, the bound
\eqref{eqA2} holds with radius
\[
  R=\frac{2\,(2\mu+L)^{3/2}}{\mu}.
\]
\end{remark}
\begin{proof}
By \textnormal{\citet[Lemma~6.4]{SunQuWright2018}} (with $\delta=\tfrac12$), for $N$ large enough the event
\[
  \mathcal{E}:\qquad \frac1N\sum_{k=1}^N|a_k^{*}z|^2|a_k^{*}w|^2\ \ge\ \frac12|z|^2|w|^2
  \qquad\text{for all }z,w\in\mathbb{C}^n
\]
holds with high probability over the draw of $\{a_k\}$. We fix such a draw. Realifying the
bound defining $\mathcal{E}$ (using $x^{\top}A_kx=|a_k^{*}z|^2$, $w^{\top}A_kw=|a_k^{*}\tilde w|^2$ and
$|z|=|x|$, $|\tilde w|=|w|$) gives, for every $w\in\mathbb{R}^d$,
\[
  w^{\top}u_2(x)\,w=\frac{2}{N}\sum_{k=1}^N(x^{\top}A_kx)(w^{\top}A_kw)\ \ge\ |x|^2|w|^2,
  \qquad\text{that is}\qquad u_2(x)\succeq|x|^2 I.
\]
Since $u_1(x)\succeq0$ and $u_3\preceq LI$, the decomposition \eqref{eq:pr-hess} yields
$\nabla^2u(x)\succeq u_2(x)-u_3\succeq(|x|^2-L)I$. Fix $\mu>0$ and set $\mu_0:=2\mu$ and
$\rho:=\sqrt{\mu_0+L}=\sqrt{2\mu+L}$; then $\nabla^2u(x)\succeq\mu_0 I$ whenever $|x|\ge\rho$. The potential
is $L$-semi-convex by Remark~\ref{prop:pr-semiconvex}, so Lemma~\ref{lem:pr-chord} applies with these
$\mu_0,\rho,L$ and gives \eqref{eqA2} with constant $\mu_0/2=\mu$ and radius
$4\rho(\mu_0+L)/\mu_0=2(2\mu+L)^{3/2}/\mu$.
\end{proof}

In practice the energy \eqref{eq:pr-potential} is rarely minimised on its own. It is common, when
the signal is sparse, to be augmented with a sparsity promoting $\ell_1$ penalty, giving the
standard model of sparse phase retrieval \citep{ShechtmanBeckEldar},
\[
  \tilde u(x)=\frac{1}{2N}\sum_{k=1}^{N}\bigl(x^{\top}A_kx-b_k\bigr)^2+\tau|x|_1,
  \qquad \tau>0.
\]
The term $\tau|x|_1$ keeps $\tilde u$ continuous but, being non-differentiable on the coordinate
hyperplanes $\{x^{(i)}=0\}$, introduces gradient discontinuities on this set of measure zero, where
the drift is given by a bounded subgradient selection. As $\tau|x|_1$ is convex with subgradient of
norm at most $\tau\sqrt{d}$, Assumptions \hyperlink{A1}{A1}--\hyperlink{A3}{A3} are preserved. The target is thus simultaneously
superlinear, non-smooth and non-convex.
\bibliographystyle{plainnat}
\section*{References} \vspace{-2.0em} \renewcommand\refname{} 
\appendix
\section{Well Posedness}
Consider the infinitesimal generator $\mathcal{L}$ associated with \eqref{SDE_1} defined for all $\phi\in C^2(\mathbb{R}^d)$ and $x\in\mathbb{R}^d$ by $\mathcal{A}\phi(x)=-\langle h(x),\nabla\phi(x)\rangle +\beta^{-1}\Delta\phi(x)$. Next define the Lyapunov function $V_p(x)=(1+|x|^2)^{p/2}$ for all $x\in\mathbb{R}^d$. Note that $V_p$ is twice continuously differentiable, and under Assumptions \hyperlink{A2}{A2} and \hyperlink{A3}{A3} so that Remark \ref{remark_diss} holds, one gets the following growth condition
\begin{align}
    \mathcal{A}V_p(x)\leq C_2 V_p(x),\ \forall x\in\mathbb{R}^d,\label{lyapunov_growth}
\end{align}
where $C_2$ is given in Lemma \ref{lemmaB1}.  Moreover under Assumptions \hyperlink{A2}{A2} and \hyperlink{A3}{A3}, it satisfies the geometric drift condition \begin{align}
    \mathcal{A}V_p(x)\leq -C_3(p)V_p(x)+C_4(p),\ \forall x\in\mathbb{R}^d,\label{lyapunov_geometric}
\end{align} where $C_3(p)$ and $C_4(p)$ are given in Lemma \ref{lemmaB2}. It follows that
\begin{align}
    \lim_{|x|\to\infty}V_p(x)=+\infty,\ \lim_{|x|\to\infty}\mathcal{L}V_p(x)=-\infty\label{lyapunov_limits}.
\end{align}
\begin{proposition}\label{prop_SDE}
Let Assumptions \hyperlink{A1}{A1}--\hyperlink{A4}{A4} hold. The SDE \eqref{SDE_1} has a unique strong solution.
\end{proposition}
\begin{proof}
  Uniqueness is guaranteed under the monotonicity condition \hyperlink{A1}{A1} and due to the diffusion coefficient being constant. Moreover, all conditions of Theorem 2.8 in \cite{Krylov} are satisfied under our assumptions; therefore, the SDE \eqref{SDE_1} admits a unique strong solution. In particular, since the drift coefficient is subject to the growth condition of Remark \ref{remark_diss} and the diffusion coefficient is constant, in view of \eqref{lyapunov_growth}, they trivially satisfy the conditions (i), (ii) and (iv). Condition (iii) is also satisfied trivially as in our case the domain is $D=\mathbb{R}^d$.
\end{proof}
\begin{proposition}\label{prop_measure}
    Let Assumptions \hyperlink{A1}{A1}--\hyperlink{A3}{A3} hold, the Langevin SDE \eqref{SDE_1} admits a unique invariant measure.
\end{proposition}
\begin{proof}
  The existence of an invariant measure is established under Assumptions \hyperlink{A2}{A2} and \hyperlink{A3}{A3}. In particular, the Langevin SDE \eqref{SDE_1} has a constant diffusion coefficient and Assumption \hyperlink{A3}{A3} ensures that the drift coefficient is locally integrable. Consequently, in view of \eqref{lyapunov_limits}, all the conditions of Theorem 2.2 in \cite{bogachev2000uniqueness} are satisfied, the existence of at least one invariant measure follows. Moreover, with the inclusion of Assumption \hyperlink{A1}{A1}, the contraction results in Appendix~\ref{sec_contraction} imply the uniqueness of the invariant measure. This is a direct consequence of Proposition \ref{prop_contraction1}, by setting the initial condition $Z_0$ in \eqref{SDE_1} to be such that $\mathcal{L}(Z_0)=\mathcal{L}(\pi_{\beta}).$
\end{proof}
\begin{proposition}\label{prop_measure_id}
     Let Assumptions \hyperlink{A1}{A1}--\hyperlink{A3}{A3} hold. The invariant measure $\pi_\beta$ of the SDE \eqref{SDE_1}, is characterized by the density $Z^{-1}\exp(-\beta u(x))$, with $Z$ being the normalization constant.
\end{proposition}
\begin{proof}
   Under Assumption \hyperlink{A3}{A3} one yields that $u\in \mathbb{H}_{\text{loc}}^1$ and the rest follow from Theorem 3 in \cite{Fruehwirth_Ergodicity_2024}.
\end{proof}
\section{Technical Prerequisites}
This appendix introduces the auxiliary processes used in the analysis and records the
properties of the taming functions.

\subsection{Auxiliary Processes}
We next introduce the auxiliary processes which are used in our analysis. For each $\lambda>0$, the time-scaled process $(Z_t^{\lambda})_{t\in\mathbb{R}_{+}}$ is defined by $Z_t^{\lambda}:=Z_{\lambda t}$, $t\in\mathbb{R}_{+}$. We note that
\begin{align}
    dZ_t^{\lambda}=-\lambda h(Z_t^{\lambda})dt+\sqrt{2\lambda \beta^{-1}}d{\tilde{B}}_t^{\lambda},\ Z_0^{\lambda}=\theta_0, \label{SDE_scaled}
\end{align}
where the Brownian motion $(\tilde{B}_t^{\lambda})_{t\geq 0}$ is defined as $\tilde{B}_t^{\lambda}:=B_{\lambda t}/\sqrt{\lambda},\ t\geq0$. The natural filtration of $(\tilde{B}_t^{\lambda})_{t\geq 0}$ is denoted by $(\mathcal{F}_t^{\lambda})_{t\geq 0}$ with $\mathcal{F}_t^{\lambda}:=\mathcal{F}_{\lambda t}, t\in\mathbb{R}_{+}.$ Then, we define $(\bar{\theta}_t^{\lambda})_{t\in\mathbb{R}_{+}}$, the continuous-time interpolation of SG-TULA \eqref{SGTULA}, as
\begin{align}
   d\bar{\theta}_t^{\lambda}=-\lambda h_{\lambda}(\bar{\theta}_{\lfloor t \rfloor}^{\lambda})dt+\sqrt{2\lambda\beta^{-1}}d\tilde{B}_t^{\lambda},\ \bar{\theta}_0^{\lambda}=\theta_0.\label{SGTULA_cont}
\end{align}
The law of this process coincides with the law of the algorithm at grid points i.e., $\mathcal{L}(\bar{\theta}_n^{\lambda})=\mathcal{L}(\theta_{n}^{\lambda})$ for every $n\in\mathbb{N}$. 
\begin{definition}\label{def1} Fix $n_0\in\mathbb{N}$. For each $n\geq n_0$, let $(W^{\lambda,n}_t)_{t\in[0,n_0]}$ denote the Brownian increments $W^{\lambda,n}_t:=\tilde{B}_{t+(n-n_0)}^{\lambda}-\tilde{B}_{(n-n_0)}^{\lambda}$.
\begin{itemize}
    \item The restarted continuous process $(\bar{\zeta}_t^{\lambda,n})_{t\in[0,n_0]}$ is defined by $$d\bar{\zeta}_t^{\lambda,n}=-\lambda h(\bar{\zeta}_t^{\lambda,n})dt+\sqrt{2\lambda\beta^{-1}}dW^{\lambda,n}_t,\quad \bar{\zeta}_0^{\lambda,n}=\theta_{n-n_0}^{\lambda}$$
    \item The restarted interpolation process $({\vartheta}_t^{\lambda,n})_{t\in[0,n_0]}$ is defined by $$d{\vartheta}_t^{\lambda,n}=-\lambda h_{\lambda}({\vartheta}_{\lfloor t \rfloor}^{\lambda,n})dt+\sqrt{2\lambda\beta^{-1}}dW^{\lambda,n}_t,\ {\vartheta}_0^{\lambda,n}=\theta_{n-n_0}^{\lambda}$$
\end{itemize}
\end{definition}
By construction, both processes share the same initial condition and driving noise, and evolve on a common time interval $[0,n_0]$. The restarted interpolation process casts $n_0$ steps of SG-TULA, so that $\mathcal{L}(\vartheta_{n_0}^{\lambda,n})=\mathcal{L}(\bar{\theta}_{n}^{\lambda})$. Additionally, $\mathcal{L}(\bar{\zeta}_{n_0}^{\lambda,n})$ corresponds to the same law as the time-scaled Langevin SDE \eqref{SDE_scaled}, started from the initial condition $\theta_{n-n_0}^{\lambda}$ at time $n-n_0$, and evolves over a fixed time interval of length $n_0$. That is $\mathcal{L}(\bar{\zeta}_{n_0}^{\lambda,n})=\mathcal{L}(\bar{\theta}_{n-n_0}^{\lambda})P_{n_0\lambda}$.
\subsection{Taming Properties}\label{sec_taming}
\begin{remark} \label{remark_growth_unif} For all $\lambda>0$ and  $x\in\mathbb{R}^{d}$, it holds that \begin{align}|h_{\lambda,u}(x)|\leq \mu|x|+{\lambda}^{-1/2} \text{ and } |h_{\lambda,c}(x)|\leq \sqrt2\mu|x|+\sqrt{2d}{\lambda}^{-1/2}.\label{eqR2}\end{align}
\end{remark}
\begin{proof}
    The proof is postponed to Appendix \hyperlink{proof_remark_growth_unif}{\ref*{sec:postponed}}.
\end{proof}
\begin{remark} \label{remark_close_unif} Let Assumption \hyperlink{A3}{A3} hold. Then for all $\lambda>0$ and  $x\in\mathbb{R}^{d}$, it holds that \begin{align}|h_{\lambda}(x)-h(x)|\leq C_1{\lambda}^{1/2}(1+|x|^{2p}),\label{eqR3}\end{align}
\noindent where either $h_{\lambda}=h_{\lambda,u}$ or $h_{\lambda}=h_{\lambda,c}$ and $C_1=2\mu^2+4\max\left\{m^2,K^2\right\}$.
\end{remark}
\begin{proof}
    The proof is postponed to Appendix \hyperlink{proof_remark_close_unif}{\ref*{sec:postponed}}.
\end{proof}
\begin{remark} \label{remark_diss_unif} Let Assumptions \hyperlink{A2}{A2}--\hyperlink{A3}{A3} hold. Then for all $\lambda>0$ and  $x\in\mathbb{R}^{d}$, it holds that \begin{align}\langle x,h_{\lambda,u}(x)\rangle\geq \dfrac{\mu}{2}|x|^2-b,\label{eqR4}\end{align}
\noindent where $b$ is given in Remark \ref{remark_diss}. In addition, let Assumption \hyperlink{A5}{A5} hold. Then for all $\lambda>0$ and  $x\in\mathbb{R}^{d}$, it holds that 
\begin{align*}\langle x,h_{\lambda,c}(x)\rangle\geq \dfrac{\mu}{2}|x|^2-bd.\end{align*}
\end{remark}
\begin{proof}
    The proof is postponed to Appendix \hyperlink{proof_remark_diss_unif}{\ref*{sec:postponed}}.
\end{proof}
\section{Moments \& One-Step Error Estimates}\label{se-mom}
\begin{lemma}\label{lemma_MB}
    Let Assumptions \hyperlink{A2}{A2}--\hyperlink{A3}{A3} and \hyperlink{A4}{A4} hold and $\lambda_0\in(0,1/(4\mu))$. Then, for every $\lambda\in(0,\lambda_0)$ and $n\in\mathbb{N}$, one has
    \begin{align}
        \mathbb{E}|\theta_n^{\lambda}|^2\leq C_{B_1},\label{lemma_MB_eq1}
    \end{align}
    where $C_{B_1}=\mathbb{E}|\theta_0|^2+4(1/\mu)(1+b+d\beta^{-1})$.
\end{lemma}
\begin{proof}
    The proof is postponed to Appendix \hyperlink{proof_lemma_MB}{\ref*{sec:postponed}}.
\end{proof}
Observe that the step-size restriction depends only on $\mu$ and not on the growth of the drift. This is a benefit of taming, by \eqref{eqR2} the tamed drift grows linearly with constant $\mu$, so $\mu$ takes the place of the global Lipschitz constant that the original superlinear drift does not possess.
\begin{lemma}\label{lemma_MB_higher}
    Let Assumptions \hyperlink{A2}{A2}--\hyperlink{A3}{A3} and \hyperlink{A4}{A4} hold and $\lambda_0\in(0,1/(4\mu))$. Then, for every $\lambda\in(0,\lambda_0)$, $p\in\mathbb{N}\setminus\{1\}$ and $n\in\mathbb{N}$, one has
    \begin{align}
        \mathbb{E}|\theta_n^{\lambda}|^{2p}\leq C_{B_2},\label{lemma_MB_higher_eq1}
    \end{align}
    where $C_{B_2}=\mathbb{E}|\theta_0|^{2p}+M(p)(1/\mu)^p(1+b+d\beta^{-1})^p$ with
  $M(p)=p(p+1)2^{6p-2}\left((2p-1)!!\right)^{1/2}$.
\end{lemma}
\begin{proof}
    The proof is postponed to Appendix \hyperlink{proof_lemma_MB_higher}{\ref*{sec:postponed}}.
\end{proof}
\noindent In particular, the dimension enters only through the ratio $d/\beta$.
\begin{lemma}\label{lemma_MB_aux}
    Let Assumptions \hyperlink{A2}{A2}--\hyperlink{A3}{A3} and \hyperlink{A4}{A4} hold and $\lambda_0\in(0,1/(4\mu))$. Then, for every $\lambda\in(0,\lambda_0)$, $p\in\mathbb{N}\setminus\{1\}$, $n_0\in\mathbb{N}$ and $n\geq n_0$ one has \begin{align}
        \sup_{t\in[0,n_0]}\mathbb{E}[V_{2p}(\zedt)]\leq C_{B_3},\label{lemma_MB_aux_eq1}
    \end{align}
    where $C_{B_3}=2^{p/2-1}\left(1+\mathbb{E}|\theta_0|^p+C_{p,\mu,b,\beta,d}\right)$, with $C_{p,\mu,b,\beta,d}=\mathcal{O}((d/\beta)^{p/2})$. Note that $C_{B_3}$ does not depend on $n_0$.
\end{lemma}
\begin{proof}
    The proof is postponed to Appendix \hyperlink{proof_lemma_MB_aux}{\ref*{sec:postponed}}.
\end{proof}
\begin{lemma}\label{lemma_MB_step}
    Let Assumptions \hyperlink{A2}{A2}--\hyperlink{A3}{A3} and \hyperlink{A4}{A4} hold and $\lambda_0\in(0,1/(4\mu))$. Then, for every $\lambda\in(0,\lambda_0)$ and $t\geq0$ one has \begin{align}
        \mathbb{E}|\bar{\theta}_{\lfloor t\rfloor}^{\lambda}-\bar{\theta}_{ t}^{\lambda}|^2\leq C_{B_4}\lambda ,\label{lemma_MB_step_eq1}
    \end{align}
    where $C_{B_4}=\mu C_{B_1}+4+4d\beta^{-1}$. The bound also holds for the restarted interpolation $(\vartheta_t^{\lambda,n})_{t\in[0,n_0]}$ of Definition~\ref{def1}.
\end{lemma}
\begin{proof}
    The proof is postponed to Appendix \hyperlink{proof_lemma_MB_step}{\ref*{sec:postponed}}.
\end{proof}
\section{Contraction Estimates}\label{sec_contraction}
This appendix establishes the log-Sobolev inequality for $\pi_\beta$ and the exponential
contraction of the Langevin diffusion \eqref{SDE_1} in $W_2$.

\subsection{Log-Sobolev Inequality}
\hypertarget{prop_lsi}{}\begin{proposition}\label{prop_lsi}
Let Assumptions \hyperlink{A1}{A1} and \hyperlink{A2}{A2} hold. Then the Langevin SDE
\eqref{SDE_1} admits a unique invariant measure $\pi_\beta\propto e^{-\beta u}$ which
satisfies a logarithmic Sobolev inequality with some constant $\rho(\beta)>0$, for every
$\phi\in C^1_b(\mathbb{R}^d)$ with $\int_{\mathbb{R}^d}\phi^2\,d\pi_\beta=1$,
\begin{align}
    \int_{\mathbb{R}^d}\phi^2\log\phi^2\,d\pi_\beta
    \;\leq\; \rho(\beta)\int_{\mathbb{R}^d}|\nabla\phi|^2\,d\pi_\beta .
    \label{eq_lsi}
\end{align}
\end{proposition}
\begin{proof}
Under \hyperlink{A1}{A1} the potential $\beta u$ is
semi-convex with $\mathrm{Hess}(\beta u)\succeq -\beta L\,\mathrm{I_d}$ distributionally,
and \hyperlink{A2}{A2} provides a dissipativity at infinity condition. Thus the LSI follows from
\citet[Corollary~2.1]{CattiauxGuillinWu2010}, with the HWI step of the underlying
Theorem~1.9 justified for our non-smooth $\beta u$ by \citet[Theorem~9.17]{Villani2003}
as explained in Remark~\ref{LSI_remark}.
\end{proof}
\begin{remark}\label{LSI_remark}
Proposition~\hyperlink{prop_lsi}{\ref{prop_lsi}} follows from the curvature of the
potential alone. The principle goes back to \cite{Wang2001}, where, in the smooth
case, a negative uniform lower bound on the curvature (semi-convexity,
\hyperlink{A1}{A1}) together with finite exponential moments of $\pi_\beta$
suffices for a LSI.

However we invoke the result of \citet[Corollary~2.1]{CattiauxGuillinWu2010}
instead, which derives the LSI from a Lyapunov (dissipativity) condition together
with the curvature lower bound $\mathrm{Hess}(V)\ge K\,\mathrm{I_d}$ (here
$V=\beta u$, $K=-\beta L$, i.e.,\ \hyperlink{A1}{A1}) and reduces to Theorem~1.9
therein, combining a defective transport inequality with the HWI inequality.
The only smoothness requirement enters through the HWI inequality quoted from \cite{OttoVillani2000}, which assumes $V\in C^2$.
Replacing that reference by \citet[Theorem~9.17]{Villani2003} establishes the same
HWI inequality assuming only a negative lower bound for the distributional hessian. Their argument uses the fact that the Alexandrov Hessian is dominated by the
distributional one. The remaining hypotheses of \citet[Theorem~9.17]{Villani2003}
hold here, that is, $\int V e^{-V}<\infty$ and $\pi_\beta\in P_{ac,2}(\mathbb{R}^d)$
follow from the moment bound $\int_{\mathbb{R}^d}V_{2p}\,d\pi_\beta\le
C_4(2p)/C_3(2p)$ established in the proof of Lemma~\hyperlink{lemmaER1_proof}{10}.
\end{remark}

Notice that \hyperlink{A2}{A2} in turn implies dissipativity at infinity, and,
together with \hyperlink{A3}{A3}, the global dissipativity of
Remark~\ref{remark_diss}. In fact \hyperlink{A1}{A1} together with dissipativity
already suffices for the log-Sobolev inequality, so \hyperlink{A2}{A2} could be
dropped in favour of a dissipativity assumption. We nevertheless retain the
strong convexity at infinity \hyperlink{A2}{A2} for two reasons. In view of
Lemma~\ref{lem:pr-chord} it is straightforward to verify, even for certain neural
network architectures, and the dissipativity route yields a log-Sobolev constant
that is in general exponential in $d$, whereas the addition of \hyperlink{A2}{A2}--\hyperlink{A3}{A3} gives the
dimension-free $\rho(\beta)$ of Remark~\ref{remark_lsi_const}.

\begin{remark}\label{remark_lsi_const}
Let Assumptions \hyperlink{A1}{A1}, \hyperlink{A2}{A2} and \hyperlink{A3}{A3} hold. Then the
log-Sobolev inequality \eqref{eq_lsi} holds with
$$\rho(\beta)\leq \dfrac{2}{\beta\mu}\exp\!\Big(\beta\Big[\,4R\,(m+KR^{p})+\tfrac{\mu}{2}R^{2}\,\Big]\Big).$$
\end{remark}
\begin{proof}
The key idea rests on a perturbation argument that requires no
smoothness of the potential. A $\beta\mu$-strongly convex proxy $W_0$ is constructed, which
coincides with $\beta u$ outside $\mathcal{B}(0,R)$ (Lemma~\ref{lem:proxy}). The reference
measure $\pi_0\propto e^{-W_0}$ satisfies a log-Sobolev inequality with constant $2/(\beta\mu)$
by the Bobkov--Ledoux criterion \citep[Proposition~3.1]{BobkovLedoux}, which, in contrast to the
Bakry--\'Emery criterion, applies to non-smooth strongly convex potentials. Since the
perturbation $\psi=\beta u-W_0$ is supported in $\overline{\mathcal{B}}(0,R)$ with
$\operatorname{osc}(\psi)\leq\beta[4R(m+KR^p)+\tfrac{\mu}{2}R^2]$ (Lemma~\ref{lem:lsi}), the
Holley--Stroock perturbation principle (see Theorem~1, Chapter~5 in \citealp{Bakry_ref}) yields the
stated constant.
\end{proof}
Strong convexity at infinity (\hyperlink{A2}{A2}) implies that the drift is dissipative at infinity, see Proof of Remark~\ref{remark_diss}. This yields a Lyapunov condition ensuring that the invariant measure has finite exponential moments, up to a constant in the exponent. Semi-convexity  (\hyperlink{A1}{A1}) provides a negative uniform lower bound on the curvature of the potential. Combined, they allow one to establish a LSI for $\pi$, see \cite{Wang2001} and references therein. Under the growth condition  (\hyperlink{A3}{A3}), the LSI constant can be explicitly controlled using the Holley--Stroock perturbation theorem together with the Bakry--Emery criterion. Similar arguments in the smooth potential setting can be found in \cite{Flammarion_supp}.

Therefore, the significance of Remark~\ref{remark_lsi_const} is primarily practical. Verifying a LSI directly for a given target measure is often infeasible beyond idealized settings, whereas strong convexity at infinity is typically much easier to establish and can often be checked by algebraic arguments. This includes many nontrivial examples of interest, such as mixture of Gaussians or double-well potentials, which are nonconvex but strongly convex outside a compact set.
\begin{remark}[Extra assumptions for polynomial dependence]\label{remark_beta_poly}
    Although in the general case one cannot avoid exponential dependence on $\beta$, an additional structure can yield polynomial dependence on $\beta$. Earlier works proved this result under smoothness assumptions and strict conditions on the local minima and saddle points in compact Riemannian manifolds \citep{Li2023} and Euclidean spaces 
\citep{Lytras_Sabanis_2025,Kinoshita_Suzuki_2022}. Later work obtained polynomial dependence under the PL condition i.e., for some $C_{PL}>0$
\begin{equation}\label{eq-PL}
    u(x)-u^*\leq C_{PL} |\nabla u(x)|^2\quad \forall x\in \mathbb{R}^d
\end{equation}
In \cite{chewiballistic} it was proved that for $u\in {C}^2$ such that $\Delta u\leq C(1+||\nabla u||^2)$ satisfying \eqref{eq-PL}, having a unique global minimizer there holds
\[\beta \rho(\beta)\rightarrow C_{PL} \quad \text{as} \quad \beta \rightarrow \infty.\]
This was further generalized for multiple global minimizers in \cite{LSI_lyap} where it was shown that under \eqref{eq-PL} 
and some dissipativity and growth conditions on the first and second derivatives, that the target measure satisfies a Poincar\'e inequality (PI) that depends on the local Poincar\'e constant in a ball around the minimizers set $B(W^*,r)$ for some $r>0$. Under the additional assumption of $2$-dissipativity and semi-convexity it was shown (see Theorem 3.1 in \citealp{LSI_lyap}) that the target measure satisfies LSI with constant \[\rho(\beta)\leq \beta^2 C_{PI, B(W^*,r)}.\]
This underpins the connection between the global LSI constant and the local behavior near the minimizers. More specifically, if the measure admits a Poincar\'e constant in $B(W^*,r)$ that is independent of $\beta$ (for example if $B(W^*,r)$ is convex and $u$ is convex on this set), the resulting LSI constant is polynomial in $\beta$.
\end{remark}
\subsection{Exponential Contraction}
We follow through the proof of Theorem~2.1(3) in \citet{Wang2020}, in which
\hyperlink{A1}{A1}--\hyperlink{A2}{A2} are assumed for a smooth drift $h\in C^2$. The
reason for reproducing the argument is twofold: to identify the explicit constants, and
to accommodate our setting, in which $h$ is a measurable selection of $\partial u$ and
therefore discontinuous. For the latter we use a convolution smoothing argument, allowing us to invoke tools such as the log-Harnack inequality at the level of a smooth
potential.

\begin{lemma}\label{lemma_mollify}
Let $\varphi\in C_c^\infty(\mathbb{R}^d)$ satisfy
\begin{align}
\varphi\geq0,\qquad
\operatorname{supp}\varphi\subseteq\overline{\mathcal{B}}(0,1),\qquad
\int_{\mathbb{R}^d}\varphi(w)\,dw=1,\label{eq_kernel}
\end{align}
and for $\epsilon\in(0,1]$ define $\varphi_\epsilon(w):=\epsilon^{-d}\varphi(w/\epsilon)$,
which inherits the properties \eqref{eq_kernel} with
$\operatorname{supp}\varphi_\epsilon\subseteq\overline{\mathcal{B}}(0,\epsilon)$. Set
\begin{align}
h_\epsilon(x):=(h*\varphi_\epsilon)(x)=\int_{\mathbb{R}^d}h(x-w)\varphi_\epsilon(w)\,dw,
\qquad
u_\epsilon(x):=(u^*\varphi_\epsilon)(x).\label{eq_mollified}
\end{align}
Let Assumptions \hyperlink{A1}{A1}--\hyperlink{A3}{A3} hold. Then:
\begin{enumerate}
\item[(i)] $u_\epsilon,h_\epsilon\in C^\infty(\mathbb{R}^d)$ and
$\nabla u_\epsilon=h_\epsilon$;
\item[(ii)] $h_\epsilon$ satisfies \hyperlink{A1}{A1} with the same constant $L$ and
\hyperlink{A2}{A2} with the same constants $\mu$ and $R$;
\item[(iii)] $h_\epsilon$ satisfies \hyperlink{A3}{A3} in the form
\begin{align*}
\left|h_\epsilon(x)\right|\leq m+K\left(|x|+\epsilon\right)^{p}
\leq m+K\left(|x|+1\right)^{p},
\qquad\forall x\in\mathbb{R}^d,\ \forall\epsilon\in(0,1];
\end{align*}
\item[(iv)] the measure $\pi_\beta^\epsilon\propto e^{-\beta u_\epsilon}$ satisfies
\eqref{eq_lsi} with constant
\begin{align}
\rho_\epsilon(\beta)=\dfrac{2}{\beta\mu}
\exp\left(\beta\left[4R\left(m+K(R+\epsilon)^{p}\right)+\dfrac{\mu}{2}R^{2}\right]\right);
\label{eq_rho_eps}
\end{align}
\item[(v)] as $\epsilon\to0$, $u_\epsilon\to u$ locally uniformly and
$h_\epsilon\to h$ Lebesgue-almost everywhere on $\mathbb{R}^d$.
\end{enumerate}
\end{lemma}

\begin{proof}
The proof is postponed to Appendix \hyperlink{proof_lemma_mollify}{\ref*{sec:postponed}}
\end{proof}
\noindent Define the time-scaled process $Y_s:=Z_{\beta s}$ and $\hat{B}_s:=B_{\beta s}/\sqrt{\beta}$, then for
$\epsilon\in(0,1]$ let $Y^\epsilon$ be the corresponding process with the mollified drift
$h_\epsilon$ of Lemma \ref{lemma_mollify}. Then
\begin{align}
dY_s&=-\beta h(Y_s)\,ds+\sqrt{2}\,d\hat{B}_s,
&&Y_0=Z_0,\label{eq_rescaled_sde}\\
dY_s^\epsilon&=-\beta h_\epsilon(Y_s^\epsilon)\,ds+\sqrt{2}\,d\hat{B}_s,
&&Y_0^\epsilon=Z_0,\label{eq_rescaled_sde_eps}
\end{align}
with semigroups $(Q_s)_{s\geq0}$ and $(Q_s^\epsilon)_{s\geq0}$, the latter admitting the
unique invariant measure $\pi_\beta^\epsilon\propto e^{-\beta u_\epsilon}$.

\begin{lemma}\label{lemma_stability}
Let Assumptions \hyperlink{A1}{A1}--\hyperlink{A4}{A4} hold and let $Y,Y^\epsilon$ solve
\eqref{eq_rescaled_sde} and \eqref{eq_rescaled_sde_eps} respectively. Then
\begin{align}
\lim_{\epsilon\to0}\mathbb{E}\left|Y_s-Y_s^\epsilon\right|^{2}=0
\qquad\text{for every }s\geq0,
\end{align}
and consequently $W_2\left(\nu Q_s,\nu Q_s^\epsilon\right)\to0$.
\end{lemma}

\begin{proof}
The proof is postponed to Appendix \hyperlink{proof_lemma_stability}{\ref*{sec:postponed}}
\end{proof}

\begin{lemma}\label{lemma_pi_conv}
Under Assumptions \hyperlink{A1}{A1}--\hyperlink{A3}{A3} one has
$W_2(\pi_\beta^\epsilon,\pi_\beta)\to0$ as $\epsilon\to0$.
\end{lemma}

\begin{proof}
The proof is postponed to Appendix \hyperlink{proof_lemma_pi_conv}{\ref*{sec:postponed}}
\end{proof}
\begin{proposition}\label{prop_contraction1}
Let Assumptions \hyperlink{A1}{A1}--\hyperlink{A4}{A4} hold and let $\rho(\beta)$ be as in
Remark \ref{remark_lsi_const}. Then $\pi_\beta$ is the unique invariant measure of
\eqref{SDE_1} and, for every $t\geq0$ and $\nu\in\mathcal{P}_2(\mathbb{R}^d)$,
\begin{align}
W_2\left(\nu P_t,\pi_\beta\right)\leq C_we^{-C_rt}W_2\left(\nu,\pi_\beta\right),
\qquad C_r=\dfrac{2}{\beta\rho(\beta)},
\label{eq_contraction_main}
\end{align}
where, $t_\star=\dfrac{1}{2L}\log\left(1+\dfrac{\beta L\rho(\beta)}{2}\right)$, \text{ and}
\begin{align}
C_w=\max\left\{
e^{\left(L+\frac{2}{\beta\rho(\beta)}\right)t_\star},\;
e^{\frac{2t_\star}{\beta\rho(\beta)}}
\sqrt{\dfrac{\beta L\rho(\beta)}{2\left(1-e^{-2Lt_\star}\right)}}\,\right\}
\leq\sqrt{e\left(1+\dfrac{\beta L\rho(\beta)}{2}\right)}.
\label{eq_Cw_statement}
\end{align}
\end{proposition}

\begin{proof}
The proof is postponed to Appendix \hyperlink{proof_prop_contraction1}{\ref*{sec:postponed}}
\end{proof}

\begin{remark}\label{remark_convex_case}
In the convex case, i.e., $L=0$, one repeats the calculation so that \eqref{eq_contraction_main} reads as
$$W_2\left(\nu P_t,\pi_\beta\right)\leq\sqrt{e}\,e^{-2t/(\beta\rho(\beta))}
W_2\left(\nu,\pi_\beta\right)$$
\end{remark}

\section{Discretization Error Estimates}
\noindent For the remainder of the analysis we fix the restart horizon of
Definition~\ref{def1} as
\begin{align}
n_0:=\left\lceil \dfrac{\ln C_w+1}{C_r\lambda}\right\rceil,\qquad\text{so that}\qquad
\lambda n_0\leq \dfrac{\ln C_w+1}{C_r}+\lambda_0=:T_0,\label{def_horizon}
\end{align}
where $C_w,C_r$ are the contraction constants of Proposition~\ref{prop_contraction1}.
\begin{lemma}\label{lemma_err_W2}
    Let Assumptions \hyperlink{A1}{A1}--\hyperlink{A4}{A4} hold and $\lambda_0\in(0,1/(4\mu))$. Then, for every $\lambda\in(0,\lambda_0)$, $t\in[0,n_0]$ and $n\geq n_0$, with $n_0$ as in
  \eqref{def_horizon}, one has \begin{align}
        W_2\left(\mathcal{L}(\vartheta_t^{\lambda,n}),\mathcal{L}(\bar{\zeta}_t^{\lambda,n})\right)\leq C_{E}\lambda^{1/4},\label{lemma_err_W2_eq1}
    \end{align}
    where $C^2_{E}=e^{5LT_0}\,T_0\Bigl(2LC_{B_4}/\sqrt{\mu}+L^{-1}C_1(1+C_{B_2})+4\big(\max\{2m^2,K^2\}(1+C_{B_2}+C_{B_3})C_{B_4}\big)^{1/2}\Bigr)$.
\end{lemma}
\begin{proof}
    The proof is postponed to Appendix \hyperlink{proof_lemma_err_W2}{\ref*{sec:postponed}}.
\end{proof}
\section{Excess Risk} 
\begin{lemma}
\label{lemmaER1}
Let Assumptions \hyperlink{A1}{A1}--\hyperlink{A4}{A4} hold and $\lambda_0\in(0,1/(4\mu))$. Then, for any $\lambda\in(0,\lambda_0)$ and $n\in\mathbb{N}$, the following bound for $\mathcal{T}_1=\mathbb{E}[u(\theta_n^{\lambda})]-\mathbb{E}[u(\theta_{\infty})]$, holds
\begin{align}
    \mathcal{T}_1\leq C_{\mathcal{T}_1}W_2\left(\pi_{\beta},\mathcal{L}(\theta_n^{\lambda})\right),
\end{align}
where $C_{\mathcal{T}_1}=m+\dfrac{2^{p-1}K}{p+1}\left(\sqrt{C_{B_2}}+A(2p)^{p/2}\right)$.
\end{lemma}
\begin{proof}
    The proof is postponed to Appendix \hyperlink{lemmaER1_proof}{\ref*{sec:postponed}}.
\end{proof}
\begin{lemma}
\label{lemmaER2}
Let Assumptions \hyperlink{A1}{A1}--\hyperlink{A4}{A4} hold and $\lambda_0\in(0,1/(4\mu))$.  Then, for any $\lambda\in(0,\lambda_0)$, $n\in\mathbb{N}$ and $\beta\geq \max\{4/\mu,\,1/J\}$, the following bound for $\mathcal{T}_2=\mathbb{E}[u(\theta_{\infty})]-u(\theta^{\star})$, holds
\begin{align}
    \mathcal{T}_2\leq \mathbb{E}[u(\theta_{\infty})]-u(\theta^{\star})\leq
  \dfrac{d}{2\beta}\log\left(\dfrac{2(b+d/\beta)\beta^2J^2}{\mu }\right)+\dfrac{1}{2\beta}\log\left(\pi d\right)+\dfrac{13}{6\beta}=:C_{\mathcal{T}_2},
  \end{align} where $ J=m+K2^{2p-2}(1+(2b/\mu)^{p/2})+K2^{p-1}/(p+1) $.
\end{lemma}
\begin{proof}
    The proof is postponed to Appendix \hyperlink{lemmaER2_proof}{\ref*{sec:postponed}}.
\end{proof}
\section{Excess-Risk Through Optimization Arguments Under Non-Smooth PL}
Throughout this section $x_\epsilon$ denotes the proximal point and the corresponding envelope is defined as
\[u_\epsilon(x)=u(x_\epsilon)+\frac{1}{2\epsilon}|x-x_\epsilon|^2=\operatorname*{min}_y \{u(y)+\frac{1}{2\epsilon}|x-y|^2\}.\]
Consequently, $u_\epsilon\leq u$ and $u_\epsilon$ attains the same minimum as $u$. Moreover $u_\epsilon$ is $\frac{1}{\epsilon}$-smooth such that 
\begin{align*}\nabla u_\epsilon(x)=(1/\epsilon)(x-x_\epsilon)\in \partial u(x_\epsilon).
\end{align*}
\par\begin{lemma}\label{MY-PL}
    Let $u_{\epsilon}$ be the MY envelope with parameter $\epsilon<1/K$ satisfying \eqref{eq-nonsmoothPL}. There holds 
    \[u_\epsilon-u^*\leq \frac{1+\epsilon C_{PL}}{2C_{PL}}|\nabla u_\epsilon|^2. \]
\end{lemma}
  \begin{proof}
The proof is postponed to Appendix~\hyperlink{proof_lemma_MY-PL}{\ref*{sec:postponed}}.
\end{proof}
\begin{lemma}\label{lemma-MY-bound} Let \hyperlink{A1}{A1} and \hyperlink{A3}{A3} hold and $\epsilon\leq 1/K$. Then for every element $g\in \partial u(x)$ we have \begin{align}\label{eq-MY-bound} \langle \nabla u_\epsilon(x),g\rangle\geq \frac{1}{2}|\nabla u_\epsilon|^2, \text{ and }\\ |\nabla u_\epsilon(x)|\leq 2|g|\leq 2C(1+|x|^p).\label{eq-subdomMY} \end{align} \end{lemma}
\begin{proof}
   The proof is postponed to Appendix~\hyperlink{proof_lemma-MY-bound}{\ref*{sec:postponed}}.
\end{proof}
\begin{lemma}\label{approximation MY}
    The approximation error of the MY envelope satisfies \[u(x)-u_\epsilon(x)\leq C\epsilon(1+|x|^p). \]
\end{lemma}
\begin{proof}
The proof is postponed to Appendix~\hyperlink{proof_lemma-approximation MY}{\ref*{sec:postponed}}. 
\end{proof}

\section{Postponed Proofs}\label{sec:postponed}
This appendix collects the proofs postponed from the main text and from the preceding
appendices.

\subsection{Proof of Remark~\ref{remark_diss}}\hypertarget{proof_remark_diss}{}
\begin{proof}
    Let $|x|\geq R$, then through \hyperlink{A2}{A2} one obtains
    \begin{align}
        \langle x,h(x) \rangle&=\langle x-0,h(x)-h(0)\rangle +\langle x,h(0)\rangle\geq \mu|x|^2-|x||h(0)|\geq \dfrac{\mu}{2}|x|^2-\dfrac{|h(0)|^2}{2\mu}.\label{eqR11}
    \end{align}
    Now let $|x|<R$, due to the polynomial growth in \hyperlink{A3}{A3} one writes
    \begin{align}
        \langle x,h(x)\rangle &\geq -|x||h(x)|\geq -m|x|-K|x|^{p+1}\geq -mR-KR^{p+1}+\dfrac{\mu}{2}R^2-\dfrac{\mu}{2}R^2\nonumber\\
        &\geq \dfrac{\mu}{2}|x|^2-\left(mR+(\mu/2)R^2+KR^{p+1}\right).\label{eqR12}
    \end{align}
    Combining \eqref{eqR11} and \eqref{eqR12} yield \eqref{eqR1}, where $b_1=\max\left\{|h(0)|^2/(2\mu),mR+(\mu/2)R^2+KR^{p+1}\right\}$.
    
    \noindent Furthermore, if one assumes \hyperlink{A1}{A1}, the dependency on $R$ becomes quadratic. For $|x|<R$, by \hyperlink{A1}{A1} for $y=0$
    \begin{align*}
        \langle x,h(x)\rangle &\geq \langle x-0,h(x)-h(0)\rangle +\langle x,h(0)\rangle\geq -L|x|^2 -|x||h(0)|\\
        &\geq (\mu/2)|x|^2-\left((L+\mu/2)R^2+mR\right),
    \end{align*}
    and therefore $b_2=\max\left\{m^2/(2\mu),mR+(L+\mu/2)R^2\right\}$
\end{proof}
\subsection{Proof of Remark~\ref{remark_growth_unif}}\hypertarget{proof_remark_growth_unif}{}
\begin{proof}
     One calculates directly from \eqref{tame_unif}
    \begin{align*}
        |h_{\lambda,u}(x)|\leq \mu|x|+\lambda^{-1/2}\dfrac{\lambda^{1/2}|h(x)-\mu x|}{1+\lambda^{1/2}|h(x)-\mu x|}\leq\mu|x|+\lambda^{-1/2}.
    \end{align*}
    Furthermore, 
    \begin{align*}
        |h_{\lambda,c}(x)|=\sqrt{\sum_{i=1}^d|h_{\lambda,c}^{(i)}(x)|^2}\leq\sqrt{\sum_{i=1}^d2\mu^2|x^{(i)}|^2+2\lambda^{-1}}\leq\sqrt{2}\mu|x|+\sqrt{2d}\lambda^{-1/2}.
    \end{align*}
\end{proof}
\subsection{Proof of Remark~\ref{remark_close_unif}}\hypertarget{proof_remark_close_unif}{}
\begin{proof}
    In view of \hyperlink{A1}{A1}, one calculates
    \begin{align*}
        |h_{\lambda,u}(x)-h(x)|&\leq\left|\dfrac{h(x)-\mu x}{1+\lambda^{1/2}|h(x)-\mu x|}-\left(h(x)-\mu x\right)\right|\\
        &\leq|h(x)-\mu x|\dfrac{\lambda^{1/2}|h(x)-\mu x|}{1+\lambda^{1/2}|h(x)-\mu x|}\\
        &\leq \lambda^{1/2}|h(x)-\mu x|^2\leq\lambda^{1/2}\left(2|h(x)|^2+2\mu^2|x|^2\right)\\
        &\leq \lambda^{1/2}\left(4m^2+2\mu^2|x|^2+4K^2|x|^{2p}\right)\\
        &\leq \lambda^{1/2}C_1(1+|x|^{2p}),        
    \end{align*}
    where $C_1=2\mu^2+4\max\left\{m^2,K^2\right\}$, using $|x|^2\leq 1+|x|^{2p}$ for $p\geq1$. Furthermore,
    \begin{align*}
        |h_{\lambda,c}(x)-h(x)|&=\sqrt{\sum_{i=1}^d|h_{\lambda,c}^{(i)}(x)-h^{(i)}(x)|^2}\leq \lambda^{1/2}\sqrt{\sum_{i=1}^d|h^{(i)}(x)-\mu x^{(i)}|^4}\\
        &\leq \lambda^{1/2}\sum_{i=1}^d|h^{(i)}(x)-\mu x^{(i)}|^2=\lambda^{1/2}|h(x)-\mu x|^2\leq\lambda^{1/2}C_1(1+|x|^{2p}).
    \end{align*}
\end{proof}
\subsection{Proof of Remark~\ref{remark_diss_unif}}\hypertarget{proof_remark_diss_unif}{}
\begin{proof}
    One writes
    \begin{align*}
        \langle h_{\lambda,u}(x),x\rangle=\left\langle \dfrac{h(x)-\mu x}{1+\lambda^{1/2}|h(x)-\mu x|},x\right\rangle+\mu |x|^2.
    \end{align*}
    If $\langle h(x)-\mu x,x\rangle \geq 0$, then one immediately gets $\langle h_{\lambda,u}(x),x\rangle \geq \mu|x|^2\geq (\mu/2)|x|^2-b$. Otherwise, we calculate
    \begin{align*}
        \left\langle \dfrac{h(x)-\mu x}{1+\lambda^{1/2}|h(x)-\mu x|},x\right\rangle&\geq-\left|\left\langle \dfrac{h(x)-\mu x}{1+\lambda^{1/2}|h(x)-\mu x|},x\right\rangle\right|\geq -\left|\left\langle h(x)-\mu x,x \right\rangle\right|\\
        &=\langle h(x)-\mu x,x\rangle =\langle h(x),x\rangle -\mu|x|^2\geq-(\mu/2)|x|^2-b,
    \end{align*}
    where the last inequality holds due to Remark \ref{remark_diss}. By combining the above bounds, one derives the result in \eqref{eqR4}. Furthermore, for the coordinate-wise case, we consider the following scenarios:
\begin{itemize}
    \item If $h^{(i)}(x)x^{(i)}-\mu |x^{(i)}|^2\geq0$ for every $i\in\{1,\ldots,d\}$, then we trivially obtain $h^{(i)}_{\lambda,c}(x)x^{(i)}\geq \mu |x^{(i)}|^2\Rightarrow \langle h_{\lambda,c}(x),x\rangle \geq \mu|x|^2\geq (\mu/2)|x|^2-b$.
    \item If $h^{(i)}(x)x^{(i)}-\mu |x^{(i)}|^2\leq 0$ for every $i\in\{1,\ldots,d\}$, then we have
    \begin{align*}
        \sum_{i=1}^d\dfrac{h^{(i)}(x)-\mu x^{(i)}}{1+\lambda^{1/2}|h^{(i)}(x)-\mu x^{(i)}|}x^{(i)}&\geq -\sum_{i=1}^d\left|\dfrac{h^{(i)}(x)-\mu x^{(i)}}{1+\lambda^{1/2}|h^{(i)}(x)-\mu x^{(i)}|}x^{(i)}\right|\\&\geq \sum_{i=1}^d(h^{(i)}(x)-\mu x^{(i)})x^{(i)}\geq \langle h(x),x \rangle -\mu |x|^2\\&\Rightarrow \langle h_{\lambda,c}(x),x \rangle\geq (\mu/2)|x|^2-b.
    \end{align*}
    \item If $h^{(i)}(x)x^{(i)}-\mu |x^{(i)}|^2\geq0$ for every $i\in J \subset\{1,\ldots,d\}$, then we have 
    \begin{align*}
        \langle h_{\lambda,c}(x),x \rangle&= {\langle h_{\lambda,c}(x),x \rangle}_{J}+{\langle h_{\lambda,c}(x),x \rangle}_{J^c}\geq \mu|x|^2_{J}+{\langle h(x),x \rangle}_{J^c}\\&\geq \mu|x|^2_{J}+\dfrac{\mu}{2}|x|^2_{J^c}-b\,|J^c|\geq \dfrac{\mu}{2}|x|^2-bd,
    \end{align*}
    where the second inequality applies \eqref{eqA5} coordinate-wise on $J^c$.
\end{itemize}
\end{proof}
\subsection{Proof of Lemma~\ref{lemma_MB}}\hypertarget{proof_lemma_MB}{}
\begin{proof}
    We first observe that $\mathbb{E}|\theta_n^{\lambda}|^2<\infty$ for every $n$, by induction from Assumption \hyperlink{A4}{A4} and the linear growth \eqref{eqR2}. Consider the SG-TULA iterates $(\theta_n^{\lambda})_{n\geq 0}$ from \eqref{SGTULA}
    \begin{align}
        |\theta_{n+1}^{\lambda}|^2=|\theta_{n}^{\lambda}-\lambda h_{\lambda}(\theta_n^{\lambda})|^2+2\lambda\beta^{-1}|\xi_{n+1}|^2+2\sqrt{2\lambda\beta^{-1}}\langle\theta_n^{\lambda}-\lambda h_{\lambda}(\theta_n^{\lambda}),\xi_{n+1}\rangle\nonumber    
    \end{align}
    Since $\theta_n^{\lambda}$ is independent of $\xi_{n+1}$, the last term vanishes under expectation. Thus by taking conditional expectations on both sides and using Remark \ref{remark_growth_unif} and \ref{remark_diss_unif}, we obtain
    \begin{align}
        \mathbb{E}[|\theta_{n+1}^{\lambda}|^2|\theta_n^{\lambda}]&=\mathbb{E}[|\theta_n^{\lambda}|^2|\theta_n^{\lambda}]-2\lambda\mathbb{E}[\langle \theta_n^{\lambda},h_{\lambda}(\theta_n^{\lambda})\rangle|\theta_n^{\lambda}]+\lambda^2\mathbb{E}[|h_{\lambda}(\theta_n^{\lambda})|^2|\theta_n^{\lambda}]+2\lambda d\beta^{-1}\nonumber\\ 
        &\leq |\theta_n^{\lambda}|^2-\lambda\mu|\theta_n^{\lambda}|^2+2\lambda^2\mu^2|\theta_n^{\lambda}|^2+2\lambda b+2\lambda+2\lambda d\beta^{-1}\nonumber\\
        &\leq (1-\lambda\mu+2\lambda^2\mu^2)|\theta_n^{\lambda}|^2+2\lambda(b+1+d\beta^{-1})\label{lem1eq1}
    \end{align}
    Recalling the restriction $\lambda<\lambda_0<1/(4\mu)$, iterating inequality \eqref{lem1eq1} yields
    \begin{align}
        \mathbb{E}|\theta_{n}^{\lambda}|^2&\leq (1-\lambda\mu/2)^n\mathbb{E}|\theta_0^{\lambda}|^2+\dfrac{1-(1-\lambda\mu/2)^n}{\lambda\mu/2}2\lambda(b+1+d\beta^{-1})\nonumber\\
        &\leq \mathbb{E}|\theta_0|^2+4(b+1+d\beta^{-1})/\mu.\nonumber
    \end{align}
\end{proof}
\subsection{Proof of Lemma~\ref{lemma_MB_higher}}\hypertarget{proof_lemma_MB_higher}{}
\begin{proof}
Consider the SG-TULA iterates $(\theta_n^{\lambda})_{n\geq 0}$ from \eqref{SGTULA} and define the following auxiliary processes
\begin{align}
    \Delta_n^{\lambda}:=\theta_n^{\lambda}-\lambda h_{\lambda}(\theta_n^{\lambda}),\ G_n^{\lambda}:=\sqrt{2\lambda\beta^{-1}}\xi_{n+1}.\nonumber
\end{align}
Furthermore, define $A_n^{\lambda}:=|\Delta_n^{\lambda}|^2$ and $B_n^{\lambda}:=2\langle\Delta_n^{\lambda},G_n^{\lambda}\rangle+|G_n^{\lambda}|^2$. For the $2p$-th moment, we write
\begin{align*}
    |\theta_{n+1}^{\lambda}|^{2p}=(A_n^{\lambda}+B_n^{\lambda})^p, \nonumber
\end{align*}\vspace{-1cm}\begin{align}
\mathbb{E}[|\theta_{n+1}^{\lambda}|^{2p}\mid\theta_{n}^{\lambda}]\leq \mathbb{E}[(A_n^{\lambda})^p\mid\theta_{n}^{\lambda}]+\mathbb{E}[2p(A_n^{\lambda})^{p-1}B_n^{\lambda}\mid\theta_{n}^{\lambda}]+\mathbb{E}\left[\sum_{k=2}^p\binom{p}{k}|A_n^{\lambda}|^{p-k}|B_n^{\lambda}|^k\mid\theta_{n}^{\lambda}\right].\label{lem2eq1}
\end{align}
Since Lemma \ref{lemma_MB} holds, using bound \eqref{lem1eq1}, we estimate the first term of \eqref{lem2eq1}
\begin{align}
    \mathbb{E}[(A_n^{\lambda})^p\mid\theta_{n}^{\lambda}]&=(A_n^{\lambda})^p\leq\left((1-\lambda\mu/2)|\theta_n^{\lambda}|^2+2\lambda(b+1)\right)^p\nonumber\\
    &\leq\left(1+\lambda\mu/4\right)^{p-1}\left(1-\lambda\mu/2\right)^{p}|\theta_n^{\lambda}|^{2p}+\left(1+4\lambda^{-1}\mu^{-1}\right)^{p-1}(2\lambda(b+1))^p\nonumber\\
    &\leq (1-\lambda\mu/4)^{p-1}(1-\lambda\mu/2)|\theta_n^{\lambda}|^{2p}+\lambda(\lambda+4/\mu)^{p-1}(2(1+b))^p\nonumber\\
    &\leq r^{\lambda}_p|\theta_n^{\lambda}|^{2p}+w^{\lambda}_p\label{lem2eq2},
\end{align}
where $r^{\lambda}_p=(1-\lambda\mu/4)^{p-1}(1-\lambda\mu/2)$ and $w^{\lambda}_p=\lambda(\lambda+4/\mu)^{p-1}(2(1+b))^p$. Note that we applied the elementary inequality $(a+b)^p\leq(1+\epsilon)^{p-1}a^p+(1+1/\epsilon)^{p-1}b^p$ with $\epsilon=\lambda\mu/4$. We similarly estimate the second term of \eqref{lem2eq1} via bound \eqref{lem2eq2}
\begin{align}
    \mathbb{E}[2p(A_n^{\lambda})^{p-1}B_n^{\lambda}\mid\theta_{n}^{\lambda}]&\leq 2p(A_n^{\lambda})^{p-1}\mathbb{E}[B_n^{\lambda}\mid\theta_{n}^{\lambda}]\leq 4p\lambda d\beta^{-1}(A_n^{\lambda})^{p-1}\nonumber\\
    &\leq 2p(2\lambda d\beta^{-1})(r^{\lambda}_{p-1}|\theta_n^{\lambda}|^{2(p-1)}+w^{\lambda}_{p-1}).\label{lem2eq3}
\end{align}
The third term in inequality \eqref{lem2eq1} expands as follows
\begin{align}
    \sum_{k=2}^p\binom{p}{k}|A_n^{\lambda}|^{p-k}|B_n^{\lambda}|^k&=\sum_{m=0}^{p-2}\binom{p}{m+2}|A_n^\lambda|^{p-2-m}|B_n^{\lambda}|^{m+2}\nonumber\\
    &=\sum_{m=0}^{p-2}\dfrac{p(p-1)}{(m+2)(m+1)}\binom{p-2}{m}|A_n^\lambda|^{p-2-m}|B_n^{\lambda}|^{m}|B_n^{\lambda}|^{2}\nonumber\\
    &\leq p(p-1)|B_n^{\lambda}|^2\sum_{m=0}^{p-2}\binom{p-2}{m}|A_n^\lambda|^{p-2-m}|B_n^{\lambda}|^{m}\nonumber\\
    &= p(p-1)(|A_n^{\lambda}|+|B_n^{\lambda}|)^{p-2}|B_n^{\lambda}|^2\nonumber\\
    &\leq p(p-1)2^{p-3}|A_n^{\lambda}|^{p-2}|B_n^{\lambda}|^2+p(p-1)2^{p-3}|B_n^{\lambda}|^p:=D_n^{\lambda}+F_n^{\lambda}.\label{lem2eq4}
\end{align}
To bound $\mathbb{E}[D_n^{\lambda}\mid\theta_n^{\lambda}]$, we first estimate the factor $\mathbb{E}[|B_n^{\lambda}|^2\mid\theta_n^{\lambda}]$
\begin{align}
    \mathbb{E}[|B_n^{\lambda}|^2\mid\theta_n^{\lambda}]&\leq  \mathbb{E}\left[\left|2\langle\Delta_n^{\lambda},G_n^{\lambda}\rangle+|G_n^{\lambda}|^2\right|^2\mid\theta_n^{\lambda}\right]\leq \mathbb{E}\left[8|\Delta_n^{\lambda}|^2|G_n^{\lambda}|^2+2|G_n^{\lambda}|^4\mid\theta_n^{\lambda}\right]\nonumber\\
    &\leq  8|\Delta_n^{\lambda}|^2\mathbb{E}\left[|G_n^{\lambda}|^2\mid\theta_n^{\lambda}\right]+2\mathbb{E}\left[|G_n^{\lambda}|^4\mid\theta_n^{\lambda}\right]=8(2\lambda d\beta^{-1})A_n^{\lambda}+2(2\lambda d\beta^{-1})^2.\nonumber
\end{align}
In view of the above result and \eqref{lem2eq2}, we derive
\begin{align}
    \mathbb{E}[D_n^{\lambda}\mid\theta_n^{\lambda}]&=p(p-1)2^{p-3}|A_n^{\lambda}|^{p-2}\mathbb{E}[|B_n^{\lambda}|^2\mid\theta_n^{\lambda}]\nonumber\\
    &\leq p(p-1)2^{p-2}\left(4(2\lambda d\beta^{-1})|A_n^{\lambda}|^{p-1}+(2\lambda d\beta^{-1})^2|A_n^{\lambda}|^{p-2}\right)\nonumber\\
    &\leq p(p-1)2^{p}(2\lambda d\beta^{-1})(r_{p-1}^{\lambda}|\theta_n^{\lambda}|^{2(p-1)}
+w_{p-1}^{\lambda})\nonumber\\
&+p(p-1)2^{p-2}(2\lambda d\beta^{-1})^2(r_{p-2}^{\lambda}|\theta_n^{\lambda}|^{2(p-2)}
+w_{p-2}^{\lambda}).\label{lem2eq5}
\end{align}
We now turn to estimating the remainder
\begin{align}
    \mathbb{E}[F_n^{\lambda}\mid\theta_n^{\lambda}]&\leq p(p-1)2^{p-3}\mathbb{E}\left[\left|2\langle\Delta_n^{\lambda},G_n^{\lambda}\rangle+|G_n^{\lambda}|^2\right|^p\mid\theta_n^{\lambda}\right]\nonumber\\
    &\leq p(p-1)2^{3p-4}|\Delta_n^{\lambda}|^p\mathbb{E}[|G_n^{\lambda}|^p\mid\theta_n^{\lambda}]+p(p-1)2^{2p-4}\mathbb{E}[|G_n^{\lambda}|^{2p}\mid\theta_n^{\lambda}]\nonumber\\
    &\leq p(p-1)2^{3p-4}(2\lambda d\beta^{-1})^{p/2}(A_n^{\lambda})^{p/2}+p(p-1)2^{2p-4}(2\lambda d\beta^{-1})^p\nonumber\\
    &\leq p(p-1)2^{3p-4}(2\lambda d\beta^{-1})^{p/2}(r_{p/2}^{\lambda}|\theta_n^{\lambda}|^{p}
+w_{p/2}^{\lambda})+p(p-1)2^{2p-4}(2\lambda d\beta^{-1})^p.\label{lem2eq6}
\end{align}
Hence, combining the bounds \eqref{lem2eq2}--\eqref{lem2eq6} and substituting into \eqref{lem2eq1}, we obtain
\begin{align*}
    \mathbb{E}[|\theta_{n+1}^{\lambda}|^{2p}\mid\theta_{n}^{\lambda}]&\leq r^{\lambda}_p|\theta_n^{\lambda}|^{2p}+w^{\lambda}_p\nonumber\\
    &+ p(p+1)2^{p}(2\lambda d\beta^{-1})(r_{p-1}^{\lambda}|\theta_n^{\lambda}|^{2(p-1)}
    +w_{p-1}^{\lambda})\nonumber\\
    &+ p(p-1)2^{p-2}(2\lambda d\beta^{-1})^2(r_{p-2}^{\lambda}|\theta_n^{\lambda}|^{2(p-2)}
    +w_{p-2}^{\lambda})\nonumber\\
    &+ p(p-1)2^{3p-4}(2\lambda d\beta^{-1})^{p/2}(r_{p/2}^{\lambda}|\theta_n^{\lambda}|^{p}
    +w_{p/2}^{\lambda})+p(p-1)2^{2p-4}(2\lambda d\beta^{-1})^p.\nonumber
\end{align*}
We establish the desired bound on the event $\{|\theta_n^{\lambda}|^2\geq (4\mu)(2d\beta^{-1})p(p+1)2^{p+1}\}$. Conditioning on this event yields
\begin{align}
    \mathbb{E}[|\theta_{n+1}^{\lambda}|^{2p}\mid\theta_{n}^{\lambda}]&\leq r_p^{\lambda}|\theta_n^{\lambda}|^{2p}+(1/2)(\lambda\mu/4)r_{p-1}^{\lambda}|\theta_n^{\lambda}|^{2p}\nonumber\\
    &+(1/2)(\lambda\mu/4)^2r_{p-2}^{\lambda}|\theta_n^{\lambda}|^{2p}+(1/2)(\lambda\mu/4)^{p/2}r_{p/2}^{\lambda}|\theta_n^{\lambda}|^{2p}\nonumber\\
    &+w_p^{\lambda}+p(p-1)2^{3p-2}((2\lambda d\beta^{-1})w_{p-1}^{\lambda}+(2\lambda d\beta^{-1})^2w_{p-2}^{\lambda}\nonumber\\
    &+(2\lambda d\beta^{-1})^{p/2}w_{p/2}^{\lambda}+(2\lambda d\beta^{-1})^{p})\nonumber\\
    &\leq r^{\lambda}_{p/2}\left(\left(1-\lambda\mu/4\right)^{p/2}+(1/2)(\lambda\mu/4)(1-\lambda\mu/4)^{p/2-1}\right.\nonumber\\
    &\left.+(1/2)(\lambda\mu/4)^2(1-\lambda\mu/4)^{p/2-2}\right.\nonumber\\
    &\left.+(1/2)(\lambda\mu/4)^{p/2}\right)|\theta_n^{\lambda}|^{2p}+\lambda p(p-1)2^{4p-2}{17}^{p-1}(1/4\mu)^{p-1}(1+b+d\beta^{-1})^p\nonumber\\
    &\leq (1-\lambda\mu/2)|\theta_n^{\lambda}|^{2p}+\lambda p(p-1)2^{2p-1}{17}^{p-1}(1/\mu)^{p-1}(1+b+d\beta^{-1})^p,\label{lem2eq7}
\end{align}
where we have used the restriction $\lambda<1/(4\mu)$. Subsequently, for the complementary case we have
\begin{align}
    \mathbb{E}[|\theta_{n+1}^{\lambda}|^{2p}\mid\theta_{n}^{\lambda}]&\leq (1-\lambda\mu/2)|\theta_n^{\lambda}|^{2p}+\lambda (p(p+1)2^{p+1})^{p}2^{3p-3}(1/\mu)^{p-1}(d\beta^{-1})^{p}\nonumber\\
    &+\lambda p(p-1)2^{3p-4}(p(p+1)2^{p+1})^{p-1}(1/\mu)^{p-1}(d\beta^{-1})^p\nonumber\\
    &+\lambda p(p-1)2^{4p-2}(p(p+1)2^{p+1})^{p/2}(1/\mu)^{p-1}(d\beta^{-1})^{p}\nonumber\\
    &+\lambda p(p-1)2^{2p-1}{17}^{p-1}(1/\mu)^{p-1}(1+b+d\beta^{-1})^p.\label{lem2eq8}
\end{align}
We conclude by combining \eqref{lem2eq7}--\eqref{lem2eq8}, taking expectation on both sides, and iterating the resulting bound as in the proof of Lemma \hyperlink{proof_lemma_MB}{\ref*{lemma_MB}}
\begin{align}
    \mathbb{E}|\theta_{n+1}^{\lambda}|^{2p}&\leq (1-\lambda\mu/2)\mathbb{E}|\theta_n^{\lambda}|^{2p}+\lambda 2^{4p+2}p(p-1)(p(p+1)2^{p+1})^{p/2}(1/\mu)^{p-1}(1+b+d\beta^{-1})^p\nonumber\\
    &\leq \mathbb{E}|\theta_0^{\lambda}|^{2p}+2^{4p+3}p(p-1)(p(p+1)2^{p+1})^{p/2}(1/\mu)^p(1+b+d\beta^{-1})^p.\nonumber
\end{align}
\end{proof}
\subsection{Proof of Lemma~\ref{lemma_MB_aux}}\hypertarget{proof_lemma_MB_aux}{}
\begin{proof}
Consider the auxiliary process from Definition \ref{def1} and fix $p\in\mathbb{N}\setminus\{1\}$. Using standard arguments involving stopping times, Gr\"onwall's lemma, Fatou's lemma and the bound from Lemma \ref{lemmaB2}, we obtain the existence of a constant $c_{n_0}$, which depends on time, such that $\sup_{t\in[0,n_0]}\mathbb{E}[V_p(\zedt)]\leq c_{n_0}<\infty$. Incorporating the integrating factor $e^{\lambda C_3(p)t}$ and applying It\^o's formula to $(t,x)\mapsto e^{\lambda C_3(p)t}(1+|x|^2)^{p/2}$ for $t\in[0,n_0]$ yields
\begin{align}
    e^{\lambda C_3(p)n_0}V_p(\zedt)&=V_p(\bar{\zeta}_{0}^{\lambda,n})+\int_{0}^{n_0}e^{\lambda C_3(p)s}\Big(\lambda C_3(p)V_p(\zeds)+\lambda\beta^{-1}\Delta V_p(\zeds)\nonumber\\
    &\qquad-\lambda\langle h(\zeds),\nabla V_p(\zeds)\rangle\Big)ds\nonumber\\
    &\quad+\int_{0}^{n_0}\sqrt{2\lambda\beta^{-1}}\,e^{\lambda C_3(p)s}\langle\nabla V_p(\zeds),dW_s^{\lambda,n}\rangle. \label{lemma3eq1}
\end{align}
Since $|\nabla V_{p}(x)|^2\leq p^2|V_{p}(x)|^2=p^2 V_{2p}(x)$ and $e^{\lambda C_3(p)s}$ is bounded on $[0,n_0]$, the moment bound from Lemma \ref{lemmaB2} gives $\mathbb{E}\int_{0}^{n_0}|\nabla V_p(\zeds)|^2\,ds<\infty$, so the stochastic integral in \eqref{lemma3eq1} is a true martingale and vanishes under expectation. Taking expectations and using the drift condition $\lambda C_3(p)V_p+\lambda\mathcal{A}V_p\leq\lambda C_4(p)$ from Lemma \ref{lemmaB2}, one writes
\begin{align}
    e^{\lambda C_3(p)n_0}\mathbb{E}[V_p(\zedt)]&=\mathbb{E}[V_p(\bar{\zeta}_{0}^{\lambda,n})]+\lambda\int_{0}^{n_0}e^{\lambda C_3(p)s}\mathbb{E}\big[C_3(p)V_p(\zeds)+\mathcal{L}V_p(\zeds)\big]ds\nonumber\\
    &\leq \mathbb{E}[V_p(\theta_{n-n_0}^{\lambda})]+\lambda C_4(p)\int_{0}^{n_0}e^{\lambda C_3(p)s}\,ds\nonumber\\
    &=\mathbb{E}[V_p(\theta_{n-n_0}^{\lambda})]+\dfrac{C_4(p)}{C_3(p)}\big(e^{\lambda C_3(p)n_0}-1\big).\nonumber
\end{align}
Dividing through by $e^{\lambda C_3(p)n_0}$, one gets
\begin{align}
    \mathbb{E}[V_p(\zedt)]&\leq e^{-\lambda C_3(p)n_0}\mathbb{E}[V_p(\theta_{n-n_0}^{\lambda})]+\dfrac{C_4(p)}{C_3(p)}\big(1-e^{-\lambda C_3(p)n_0}\big)\nonumber\\
    &\leq \mathbb{E}[V_p(\theta_{n-n_0}^{\lambda})]+\dfrac{C_4(p)}{C_3(p)}.\nonumber
\end{align}
In view of Lemma \ref{lemma_MB_higher} and Minkowski's inequality, one gets
\begin{align*}
    \mathbb{E}[V_p(\zedt)]&\leq \left(1+\mathbb{E}[|\theta_{n-n_0}^{\lambda}|^p]^{2/p}\right)^{p/2}+\dfrac{C_4(p)}{C_3(p)}\\&\leq 2^{p/2-1}+2^{p/2-1}C_{B_2}+\left(1+R_0\right)^{p/2}\\&\leq2^{p/2-1}+2^{p/2-1}\left(\mathbb{E}|\theta_0|^{p}+M(p/2)(1/\mu)^{p/2}(1+b+d\beta^{-1})^{p/2}\right)+\left(1+R_0\right)^{p/2}\\
    &\leq 2^{p/2-1}\left[
1+\mathbb{E}|\theta_0|^p
+
2^{p/2}
\right.\\
&\qquad\left.+
\left(4^{p/2}+M(p/2)\right)
(1/\mu)^{2p}
\left(
1+b+\beta^{-1}\bigl(1+b+\beta^{-1}(d+p-2)\bigr)^{p/2}
\right)
\right].
\end{align*}
\end{proof}
\subsection{Proof of Lemma~\ref{lemma_MB_step}}\hypertarget{proof_lemma_MB_step}{}
\begin{proof}
One considers the difference between $(\bar{\theta}_{\lfloor t\rfloor}^{\lambda})_{t\geq0}$, and $(\bar{\theta}_{t}^{\lambda})_{t\geq0}$, to get the one-step error
\begin{align}
    |\bar{\theta}_{\lfloor t\rfloor}^{\lambda}-\bar{\theta}_{t}^{\lambda}|^2
\leq 2\left|\int_{\lfloor t\rfloor}^t \lambda h_{\lambda}(\bar{\theta}_{\lfloor s\rfloor}^{\lambda})\,ds\right|^2
+4\lambda\beta^{-1}|\tilde{B}_{t}^{\lambda}-\tilde{B}_{\lfloor t\rfloor}^{\lambda}|^2.\nonumber
\end{align}
Since $\lfloor s\rfloor=\lfloor t\rfloor$ for $s\in[\lfloor t\rfloor,t]$, the drift
integrand is constant on this interval, hence
\begin{align}
    |\bar{\theta}_{\lfloor t\rfloor}^{\lambda}-\bar{\theta}_{t}^{\lambda}|^2\leq 2\lambda^2(t-\lfloor t\rfloor)^2| h_{\lambda}(\bar{\theta}_{\lfloor t\rfloor}^{\lambda})|^2+4\lambda\beta^{-1}|\tilde{B}_{t}^{\lambda}-\tilde{B}_{\lfloor t\rfloor}^{\lambda}|^2.\nonumber
\end{align}
By \eqref{eqR2}, $|h_{\lambda}(x)|^2\leq 2\mu^2|x|^2+2\lambda^{-1}$. Taking expectations and
applying Lemma \ref{lemma_MB},
\begin{align}
    \mathbb{E}|\bar{\theta}_{\lfloor t\rfloor}^{\lambda}-\bar{\theta}_{t}^{\lambda}|^2\leq 4\lambda^2\mu^2C_{B_1}+4\lambda+4\lambda d\beta^{-1}\leq\lambda\left(\mu C_{B_1}+4+4d\beta^{-1}\right),
\end{align}
where the last inequality uses $4\lambda\mu\leq1$. The bound for
$(\vartheta_t^{\lambda,n})_{t\in[0,n_0]}$ is identical since
$\mathcal{L}(\vartheta_k^{\lambda,n})=\mathcal{L}(\theta_{n-n_0+k}^{\lambda})$ for
$k\in\{0,\ldots,n_0\}$.
\end{proof}
\subsection{Proof of Lemma~\ref{lemma_err_W2}}\hypertarget{proof_lemma_err_W2}{}
\begin{proof}
    Consider the difference between the continuous-time interpolation \eqref{SGTULA_cont} and the auxiliary process from Definition \ref{def1}. Applying It\^o's formula to $x\to|x|^2$, then one writes for any $t\in[0,n_0]$
    \begin{align}
        |\vartheta_{t}^{\lambda,n}-\bar{\zeta}_{t}^{\lambda,n}|^2&=-2\lambda\int_{0}^{t}\langle \vartheta_{s}^{\lambda,n}-\bar{\zeta}_s^{\lambda,n},h_{\lambda}(\vartheta_{\lfloor s\rfloor}^{\lambda,n})-h(\bar{\zeta}_s^{\lambda,n})\rangle ds\nonumber\\
        &=-2\lambda\int_{0}^{t}\langle \vartheta_{\lfloor s\rfloor}^{\lambda,n}-\bar{\zeta}_s^{\lambda,n},h(\vartheta_{\lfloor s\rfloor}^{\lambda,n})-h(\bar{\zeta}_s^{\lambda,n})\rangle ds\nonumber\\
        &-2\lambda\int_{0}^{t}\langle \vartheta_s^{\lambda,n}-\vartheta_{\lfloor s\rfloor}^{\lambda,n},h(\vartheta_{\lfloor s\rfloor}^{\lambda,n})-h(\bar{\zeta}_s^{\lambda,n})\rangle ds\nonumber\\
        &-2\lambda\int_{0}^{t}\langle \vartheta_s^{\lambda,n}-\bar{\zeta}_s^{\lambda,n},h_{\lambda}(\vartheta_{\lfloor s\rfloor}^{\lambda,n})-h(\vartheta_{\lfloor s\rfloor}^{\lambda,n})\rangle ds:=k_1(t)+k_2(t)+k_3(t).\label{lemma4eq1}
    \end{align}
    The first term, $k_1(t)$, is controlled through the one-sided Lipschitz Assumption \hyperlink{A1}{A1}, in particular
    \begin{align}
        k_1(t)&=-2\lambda\int_{0}^{t}\langle \vartheta_{\lfloor s\rfloor}^{\lambda,n}-\bar{\zeta}_s^{\lambda,n},h(\vartheta_{\lfloor s\rfloor}^{\lambda,n})-h(\bar{\zeta}_s^{\lambda,n})\rangle ds\leq 2L\lambda\int_{0}^{t}|\vartheta_{\lfloor s\rfloor}^{\lambda,n}-\zeds|^2ds\nonumber\\
        &\leq 4L\lambda\int_{0}^{t}|\vartheta_s^{\lambda,n}-\zeds|^2ds+4L\lambda\int_{0}^{t}|\vartheta_{\lfloor s\rfloor}^{\lambda,n}-\vartheta_s^{\lambda,n}|^2ds.\label{k1}
    \end{align}
    Subsequently, to control the second term, $k_2(n_0)$, one uses the Cauchy--Schwarz inequality and the polynomial growth Assumption \hyperlink{A3}{A3}, to obtain
    \begin{align}
        k_2(t)&=-2\lambda\int_{0}^{t}\langle \vartheta_s^{\lambda,n}-\vartheta_{\lfloor s\rfloor}^{\lambda,n},h(\vartheta_{\lfloor s\rfloor}^{\lambda,n})-h(\bar{\zeta}_s^{\lambda,n})\rangle ds\nonumber\\
        &\leq 2\lambda\int_{0}^{t}|\vartheta_s^{\lambda,n}-\vartheta_{\lfloor s\rfloor}^{\lambda,n}||h(\vartheta_{\lfloor s\rfloor}^{\lambda,n})-h(\bar{\zeta}_s^{\lambda,n})|ds.\nonumber
    \end{align}
    Under expectation, one writes
    \begin{align}
        \mathbb{E}[k_2(t)]&\leq2\lambda\int_{0}^{t} (\mathbb{E}[|\vartheta_s^{\lambda,n}-\vartheta_{\lfloor s\rfloor}^{\lambda,n}|^2])^{1/2} (\mathbb{E}[|h(\vartheta_{\lfloor s\rfloor}^{\lambda,n})-h(\bar{\zeta}_s^{\lambda,n})|^2])^{1/2}ds\nonumber\\
        &\leq 2\sqrt{2}C_{B_4}^{1/2}\lambda^{3/2}\int_{0}^{t} (\mathbb{E}[|h(\vartheta_{\lfloor s\rfloor}^{\lambda,n})|^2+|h(\bar{\zeta}_s^{\lambda,n})|^2])^{1/2}ds\nonumber\\
        &\leq 4C_{B_4}^{1/2}\lambda^{3/2}\int_{0}^{n_0} (2m^2+K^2\mathbb{E}[|\vartheta_{\lfloor s\rfloor}^{\lambda,n}|^{2p}]+K^2\mathbb{E}[|\bar{\zeta}_s^{\lambda,n}|^{2p}])^{1/2}ds\nonumber\\
        &\leq 4(\max{\{2m^2,K^2\}}(1+C_{B_2}+C_{B_3})C_{B_4})^{1/2}{\lambda}^{3/2} n_0,\label{k2}
    \end{align}
    where in the last step, we used the moment estimates from Lemmas \ref{lemma_MB_higher}--\ref{lemma_MB_aux}, with $2m^2+K^2(C_{B_2}+C_{B_3})\leq \max\{2m^2,K^2\}(1+C_{B_2}+C_{B_3})$. The last term, $k_3(t)$, is controlled via the properties of the taming function, specifically Remark \ref{remark_close_unif}
    \begin{align}
        k_3(t)&=-2\lambda\int_{0}^{t}\langle \vartheta_s^{\lambda,n}-\bar{\zeta}_s^{\lambda,n},h_{\lambda}(\vartheta_{\lfloor s\rfloor}^{\lambda,n})-h(\vartheta_{\lfloor s\rfloor}^{\lambda,n})\rangle ds\nonumber\\
        &\leq 2\lambda \int_{0}^{n_0} (L/2)|\vartheta_s^{\lambda,n}-\bar{\zeta}_s^{\lambda,n}|^2+(2L)^{-1}|h_{\lambda}(\vartheta_{\lfloor s\rfloor}^{\lambda,n})-h(\vartheta_{\lfloor s\rfloor}^{\lambda,n})|^2  ds\nonumber\\
        &\leq L\lambda\int_{0}^{n_0}|\vartheta_s^{\lambda,n}-\bar{\zeta}_s^{\lambda,n}|^2ds+\lambda^{3/2} L^{-1}C_1\int_{0}^{n_0}(1+|\vartheta_{\lfloor s\rfloor}^{\lambda,n}|^{2p}) ds,\label{k3}
    \end{align}
    where the $\epsilon$-Young inequality was applied with $\epsilon=L$. Taking the expectation of
    \eqref{lemma4eq1} and substituting from \eqref{k1},\eqref{k2} and \eqref{k3}, we derive, for
    every $t\in[0,n_0]$,
    \begin{align}
        \mathbb{E}[|\vartheta_t^{\lambda,n}-\bar{\zeta}_t^{\lambda,n}|^2]&\leq 5L\lambda\int_{0}^{t}\mathbb{E}[|\vartheta_s^{\lambda,n}-\bar{\zeta}_s^{\lambda,n}|^2]ds+4LC_{B_4}\lambda^2n_0\nonumber\\
        &+4(\max{\{2m^2,K^2\}}(1+C_{B_2}+C_{B_3})C_{B_4})^{1/2}{\lambda}^{3/2} n_0\nonumber\\
        &+\lambda^{3/2}L^{-1}C_1(1+C_{B_2})n_0.
    \end{align}
    Recalling from \eqref{def_horizon} that $\lambda n_0\leq T_0$, and using
    $\lambda^{1/2}\leq\lambda_0^{1/2}\leq1/(2\sqrt{\mu})$ for the first constant term, one deduces
    that
    \begin{align}
        \mathbb{E}[|\vartheta_t^{\lambda,n}-\bar{\zeta}_t^{\lambda,n}|^2]&\leq 5L\lambda\int_{0}^{t}\mathbb{E}[|\vartheta_s^{\lambda,n}-\bar{\zeta}_s^{\lambda,n}|^2]ds\nonumber\\&+\lambda^{1/2}\,T_0\left(2LC_{B_4}/\sqrt{\mu}+L^{-1}C_1(1+C_{B_2})\right.\nonumber\\
        &\qquad\left.+4(\max{\{2m^2,K^2\}}(1+C_{B_2}+C_{B_3})C_{B_4})^{1/2}\right).\label{lemma4eq2}
    \end{align}
    Since the second moments of the processes $(\zedt)_{t\geq 0}$ and
    $(\vartheta_t^{\lambda,n})_{t\geq 0}$ can be controlled, the right-hand side of
    \eqref{lemma4eq2} is finite. Hence, applying Gr\"onwall's lemma on $[0,n_0]$ and using
    $5L\lambda n_0\leq 5LT_0$ leads to
    \begin{align*}
        \mathbb{E}[|\vartheta_t^{\lambda,n}-\bar{\zeta}_t^{\lambda,n}|^2]&\leq \lambda^{1/2}\,e^{5LT_0}\,T_0\left(2LC_{B_4}/\sqrt{\mu}+L^{-1}C_1(1+C_{B_2})\right.\\
        &\qquad\left.+4(\max{\{2m^2,K^2\}}(1+C_{B_2}+C_{B_3})C_{B_4})^{1/2}\right).
    \end{align*}
\end{proof}
\subsection{Proof of Theorem~\ref{theorem_main}}\hypertarget{proof_theorem_main}{}
\begin{proof}
        Let $\mathcal{L}(\theta_0)=\nu$ and denote by $(P_t)_{t\geq 0}$ the Markov semigroup of the
        SDE \eqref{SDE_1} and by $\{\bar{P}_k\}_{k\in\mathbb{N}_0}$ the discrete-time semigroup of
        the algorithm $(\theta^{\lambda}_k)_{k\in\mathbb{N}_0}$. Let $n\geq n_0$ with $n_0$ as in
        \eqref{def_horizon}. We triangulate the global error
    \begin{align}
        W_2(\mathcal{L}(\theta_n^{\lambda}),\pi_{\beta})=W_2(\nu\bar{P}_n,\pi_{\beta})&\leq W_2(\nu\bar{P}_n,\nu\bar{P}_{n-n_0}P_{t_{n_0}})+W_2(\nu\bar{P}_{n-n_0}P_{t_{n_0}},\pi)\nonumber\\&\leq W_2(\nu\bar{P}_{n-n_0}\bar{P}_{n_0},\nu\bar{P}_{n-n_0}P_{n_0\lambda})+W_2(\nu\bar{P}_{n-n_0}P_{n_0\lambda},\pi_{\beta})\nonumber\\
        &= W_2(\mathcal{L}(\vartheta_{n_0}^{\lambda,n}),\mathcal{L}(\bar{\zeta}_{n_0}^{\lambda,n}))+W_2(\mathcal{L}(\bar{\zeta}_{n_0}^{\lambda,n}),\pi_{\beta}),\label{GE01}
    \end{align}
    where the identification of the two laws follows from Definition~\ref{def1} and writing $t_n=n\lambda$. The first term is
    the discretization error over the restart window; by Lemma~\ref{lemma_err_W2} (at $t=n_0$),
    \begin{align}
         W_2(\mathcal{L}(\vartheta_{n_0}^{\lambda,n}),\mathcal{L}(\bar{\zeta}_{n_0}^{\lambda,n}))\leq C_{E}\lambda^{1/4}.\label{MT_term2}
    \end{align}
    The second term is the contraction error. Applying Proposition~\ref{prop_contraction1} with
    initial laws $\mathcal{L}(\theta_{n-n_0}^{\lambda})$ and $\pi_{\beta}$, over the time horizon
    $n_0\lambda$, and using $C_we^{-C_rn_0\lambda}\leq e^{-1}$ by the choice \eqref{def_horizon},
    \begin{align}
        W_2(\mathcal{L}(\bar{\zeta}_{n_0}^{\lambda,n}),\pi_{\beta})\leq e^{-1}\,W_2(\mathcal{L}(\theta_{n-n_0}^{\lambda}),\pi_{\beta}),\qquad \forall n\geq n_0.\label{MT_term1}
    \end{align}
    Putting \eqref{MT_term2} and \eqref{MT_term1} together yields, for any $n\geq n_0$,
    \begin{align*}
        W_2(\mathcal{L}(\theta_n^{\lambda}),\pi_{\beta})\leq C_{E}\lambda^{1/4}+e^{-1}\,W_2(\mathcal{L}(\theta_{n-n_0}^{\lambda}),\pi_{\beta}).
    \end{align*}
    The recursion lemma above (with $\epsilon=C_{E}\lambda^{1/4}$, $q=e^{-1}$) then gives, for
    every $n\in\mathbb{N}_0$,
    \begin{align*}
        W_2(\mathcal{L}(\theta_n^{\lambda}),\pi_{\beta})\leq \dfrac{C_{E}\lambda^{1/4}}{1-e^{-1}}+e^{-(n/n_0-1)}\sup_{k\in\{0,\ldots,n_0-1\}}W_2(\mathcal{L}(\theta_k^{\lambda}),\pi_{\beta}).
    \end{align*}
    The supremum is bounded by $C_{D}$ due to Lemma~\ref{lemma_crude}. Moreover, by
    \eqref{def_horizon},
    \begin{align*}
        \dfrac{n}{n_0}\geq \dfrac{n}{\dfrac{\ln C_w+1}{C_r \lambda}+1}\geq \dfrac{C_r\, n\lambda}{\ln C_w +1 +C_r\lambda_0}=\dfrac{n\lambda}{T_0}.
    \end{align*}
    Since $1/(1-e^{-1})\leq 2$, one concludes that for any $n\in\mathbb{N}_0$,
    \begin{align*}
       W_2(\mathcal{L}(\theta_n^{\lambda}),\pi_{\beta})\leq 2C_{E}\lambda^{1/4}+eC_{D}\, e^{-n\lambda/T_0}.
    \end{align*}
    \end{proof}
    \subsection{Proof of Theorem~\ref{theorem-main-opt}}\hypertarget{theorem-main-opt}{}
\begin{proof}
    Set $\epsilon=\beta^{-1/2}.$
    Since $u_\epsilon$ is $\beta^{1/2}$-smooth there holds 
    \[\begin{aligned}
        u_\epsilon(\theta^\lambda_{n+1} )-u_\epsilon(\theta^\lambda_{n} )&\leq \langle \nabla u_\epsilon(\theta^\lambda_{n} ),\theta^\lambda_{n+1} -\theta^\lambda_{n} \rangle + \frac{\sqrt{\beta}}{2} |\theta^\lambda_{n+1} -\theta^\lambda_{n} |^2
        \\&= -\lambda \langle\nabla u_\epsilon(\theta^\lambda_{n} ),h_\lambda(\theta^\lambda_{n} )\rangle + \sqrt{2\lambda \beta^{-1}} \langle\nabla u_\epsilon(\theta^\lambda_{n} ),\xi_n)\rangle\\&\quad + \lambda^2 \sqrt{\beta} |h_\lambda(\theta^\lambda_{n} )|^2 + 2\lambda \sqrt{\beta^{-1}} |\xi_n|^2.
    \end{aligned}\]
    Taking  expectations, since $\theta^\lambda_{n} $ is independent of $\xi_n$, yields
    \[    \mathbb{E}  (u_\epsilon(\theta^\lambda_{n+1} )-u_\epsilon(\theta^\lambda_{n} ))\leq -\lambda \mathbb{E} \langle\nabla u_\epsilon(\theta^\lambda_{n} ),h_\lambda(\theta^\lambda_{n} )\rangle + C_1 \lambda^2 \sqrt{\beta} + 2d \lambda \sqrt{\beta^{-1}}, \]
    where $C_1=\mathbb{E} |h_\lambda(\theta^\lambda_{n} )|^2\leq \mathcal{O}(d^{2p})$.
   Rearranging and writing $h=h_\lambda+(h-h_\lambda)$ gives
   \begin{equation}
   \begin{aligned}
       \lambda \mathbb{E} \langle\nabla u_\epsilon(\theta^\lambda_{n} ),h(\theta^\lambda_{n} )\rangle&\leq \mathbb{E}  u_\epsilon(\theta^\lambda_{n} )-\mathbb{E} u_\epsilon(\theta^\lambda_{n+1} )  + C_1 \lambda^2 \sqrt{\beta} + C' \lambda \sqrt{\beta^{-1}}\\&\quad + \lambda \mathbb{E} |\nabla u_\epsilon(\theta^\lambda_{n} )||h(\theta^\lambda_{n} )-h_\lambda(\theta^\lambda_{n} )|.
   \end{aligned}
   \end{equation}
    Applying \eqref{eq-MY-bound} in  the left hand side and \eqref{eq-subdomMY} in the last term of the right hand-side yields 
    \[ \begin{aligned}
       \frac{\lambda}{2}  \mathbb{E} |\nabla u_\epsilon(\theta^\lambda_{n} )|^2&\leq \mathbb{E}  u_\epsilon(\theta^\lambda_{n} )-\mathbb{E} u_\epsilon(\theta^\lambda_{n+1} )  + C_1 \lambda^2 \sqrt{\beta} + 2d \lambda \sqrt{\beta^{-1}}\\&\quad + 2\lambda C\mathbb{E} (1+|\theta^\lambda_{n} |^p)|h(\theta^\lambda_{n} )-h_\lambda(\theta^\lambda_{n} )|
        \\&\leq  \mathbb{E}  u_\epsilon(\theta^\lambda_{n} )-\mathbb{E} u_\epsilon(\theta^\lambda_{n+1} )  + C_1 \lambda^2 \sqrt{\beta} + 2d \lambda \sqrt{\beta^{-1}}\\&\quad + 2\lambda C \sqrt{\mathbb{E} (1+|\theta^\lambda_{n} |^p)^2}\sqrt{\mathbb{E}|h(\theta^\lambda_{n} )-h_\lambda(\theta^\lambda_{n} )|^2}.
    \end{aligned} \]
Regarding the third term, the moment bounds in Lemma~\ref{lemma_MB_higher} and the taming error in Remark~\ref{remark_close_unif} give \[\frac{\lambda}{2} \mathbb{E} |\nabla u_\epsilon(\theta^\lambda_{n} )|^2\leq \mathbb{E}  u_\epsilon(\theta^\lambda_{n} )-\mathbb{E} u_\epsilon(\theta^\lambda_{n+1} )  + C'\left( \lambda^2 \sqrt{\beta} + \lambda \sqrt{\beta^{-1}} + \lambda \sqrt{\lambda}\right).\]
 
\noindent Moreover by the property in Equation \eqref{eq-PL} one obtains 
\[\lambda \frac{C_{PL}}{1+\sqrt{\beta^{-1}}C_{PL}} \left(\mathbb{E} u_\epsilon(\theta^\lambda_{n} )-u^*\right)\leq \mathbb{E}  u_\epsilon(\theta^\lambda_{n} )-\mathbb{E} u_\epsilon(\theta^\lambda_{n+1} )  + C'\left( \lambda^2 \sqrt{\beta} + \lambda \sqrt{\beta^{-1}} + \lambda \sqrt{\lambda}\right).\]
Rearranging one obtains 
\[\mathbb{E} u_\epsilon(\theta^\lambda_{n+1} )-u^*\leq \left(1-\lambda \frac{C_{PL}}{1+\sqrt{\beta^{-1}}C_{PL}}\right) \left(\mathbb{E} u_\epsilon(\theta^\lambda_{n} )-u^*\right)+ C'\left( \lambda^2 \sqrt{\beta} + \lambda \sqrt{\beta^{-1}} + \lambda \sqrt{\lambda}\right).\]
By induction, this leads to 
\[\begin{aligned}\mathbb{E} u_\epsilon(\theta^\lambda_{n+1} )-u^*&\leq \left(1-\lambda \frac{C_{PL}}{1+\sqrt{\beta^{-1}}C_{PL}}\right)^n \left(\mathbb{E} u_\epsilon(\theta^\lambda_{0} )-u^*\right)\\&\quad+ \frac{1+\sqrt{\beta^{-1}}C_{PL}}{C_{PL}}C'\left( \lambda \sqrt{\beta} + \sqrt{\beta^{-1}} + \sqrt{\lambda}\right).\end{aligned}\]
Using the fact that $u_\epsilon\leq u$ yields 
\[\begin{aligned}\mathbb{E} u(\theta^\lambda_{n} )-u^*&\leq \mathbb{E} \left( u(\theta^\lambda_{n+1} )-u_\epsilon(\theta^\lambda_{n+1} )\right) \\&+\left(1-\lambda \frac{C_{PL}}{1+\sqrt{\beta^{-1}}C_{PL}}\right)^n \left(\mathbb{E} u_\epsilon(\theta^\lambda_{0} )-u^*\right)\\&\quad+ \frac{1+\sqrt{\beta^{-1}}C_{PL}}{C_{PL}}C'\left( \lambda \sqrt{\beta} + \sqrt{\beta^{-1}} + \sqrt{\lambda}\right).\end{aligned}\]
The result follows from Lemmas~\ref{approximation MY} and~\ref{lemma_MB_higher}, which give\[\mathbb{E} \left(u(\theta^\lambda_{n+1} )-u_\epsilon(\theta^\lambda_{n+1} )\right)\leq C''\sqrt{\beta^{-1}}.\] 
\end{proof}
\subsection{Proof of Lemma~\ref{lemmaER1}}\hypertarget{lemmaER1_proof}{}
\begin{proof}
    We notice that the function $g(t)=u(tx+(1-t)y)$ is locally Lipschitz continuous (since $u$ is semi-convex) so it has a bounded variation in $[0,1].$ Then, one can enforce the fundamental theorem of calculus since $g'(t)=\langle h(tx+(1-t)y),x-y\rangle$ a.e. Let $k(p)=2^{p-1}K/(p+1)$, then one writes
        \begin{align}
            u(x)-u(y)&=\int_0^1\langle x-y,h((1-t)y+tx) \rangle dt\leq \int_0^1|x-y||h((1-t)y+tx)| dt\nonumber\\
            &\leq \int_0^1|x-y|(m+K|(1-t)y+tx|^p) dt\leq (m+k(p)|x|^p+k(p)|y|^p)|x-y|\label{eqT1_1},
        \end{align}
        where we have used Cauchy--Schwarz and the growth Assumption \hyperlink{A3}{A3}. Now let $P$ be the coupling of $\mu,\ \nu$ that achieves $W_2(\mu,\nu)$. That is, $P=\mathcal{L}((X,Y))$ with $\mu=\mathcal{L}(X),\ \nu=\mathcal{L}(Y)$ such that $W_2^2=\mathbb{E}^{P}|X-Y|^2$. Taking expectations in \eqref{eqT1_1}, yields
        \begin{align}
            \int_{\mathbb{R}^d}g d\mu-\int_{\mathbb{R}^d}g d\nu &= \mathbb{E}^{P}[g(X)-g(Y)]\nonumber \\
            &\leq \sqrt{\mathbb{E}^P[(m+k(p)|X|^p+k(p)|Y|^p)^2]}\sqrt{\mathbb{E}^P|X-Y|^2}\nonumber\\
            &\leq \left(m+k(p)\sqrt{\mathbb{E}^P|X|^{2p}}+k(p)\sqrt{\mathbb{E}^P|Y|^{2p}}\right)W_2(\mu,\nu)\label{eqT1_2}.
        \end{align}
        where the last step follows from Minkowski's inequality. One concludes by applying inequality \eqref{eqT1_2} for $X=\theta_n^{\lambda}$ and $Y=\theta_{\infty}$
        \begin{align}
            \mathbb{E}[u(\theta_n^{\lambda})]-\mathbb{E}[u(\theta_{\infty})]\leq \left(m+k(p)\sqrt{C_{B_2}}+k(p)\sqrt{C_{\sigma,2p}}\right)W_2(\mathcal{L}(\theta_n^{\lambda}),\pi_{\beta}),
        \end{align}
        where $C_{\sigma,2p}$ is the $2p$th-moment of $\pi_{\beta}$ and  $\mathbb{E}|\theta_n^{\lambda}|^{2p}\leq C_{B_2}$ by Lemma \ref{lemma_MB_higher}. Since $\pi_{\beta}$ is the invariant measure of SDE \eqref{SDE_1}, there holds $\int_{\mathbb{R}^d}\mathcal{A}V_{2p}(x)\pi_{\beta}(dx)=0$.  Due to Lemma \ref{lemmaB2}, one estimates the constant by
        \begin{align*}
            C_{\sigma,2p}&\leq \int_{\mathbb{R}^d} V_{2p}(x)\pi_{\beta}(dx)\leq -{C_3(2p)}^{-1}\int_{\mathbb{R}^d}\mathcal{A}V_{2p}(x)\pi_{\beta}(dx)+C_4(2p)/C_3(2p)\\
            &\leq C_4(2p)/C_3(2p).
        \end{align*}
        In view of Lemma~\ref{lemmaB2}, $C_{\sigma,2p}\leq C_4(2p)/C_3(2p)=A^p(2p)$, and the
   stated constant $C_{\mathcal{T}_1}$ follows.
\end{proof}
\subsection{Proof of Lemma~\ref{lemmaER2}}\hypertarget{lemmaER2_proof}{}
\begin{proof}
    We follow a similar approach as in Section 3.5 of \cite{Raginsky}, making necessary adjustments due to the lack of a smoothness condition for the gradient $\nabla u(x):=h(x)$ and polynomial growth. According to \cite{Raginsky}, one obtains the following bound
        \begin{align}
            \mathbb{E}[u(\theta_{\infty})]\leq \dfrac{d}{2\beta}\log\left(\dfrac{4\pi e(b+d/\beta)}{\mu d}\right)-\dfrac{1}{\beta}\log Z,\label{T2_eq1}
        \end{align}
        where $Z:=\int_{\mathbb{R}^d}e^{-\beta u(x)}dx$ is the normalization constant. One writes
        \begin{align}
            \log Z=\log\int_{\mathbb{R}^d}e^{-\beta u(x)}dx=-\beta u(\theta^{\star})+\log\int_{\mathbb{R}^d}e^{\beta(u(\theta^{\star})- u(x))}dx.\label{T2_eq2}
        \end{align}
        Now we provide an upper bound for the second term of \eqref{T2_eq2}. For the remainder of this analysis, one chooses the version of the subgradient $h(x)$ such that $h(\theta^{\star})=0$. Evaluating Remark \ref{remark_diss} at $\theta^{\star}$ then gives $\tfrac{\mu}{2}|\theta^{\star}|^2\leq b$, hence $|\theta^{\star}|\leq\sqrt{2b/\mu}=:R_2$. Consequently, one calculates that
        \begin{align*}
            -(u(\theta^{\star})-u(x))&\leq|u(\theta^{\star})-u(x)|\leq\int_0^1\left|\langle h(x+t(\theta^{\star}-x)),\theta^{\star}-x\rangle\right |dt\\
            &\leq\int_0^1|h(x+t(\theta^{\star}-x))||\theta^{\star}-x|dt\leq\int_0^1\left(m+K|x+t(\theta^{\star}-x)|^p\right)|\theta^{\star}-x|dt\\
            &\leq\int_0^1\left(m+2^{p-1}K|x|^p+2^{p-1}Kt^p|\theta^{\star}-x|^p\right)|\theta^{\star}-x|dt\\
            &\leq (K2^{2p-2}+K2^{p-1}/(p+1))|\theta^{\star}-x|^{p+1}+(m+K2^{2p-2}R_2^p)|\theta^{\star}-x|.
        \end{align*}
        Hence we obtain
        \begin{align}
            I&=\int_{\mathbb{R}^d}e^{\beta(u(\theta^{\star})-u(x))}dx\geq \int_{\mathbb{R}^d}e^{-\beta (K2^{2p-2}+K2^{p-1}/(p+1))|\theta^{\star}-x|^{p+1}-\beta(m+K2^{2p-2}R_2^p)|\theta^{\star}-x|}dx\nonumber\\
            &\geq \int_{\mathbb{R}^d} e^{-\beta J\left(|\theta^{\star}-x|^{p+1}+|\theta^{\star}-x|\right)}dx,
        \end{align}
        where $J=m+K2^{2p-2}(1+(2b/\mu)^{p/2})+K2^{p-1}/(p+1)$ and the last inequality uses that both coefficients $K2^{2p-2}+K2^{p-1}/(p+1)$ and $m+K2^{2p-2}R_2^p$ are bounded above by $J$ (recall $R_2^p=(2b/\mu)^{p/2}$).
        Passing to radial coordinates, with $S_d=2\pi^{d/2}\Gamma^{-1}(d/2)$ the surface area of the unit sphere in $\mathbb{R}^d$, one obtains
        \[I\geq S_d \int_0^{\infty} e^{-\beta J\left(r^{p+1}+r\right)} r^{d-1}\, d r.\]
        Assume $\beta J\geq 1$. Restricting the domain of integration to $r\in[0,(\beta J)^{-1}]$, where $r\leq 1$ and hence $r^{p+1}\leq r$, gives $\beta J(r^{p+1}+r)\leq 2\beta J r\leq 2$, so that
        \[I\geq S_d \int_0^{(\beta J)^{-1}} e^{-\beta J\left(r^{p+1}+r\right)} r^{d-1}\, d r\geq S_d\, e^{-2}\int_0^{(\beta J)^{-1}} r^{d-1}\, d r= \frac{S_d}{d}\, e^{-2}\, (\beta J)^{-d}.\]
        Combining this with \eqref{T2_eq2} yields
        \[\frac{1}{\beta}\log Z= -u(\theta^{\star})+\frac{1}{\beta}\log I\geq -u(\theta^{\star})+ \frac{1}{\beta} \log (S_d/d)-\frac{d}{\beta}\log(\beta J)-\frac{2}{\beta}.\]
        In view of \eqref{T2_eq1}, one concludes with
        \begin{align*}
            \mathbb{E}[u(\theta_{\infty})]-u(\theta^{\star})&\leq \dfrac{d}{2\beta}\log\left(\dfrac{4\pi e(b+d/\beta)\beta^2J^2}{\mu d}\right)-\dfrac{1}{\beta}\log(S_d/d)+\dfrac{2}{\beta}\\&\leq
            \dfrac{d}{2\beta}\log\left(\dfrac{4\pi e(b+d/\beta)\beta^2J^2}{\mu d}\right)-\dfrac{1}{\beta}\log\left(\dfrac{2\pi^{d/2}}{\Gamma(d/2)d}\right)+\dfrac{2}{\beta}\\&\leq
            \dfrac{d}{2\beta}\log\left(\dfrac{4\pi e(b+d/\beta)\beta^2J^2}{\mu d}\right)-\dfrac{d}{2\beta}\log\left(\pi\right)+\dfrac{1}{\beta}\log\left(\Gamma(d/2)(d/2)\right)+\dfrac{2}{\beta}\\&\leq
            \dfrac{d}{2\beta}\log\left(\dfrac{4 e(b+d/\beta)\beta^2J^2}{\mu d}\right)+\dfrac{1}{\beta}\log\left(\Gamma(d/2+1)\right)+\dfrac{2}{\beta}
            \\&\leq
            \dfrac{d}{2\beta}\log\left(\dfrac{2(b+d/\beta)\beta^2J^2}{\mu }\right)+\dfrac{1}{2\beta}\log\left(\pi d\right)+\dfrac{13}{6\beta}:=C_{\mathcal{T}_{2}}.
        \end{align*}
\end{proof}

\subsection{Proof of Lemma~\ref{lem:pr-chord}}\hypertarget{proof_lem_pr_chord}{}
\begin{proof}
Write $g(x)=u(x)+(L/2)|x|^{2}$ which is convex by \hyperlink{A1}{A1}. Let $\mathcal{L}^m$ denote
$m$-dimensional Lebesgue measure, $\mathcal{B}(z,r):=\{w:|z-w|<r\}$, $\Omega_{\rho}:=\{z:|z|>\rho\}$. Moreover, consider the mollifier, $\eta\in C_c^\infty(\mathbb{R}^d)$ with $\eta\ge0$, $\operatorname{supp}\eta\subset\mathcal{B}(0,1)$,
$\int\eta\,d\mathcal{L}^d=1$, and set $\eta_\epsilon(z):=\epsilon^{-d}\eta(z/\epsilon)$, so that
$\eta_\epsilon\ge0$, $\int\eta_\epsilon\,d\mathcal{L}^d=1$, and $\operatorname{supp}\eta_\epsilon\subset\mathcal{B}(0,\epsilon)$. Henceforth we denote the convolution with $\eta_{\epsilon}$ by $g^\epsilon:=\eta_\epsilon*g$. Let $\Phi$ be the set of
differentiability points of $g$. We fix
\[
  v:=x-y,\qquad \gamma(t):=y+tv\ \ (t\in[0,1]),
\]
so that $\gamma(0)=y$ and $\gamma(1)=x$.

Since $g$ is convex, $\nabla g\in BV_{\mathrm{loc}}(\mathbb{R}^d;\mathbb{R}^d)$ and its distributional
Hessian is a symmetric matrix-valued Radon measure $[D^2g]$ whose directional component
$\mu^\xi:=\xi^\top[D^2g]\,\xi$ is, for each $\xi\in\mathbb{R}^d$, a \emph{nonnegative} Radon measure
\citep[Theorem~6.8]{EvansGariepy}. By Aleksandrov's theorem, the $\mathcal{L}^d$-absolutely continuous
density of $[D^2g]$ is the pointwise Hessian \citep[Theorem~6.9]{EvansGariepy}, so the density of $\mu^\xi$
is $\xi^\top\nabla^2g\,\xi=\xi^\top\nabla^2u\,\xi+L|\xi|^2\ge(\mu_0+L)|\xi|^2$ a.e.\ on $\Omega_{\rho}$. As
$\mu^\xi\ge0$, its singular part is nonnegative, hence $\mu^\xi$ dominates its absolutely continuous
part. Restricting to $\Omega_{\rho}$, yields,
\begin{equation}
  \mu^\xi\ \ge\ (\mu_0+L)\,|\xi|^2\,\mathcal{L}^d. \tag{1}
\end{equation}

The convolution $g^\epsilon\in C^\infty(\mathbb{R}^d)$ is convex, so $\nabla^2g^\epsilon\succeq0$ on
$\mathbb{R}^d$. Moreover $\nabla^2g^\epsilon=\eta_\epsilon*[D^2g]$. Hence for every $\xi\in\mathbb{R}^d$ and every $z$ such that $|z|>\rho+\epsilon$, we have $\mathcal{B}(z,\epsilon)\subset\Omega_{\rho}$. Then, estimate $(1)$ gives
\begin{align}
  \xi^\top\nabla^2g^\epsilon(z)\,\xi&=\int_{\mathbb{R}^d}\eta_\epsilon(z-w)\,d\mu^\xi(w)\ =\int_{\mathcal{B}(z,\epsilon)}\eta_\epsilon(z-w)\,d\mu^\xi(w)\nonumber\\ &\geq(\mu_0+L)|\xi|^2\int_{\mathcal{B}(z,\epsilon)}\eta_\epsilon(z-w)\,d\mathcal{L}^d(w)=\ (\mu_0+L)\,|\xi|^2. \tag{2}
\end{align}

As $g^\epsilon\in C^\infty$, the fundamental theorem of calculus applied to
$t\mapsto\langle\nabla g^\epsilon(\gamma(t)),v\rangle$, that is
\[
  \langle\nabla g^\epsilon(x)-\nabla g^\epsilon(y),\,v\rangle=\int_0^1 v^\top\nabla^2g^\epsilon(\gamma(t))\,v\,dt.
\]
Let $B_\epsilon:=\{t\in[0,1]:|\gamma(t)|\le\rho+\epsilon\}$. Notice that $B_\epsilon$ is a sublevel set of the
convex map $t\mapsto|\gamma(t)|$, and thus an interval. Writing $B_\epsilon=[t_1,t_2]$, the identity
$|\gamma(s)-\gamma(t)|=|s-t|\,|v|$ gives
\[
  \mathcal{L}^1(B_\epsilon)=t_2-t_1=\frac{|\gamma(t_2)-\gamma(t_1)|}{|v|}\le\frac{2(\rho+\epsilon)}{|v|}.
\]
Now for every
$t\in[0,1]\setminus B_\epsilon$ one has $|\gamma(t)|>\rho+\epsilon$, so $(2)$ applies pointwise and
$v^\top\nabla^2g^\epsilon(\gamma(t))\,v\ge(\mu_0+L)|v|^2$. On $B_\epsilon$ the integrand is positive (due to convexity of $g$), therefore one writes
\begin{equation}
  \langle\nabla g^\epsilon(x)-\nabla g^\epsilon(y),\,v\rangle
  \ \ge\ (\mu_0+L)\,|v|^2\,\mathcal{L}^1\big([0,1]\setminus B_\epsilon\big)
  \ \ge\ (\mu_0+L)\,|v|^2\Big(1-\tfrac{2(\rho+\epsilon)}{|v|}\Big). \tag{3}
\end{equation}

By \citet[Theorem~25.5]{Rockafellar1970}, $\Phi$ is co-null and $\nabla g$ is continuous on $\Phi$. In particular, for $x\in \Phi$, $\nabla g(x_k)\to\nabla g(x)$ as $\Phi\ni x_k\to x$. Let $x\in \Phi$ be an arbitrary element. Splitting the integral between $\Phi$ and its co-null yields:
\[
  \begin{aligned}
  |\nabla g^\epsilon(x)-\nabla g(x)|
  &\le\int_{\{x-z\in \Phi\}}\!\!\eta_\epsilon(z)\,|\nabla g(x-z)-\nabla g(x)|\,dz\\
  &\quad+\int_{\{x-z\notin \Phi\}}\!\!\eta_\epsilon(z)\,|\nabla g(x-z)-\nabla g(x)|\,dz .
  \end{aligned}
\]
The second integral vanishes identically since $\{z:x-z\notin \Phi\}$ is $\mathcal{L}^d$-null and $\nabla g$ is locally
bounded, while the first term is bounded by $\sup_{w\in \Phi\cap\mathcal{B}(x,\epsilon)}|\nabla g(w)-\nabla g(x)|
\int\eta_\epsilon = \sup_{w\in \Phi\cap\mathcal{B}(x,\epsilon)}|\nabla g(w)-\nabla g(x)|\to0$ by continuity of
$\nabla g$ at $x$. Hence $\nabla g^\epsilon(x)\to\nabla g(x)$.
Passing to the limit in $(3)$ gives $\langle\nabla g(x)-\nabla g(y),v\rangle\ge(\mu_0+L)|v|^2(1-2\rho/|v|)$. Since $\nabla g(x)=h(x)+Lx$ on $\Phi$, the left side is $\langle h(x)-h(y),v\rangle+L|v|^2$, so
\begin{equation}
  \langle h(x)-h(y),\,v\rangle\ \ge\ \mu_0|v|^2-2\rho(\mu_0+L)\,|v|. \tag{4}
\end{equation}
If $|v|=|x-y|\ge 4\rho(\mu_0+L)/\mu_0$, then $2\rho(\mu_0+L)\le\tfrac{\mu_0}{2}|v|$, hence
$2\rho(\mu_0+L)|v|\le\tfrac{\mu_0}{2}|v|^2$, and $(4)$ gives
\[
  \langle h(x)-h(y),x-y\rangle\ \ge\ \mu_0|v|^2-\tfrac{\mu_0}{2}|v|^2\ =\ \tfrac{\mu_0}{2}\,|x-y|^2.
\]
\end{proof}
\subsection{Proof of Lemma~\ref{lemma_mollify}}\hypertarget{proof_lemma_mollify}{}
\begin{proof}
Both convolutions in \eqref{eq_mollified} are well defined. In particular $u$ is locally Lipschitz by
\hyperlink{A1}{A1} and $h$ is locally bounded by \hyperlink{A3}{A3}, so both belong to
$L^1_{loc}(\mathbb{R}^d)$. Moreover $h_\epsilon$ is independent of the choice of
selection $h\in\partial u$, as any two selections agree Lebesgue-almost everywhere.
Finally $u\in W^{1,1}_{loc}(\mathbb{R}^d)$
identifies
\begin{align*}
\nabla u_\epsilon=\nabla\left(u^*\varphi_\epsilon\right)
=\left(\nabla u\right)*\varphi_\epsilon=h*\varphi_\epsilon=h_\epsilon,
\end{align*}
which is (i).

For (ii), let $x,y\in\mathbb{R}^d$ with $|x-y|\geq R$. Trivially $|(x-w)-(y-w)|\geq R$, so since
$\varphi_\epsilon\geq0$ with $\int\varphi_\epsilon=1$, applying \eqref{eqA2} yields
\begin{align*}
\left\langle h_\epsilon(x)-h_\epsilon(y),x-y\right\rangle
=\int_{\mathbb{R}^d}\left\langle h(x-w)-h(y-w),x-y\right\rangle\varphi_\epsilon(w)\,dw
\geq\mu|x-y|^{2}.
\end{align*}
Replacing \eqref{eqA2} by \eqref{eqA1}, which carries no restriction on $|x-y|$, the same
computation gives \eqref{eqA1}.

For (iii), \hyperlink{A3}{A3} together with $|x-w|\leq|x|+\epsilon$ on
$\operatorname{supp}\varphi_\epsilon$ gives
\begin{align*}
\left|h_\epsilon(x)\right|
\leq\int_{\overline{\mathcal{B}}(0,\epsilon)}\left(m+K|x-w|^{p}\right)\varphi_\epsilon(w)\,dw
\leq m+K\left(|x|+\epsilon\right)^{p},
\end{align*}
the second bound following from $\epsilon\leq1$. Since $u_\epsilon$ therefore
satisfies \hyperlink{A1}{A1}--\hyperlink{A3}{A3} by (i)--(iii), part (iv) follows from the
proof of Lemma \ref{lem:lsi}, in which the growth constants $(m,K,p)$ enter only through
$\sup_{\overline{\mathcal{B}}(0,R)}|h|$, here at most $m+K(R+\epsilon)^{p}$.

Finally (v) is the
standard convergence of mollifications since $u$ is continuous and $h\in L^1_{loc}(\mathbb{R}^d)$ \citep[\S4.2.1, Theorem~4.1]{EvansGariepy}.
\end{proof}
\subsection{Proof of Lemma~\ref{lemma_stability}}\hypertarget{proof_lemma_stability}{}
\begin{proof}
Writing the two equations in integral form,
\begin{align*}
Y_s=\theta_0-\beta\int_0^s h(Y_u)\,du+\sqrt{2}\hat{B}_s,\qquad
Y_s^\epsilon=\theta_0-\beta\int_0^s h_\epsilon(Y_u^\epsilon)\,du+\sqrt{2}\hat{B}_s,
\end{align*}
the noise cancels in the difference, so that with
$\Delta_s:=Y_s-Y_s^\epsilon$ we have
$\Delta_s=\int_0^s-\beta\bigl(h(Y_r)-h_\epsilon(Y_r^\epsilon)\bigr)\,dr$, a continuous process of finite variation. It\^o's formula
applied to $x\mapsto\left|x\right|^{2}$ thus has no second-order term and, since
$\Delta_0=0$, gives
\begin{align*}
\left|\Delta_s\right|^{2}=-2\beta\int_0^s\left\langle\Delta_r,h(Y_r)-h_\epsilon(Y_r^\epsilon)\right\rangle dr .
\end{align*}
Adding and subtracting $h_\epsilon(Y_r)$ splits the right-hand side as
\begin{align*}
\underbrace{-2\beta\int_0^s\left\langle\Delta_r,h_\epsilon(Y_r)-h_\epsilon(Y_r^\epsilon)\right\rangle dr}_{=:T_1}
\;+\;\underbrace{-2\beta\int_0^s\left\langle\Delta_r,h(Y_r)-h_\epsilon(Y_r)\right\rangle dr}_{=:T_2}.
\end{align*}
By \hyperlink{A1}{A1} for $h_\epsilon$, which holds with the same constant $L$ by
Lemma \ref{lemma_mollify}(ii), we get
\begin{align*}
T_1\leq2\beta L\int_0^s\left|\Delta_r\right|^{2}dr,
\end{align*}
while by the Cauchy--Schwarz and Young inequalities
\begin{align*}
T_2\leq2\beta\int_0^s\left|\Delta_r\right|\left|h(Y_r)-h_\epsilon(Y_r)\right|dr
\leq\beta\int_0^s\left|\Delta_r\right|^{2}dr
+\beta\int_0^s\left|h(Y_r)-h_\epsilon(Y_r)\right|^{2}dr .
\end{align*}
Combining the two bounds and taking expectations,
\begin{align*}
\mathbb{E}\left[\left|\Delta_s\right|^{2}\right]
\leq\beta\left(2L+1\right)\int_0^s\mathbb{E}\left[\left|\Delta_r\right|^{2}\right]dr
+\beta\int_0^s\mathbb{E}\left|h(Y_r)-h_\epsilon(Y_r)\right|^{2}dr .
\end{align*}
Both terms are finite, as Gr\"onwall's inequality requires: $|x|^{2}\leq V_{2p}(x)$ gives
$\left|\Delta_r\right|^{2}\leq2\left(V_{2p}(Y_r)+V_{2p}(Y_r^{\epsilon})\right)$, and
\hyperlink{A3}{A3} with Lemma \ref{lemma_mollify}(iii) gives
$\left|h(x)\right|^{2}+\left|h_{\epsilon}(x)\right|^{2}\leq C\left(1+|x|^{2p}\right)\leq2C\,V_{2p}(x)$
for every $\epsilon\in(0,1]$. Since $h_{\epsilon}$ inherits \hyperlink{A2}{A2} by
Lemma \ref{lemma_mollify}(ii), Lemma \ref{lemmaB2} at order $2p$ applies to $Y$ and
$Y^{\epsilon}$ with $\epsilon$-independent constants and yields, by \hyperlink{A4}{A4},
\begin{align}
\sup_{\epsilon\in(0,1]}\ \sup_{r\geq0}\
\mathbb{E}\left[V_{2p}(Y_r)+V_{2p}(Y_r^{\epsilon})\right]<\infty .
\label{eq_stab_moments}
\end{align}
Gr\"onwall's inequality therefore gives
\begin{align}
\mathbb{E}\left[\left|\Delta_s\right|^{2}\right]
\leq\beta e^{\beta(2L+1)s}\int_0^s\mathbb{E}\left|h(Y_r)-h_\epsilon(Y_r)\right|^{2}dr .
\label{eq_gronwall_stab}
\end{align}
It remains to show that the integral in \eqref{eq_gronwall_stab} vanishes as
$\epsilon\to0$. Set $f_\epsilon:=\left|h_\epsilon-h\right|^{2}$. By the
above,
\begin{align*}
\left|f_\epsilon(x)\right|\leq C\left(1+|x|^{2p}\right)=:g(x),
\qquad \mathbb{E}\left[g(Y_r)\right]<\infty,
\end{align*}
so $f_\epsilon(Y_r)$ is dominated by an integrable function, uniformly in
$\epsilon$. Moreover, by Lemma \ref{lemma_mollify}(v), $f_\epsilon\to0$ pointwise on $D$, whose
complement is Lebesgue-null. Since the drift only reweighs the law of the driving
Brownian path, one has $\mathcal{L}(Y_r)\ll\mathrm{Leb}$, so $Y_r$ almost surely
avoids this exceptional set and $h_\epsilon(Y_r)\to h(Y_r)$ $\mathbb{P}$-a.s. Taking any sequence
$\epsilon_k\to0$ and writing $f_k:=f_{\epsilon_k}$, the dominated convergence
theorem yields
\begin{align*}
\lim_{k\to\infty}\mathbb{E}\left[\left|h_{\epsilon_k}(Y_r)-h(Y_r)\right|^{2}\right]=0,
\end{align*}
and since the limit is the same along every such sequence, the convergence holds as
$\epsilon\to0$. Hence, by \eqref{eq_gronwall_stab},
$\mathbb{E}\left|Y_s-Y_s^\epsilon\right|^{2}\to0$ for every $s\geq0$. As
$\left(Y_s,Y_s^\epsilon\right)$ is a coupling of $\nu Q_s$ and $\nu Q_s^\epsilon$,
\begin{align*}
W_2\left(\nu Q_s,\nu Q_s^\epsilon\right)^{2}
\leq\mathbb{E}\left|Y_s-Y_s^\epsilon\right|^{2}\longrightarrow0,\ \text{as } \epsilon\to 0.
\end{align*}
\end{proof}
\subsection{Proof of Lemma~\ref{lemma_pi_conv}}\hypertarget{proof_lemma_pi_conv}{}
\begin{proof}
Since \hyperlink{A1}{A1}--\hyperlink{A3}{A3} hold for $h_\epsilon$ with constants
independent of $\epsilon$ by Lemma \ref{lemma_mollify}(ii)--(iii), one shows as in
Remark \ref{remark_diss} that dissipativity holds uniformly in $\epsilon$, whence there
are $c,\widehat{R}>0$, independent of $\epsilon$, with
$u_\epsilon(x)\geq\tfrac{\mu}{4}|x|^{2}-c$ for all $|x|\geq\widehat{R}$. Therefore
\begin{align}
e^{-\beta u_\epsilon(x)}\leq e^{\beta c}e^{-\beta\mu|x|^{2}/4}\in L^{1}(\mathbb{R}^d),
\qquad\forall\epsilon\in(0,1].\label{eq_gauss_dom}
\end{align}
Hence, by the dominated convergence theorem,
$Z_\epsilon=\int e^{-\beta u_\epsilon}\to\int e^{-\beta u}=Z$, and therefore
$\pi_\beta^\epsilon\to\pi_\beta$ pointwise for every $x\in\mathbb{R}^d$. Similarly one
uses dominated convergence, bounding via $|x|^{2}$ times the right-hand side of
\eqref{eq_gauss_dom}, for the second moment $\int_{\mathbb{R}^d}|x|^{2}e^{-\beta u_\epsilon(x)}dx$, to conclude that $\pi_\beta^\epsilon\to\pi_\beta$ in $\mathcal{P}_2(\mathbb{R}^d)$. Hence
\citet[Theorem~6.9]{Villani_new} gives
$W_2(\pi_\beta^\epsilon,\pi_\beta)\to0$ as $\epsilon\to0$.
\end{proof}
\noindent For $\mu\in\mathcal{P}(\mathbb{R}^d)$ with $\mu=f\pi_\beta^\epsilon$ we write
$H_\epsilon(\mu):=\int f\log f\,d\pi_\beta^\epsilon$, and $H_\epsilon(\mu):=+\infty$
otherwise.

\subsection{Proof of Proposition~\ref{prop_contraction1}}\hypertarget{proof_prop_contraction1}{}
\begin{proof}
By \hyperlink{A1}{A1}--\hyperlink{A2}{A2} for $u_\epsilon\in C^2$, condition (1.8) of
\cite{Wang2020} holds with $I=0$, since we are in the flat Euclidean case. We argue at the
level of the regularized semigroup $Q_s^\epsilon$ and let $\epsilon\to0$ at the end.
Four key estimates are needed, the first three govern the behavior for large $s$, the fourth
for small $s$.
\emph{(A) Talagrand inequality.} By \citet[Theorem~1]{Shao2006} the logarithmic Sobolev
inequality \eqref{eq_lsi} with constant $\rho_\epsilon$ implies
\begin{align}
W_2^{2}\left(\nu,\pi_\beta^\epsilon\right)\leq\rho_\epsilon H_\epsilon(\nu),
\qquad\forall\nu\in\mathcal{P}_2(\mathbb{R}^d).\label{eq_A_talagrand}
\end{align}
\emph{(B) Exponential decay of entropy.} By Proposition~3.3(2) in \citet{Wang2020},
\begin{align}
H_\epsilon\left(\nu Q_s^\epsilon\right)\leq e^{-\frac{4s}{\rho_\epsilon}}H_\epsilon(\nu),
\qquad s\geq0.\label{eq_B_decay}
\end{align}
\emph{(C) Log-Harnack inequality and entropy-cost estimate.} Let $\varphi$ be bounded with
$\inf\varphi>0$ and $\int\varphi\,d\pi_\beta^\epsilon=1$. By
\citet[Theorem~1.1]{Wang2010},
\begin{align}
Q_s^\epsilon\left(\log\varphi\right)(x)\leq\log Q_s^\epsilon\varphi(y)
+\dfrac{\beta L|x-y|^{2}}{2\left(1-e^{-2\beta Ls}\right)},
\qquad s>0,\ x,y\in\mathbb{R}^d.\label{eq_logharnack}
\end{align}
Let $X\sim\nu$ and $Y\sim\pi_\beta^\epsilon$ be
such that $\mathbb{E}|X-Y|^{2}=W_2^{2}(\nu,\pi_\beta^\epsilon)$. Evaluating
\eqref{eq_logharnack} at $(X,Y)$ and taking expectations,
\begin{align*}
\int\log\varphi\,d\left(\nu Q_s^\epsilon\right)
=\mathbb{E}\left[Q_s^\epsilon\left(\log\varphi\right)(X)\right]
\leq\mathbb{E}\left[\log Q_s^\epsilon\varphi(Y)\right]
+\dfrac{\beta L}{2\left(1-e^{-2\beta Ls}\right)}W_2^{2}\left(\nu,\pi_\beta^\epsilon\right),
\end{align*}
the first equality being the definition of $\nu Q_s^\epsilon$. The remaining expectation
vanishes, since by the concavity of the logarithm and the invariance of
$\pi_\beta^\epsilon$, we write
\begin{align*}
\mathbb{E}\left[\log Q_s^\epsilon\varphi(Y)\right]
=\int\log Q_s^\epsilon\varphi\,d\pi_\beta^\epsilon
\leq\log\int Q_s^\epsilon\varphi\,d\pi_\beta^\epsilon
=\log\int\varphi\,d\pi_\beta^\epsilon=0 .
\end{align*}
Let $f$ be the density of $\nu Q_s^\epsilon$ with
respect to $\pi_\beta^\epsilon$, which is admissible after truncation and satisfies
$\int f\,d\pi_\beta^\epsilon=1$. Taking $\varphi=f$ turns the left-hand side into
$\int f\log f\,d\pi_\beta^\epsilon=H_\epsilon\left(\nu Q_s^\epsilon\right)$, whence
\begin{align}
H_\epsilon\left(\nu Q_s^\epsilon\right)
\leq\dfrac{\beta L}{2\left(1-e^{-2\beta Ls}\right)}W_2^{2}\left(\nu,\pi_\beta^\epsilon\right),
\qquad s>0.\label{eq_C_transport}
\end{align}
\emph{(D) Short-time estimate.} Synchronous coupling together with Gr\"onwall's lemma, as
in the proof of Lemma \ref{lemma_stability}, gives
\begin{align}
W_2\left(\nu Q_s^\epsilon,\pi_\beta^\epsilon\right)
\leq e^{\beta Ls}W_2\left(\nu,\pi_\beta^\epsilon\right),\qquad s\geq0.\label{eq_D_short}
\end{align}
Now we assemble the above estimates. Fix a split, at time $s_0>0$. For $s\geq s_0$, chaining \eqref{eq_A_talagrand},
\eqref{eq_B_decay} and \eqref{eq_C_transport},
\begin{align*}
W_2^{2}\left(\nu Q_s^\epsilon,\pi_\beta^\epsilon\right)
&\leq\rho_\epsilon H_\epsilon\left(\nu Q_s^\epsilon\right)
\leq\rho_\epsilon e^{-\frac{4(s-s_0)}{\rho_\epsilon}}H_\epsilon\left(\nu Q_{s_0}^\epsilon\right)\\
&\leq\rho_\epsilon e^{-\frac{4(s-s_0)}{\rho_\epsilon}}
\dfrac{\beta L}{2\left(1-e^{-2\beta Ls_0}\right)}W_2^{2}\left(\nu,\pi_\beta^\epsilon\right),
\end{align*}
that is,
\begin{align*}
W_2\left(\nu Q_s^\epsilon,\pi_\beta^\epsilon\right)
\leq\left(e^{\frac{2s_0}{\rho_\epsilon}}
\sqrt{\dfrac{\rho_\epsilon\beta L}{2\left(1-e^{-2\beta Ls_0}\right)}}\right)
e^{-\frac{2s}{\rho_\epsilon}}W_2\left(\nu,\pi_\beta^\epsilon\right).
\end{align*}
For $s\leq s_0$, \eqref{eq_D_short} gives instead
\begin{align*}
W_2\left(\nu Q_s^\epsilon,\pi_\beta^\epsilon\right)
\leq e^{\left(\beta L+\frac{2}{\rho_\epsilon}\right)s}e^{-\frac{2s}{\rho_\epsilon}}
W_2\left(\nu,\pi_\beta^\epsilon\right)
\leq e^{\left(\beta L+\frac{2}{\rho_\epsilon}\right)s_0}e^{-\frac{2s}{\rho_\epsilon}}
W_2\left(\nu,\pi_\beta^\epsilon\right).
\end{align*}
Both regimes carry the same exponential rate, so for every $s\geq0$
\begin{align}
W_2\left(\nu Q_s^\epsilon,\pi_\beta^\epsilon\right)
\leq C_w^\epsilon e^{-\frac{2s}{\rho_\epsilon}}W_2\left(\nu,\pi_\beta^\epsilon\right),
\quad
C_w^\epsilon=\max\left\{e^{\left(\beta L+\frac{2}{\rho_\epsilon}\right)s_0},\;
e^{\frac{2s_0}{\rho_\epsilon}}
\sqrt{\dfrac{\rho_\epsilon\beta L}{2\left(1-e^{-2\beta Ls_0}\right)}}\,\right\}.
\label{eq_contraction_eps}
\end{align}
Writing
\begin{align*}
\psi(s):=e^{\frac{2s}{\rho_\epsilon}}
\sqrt{\dfrac{\rho_\epsilon\beta L}{2\left(1-e^{-2\beta Ls}\right)}},
\qquad
\psi'(s)=\psi(s)\left(\dfrac{2}{\rho_\epsilon}
-\dfrac{\beta Le^{-2\beta Ls}}{1-e^{-2\beta Ls}}\right),
\end{align*}
the derivative vanishes precisely when $e^{2\beta Ls}=1+\rho_\epsilon\beta L/2$, that is at
\begin{align}
s_\star=\dfrac{1}{2\beta L}\log\left(1+\dfrac{\rho_\epsilon\beta L}{2}\right).
\label{eq_tstar}
\end{align}
There $1-e^{-2\beta Ls_\star}=\rho_\epsilon\beta L/\left(2+\rho_\epsilon\beta L\right)$,
so the second expression in \eqref{eq_contraction_eps} equals
\begin{align*}
e^{\frac{2s_\star}{\rho_\epsilon}}
\sqrt{\dfrac{\rho_\epsilon\beta L}{2\left(1-e^{-2\beta Ls_\star}\right)}}
=\left(1+\dfrac{\rho_\epsilon\beta L}{2}\right)^{\frac{1}{2}}
\left(1+\dfrac{\rho_\epsilon\beta L}{2}\right)^{\frac{1}{\rho_\epsilon\beta L}},
\end{align*}
Notice that the first expression in \eqref{eq_tstar} agrees with the second for $s=s_{\star}$. Since
$\left(1+r\right)^{1/r}\leq e$ for every $r>0$, one finally obtains
\begin{align}
C_w^\epsilon\leq\sqrt{e\left(1+\dfrac{\rho_\epsilon\beta L}{2}\right)}.
\label{eq_Cw_final}
\end{align}
The laws converge in $W_2$, namely
$\nu Q_s^\epsilon\to\nu Q_s$ by Lemma \ref{lemma_stability} and
$\pi_\beta^\epsilon\to\pi_\beta$ by Lemma \ref{lemma_pi_conv}, so by
\citet[Corollary~6.11]{Villani_new} both sides of \eqref{eq_contraction_eps} pass to the
limit. Additionally all constants involved depend continuously on $\rho_\epsilon$. Since
$P_t=Q_{t/\beta}$, substituting $s=t/\beta$ turns the rate $2/\rho(\beta)$ into $C_r$ and $t_\star=\beta s_\star$, giving \eqref{eq_contraction_main} with $C_w$
as in \eqref{eq_Cw_statement}. Uniqueness follows trivially by contradiction.
\end{proof}

\begin{remark}\label{remark_split}
The split \eqref{eq_tstar} is optimal. The corresponding argument of
Theorem~2.1(3) in \citet{Wang2020} takes $t_0=1$, which suffices to produce a finite
prefactor.
\end{remark}
\subsection{Proof of Lemma~\ref{MY-PL}}\hypertarget{proof_lemma_MY-PL}{}
\begin{proof}
Since \[\begin{aligned}u_\epsilon(x)-u^*&=u(x_\epsilon)-u^* +\frac{1}{2\epsilon} |x-x_\epsilon|^2 \\&\leq \frac{1}{2C_{PL}}(\operatorname*{dist}(0,\partial u(x_\epsilon)))^2 + \frac{\epsilon}{2} \frac{1}{\epsilon^2}|x-x_\epsilon|^2\\&\leq (\frac{1}{2C_{PL}}+\frac{\epsilon}{2}) |\nabla u_\epsilon(x)|^2.\end{aligned}   \]
\end{proof}
\subsection{Proof of Lemma~\ref{lemma-MY-bound}}\hypertarget{proof_lemma-MY-bound}{}
\begin{proof}
    By the semi-convexity \hyperlink{A1}{A1} of $u$ there holds
    \[u(x_\epsilon)\geq u(x) + \langle g,x-x_\epsilon\rangle - \frac{K}{2} |x-x_\epsilon|^2= u(x) - \epsilon \langle g,\nabla u_\epsilon(x)\rangle- \frac{K}{2} \epsilon^2 |\nabla u_\epsilon(x)|^2.\]
    Equivalently, one writes
    \begin{equation}\label{eq-1}
       \epsilon \langle g,\nabla u_\epsilon(x)\rangle\geq u(x)-u(x_\epsilon)-\frac{K}{2} \epsilon^2 |\nabla u_\epsilon(x)|^2. \end{equation}
    Recall that by definition
   \begin{equation}\label{eq-2}
       u_\epsilon(x)= u(x_\epsilon)+ \frac{1}{2\epsilon}|x-x_\epsilon|^2\leq u(x_\epsilon) +\frac{\epsilon}{2} |\nabla u_\epsilon|^2.
   \end{equation}
   Rearranging \eqref{eq-2} and combining with \eqref{eq-1} yields
   \[\epsilon \langle g,\nabla u_\epsilon(x)\rangle\geq u(x) -u_\epsilon(x)+\frac{\epsilon}{2} |\nabla u_\epsilon|^2 -\frac{K}{2} \epsilon^2 |\nabla u_\epsilon(x)|^2,\]
  and since $u_\epsilon(x)\leq u(x)$ leads to
   \[\epsilon \langle g,\nabla u_\epsilon(x)\rangle\geq \epsilon \left(1- \frac{\epsilon K}{2}\right)  |\nabla u_\epsilon(x)|^2\geq (\epsilon/2)|\nabla u_\epsilon(x)|^2.\]
   Dividing by $\epsilon$ yields the first claim. Furthermore, applying the Cauchy--Schwarz inequality and the polynomial growth \hyperlink{A3}{A3} of $g$ concludes the proof.
\end{proof}
\subsection{Proof of Lemma~\ref{approximation MY}}\hypertarget{proof_lemma-approximation MY}{}
\begin{proof}
One writes \[u(x)-u_\epsilon(x)=\frac{1}{2\epsilon}|x-x_\epsilon|^2=\epsilon |\nabla u_\epsilon(x)|^2 \leq 2\epsilon |g(x)|^2,\] where $g(x)$ is an element of the subdifferential of $u.$
Since any subgradient has polynomial growth the result follows immediately.
\end{proof}
\section{Supporting Lemmas}
\begin{lemma}\label{lemmaB1}
    Let Assumptions \hyperlink{A2}{A2}--\hyperlink{A3}{A3} hold. For any $p\geq 1$, and $x\in\mathbb{R}^d$, we have\begin{align*}
        \mathcal{A}V_p(x)\leq C_2V_p(x),
    \end{align*}
    where $C_2=p(b+\beta^{-1}(d+(p-2)^+))$.
\end{lemma}
    \begin{proof}
        Let $p\geq 2$, by direct calculation
        \begin{align*}
            \mathcal{A}V_p(x)&=-p(1+|x|^2)^{p/2-1}\langle x,h(x)\rangle+\beta^{-1}p(p-2)|x|^2(1+|x|^2)^{p/2-2}\\
            &\quad+d\beta^{-1}p(1+|x|^2)^{p/2-1}\\
            &\leq -p(1+|x|^2)^{p/2-1}(\mu/2|x|^2-b))+\beta^{-1}p(d+(p-2))(1+|x|^2)^{p/2}\\
            &\leq bp(1+|x|^2)^{p/2-1}-(\mu/2)p(1+|x|^2)^{p/2-1}|x|^2+\beta^{-1}p(d+(p-2))V_p(x)\\
            &\leq bpV_p(x)+\beta^{-1}p(d+(p-2))V_p(x)\leq C_2V_p(x).
        \end{align*}
    \end{proof}
\begin{lemma}\label{lemmaB2}
    Let Assumptions \hyperlink{A2}{A2}--\hyperlink{A3}{A3} hold. For any $p\geq1$, and $x\in\mathbb{R}^d$, we have\begin{align*}
        \mathcal{A}V_p(x)\leq -C_3(p)V_p(x)+C_4(p),
    \end{align*}
    where $C_4(p)/C_3(p)\leq A^{\max\{1,p/2\}}$, with $A(p)=1+2(b+\beta^{-1}(d+p-2))/\mu$.
\end{lemma}
    \begin{proof}
        Let $x\in\mathbb{R}^d$ and $p\geq 2$, one calculates 
        \begin{align}
            \mathcal{A}V_p(x)&=-p(1+|x|^2)^{p/2-1}\langle x,h(x)\rangle+\beta^{-1}p(p-2)|x|^2(1+|x|^2)^{p/2-2}\nonumber\\
            &\quad+d\beta^{-1}p(1+|x|^2)^{p/2-1}\nonumber\\
            &\leq -(\mu/2)p|x|^2(1+|x|^2)^{p/2-1}+(b+d\beta^{-1})p(1+|x|^2)^{p/2-1}\nonumber\\
            &\quad+\beta^{-1}p(p-2)(1+|x|^2)^{p/2-1}\nonumber\\
            &\leq -(\mu/2)p|x|^2(1+|x|^2)^{p/2-1}+(b+d\beta^{-1}+\beta^{-1}(p-2))p(1+|x|^2)^{p/2-1}\nonumber\\
            &\leq -(\mu/2)p(1+|x|^2)(1+|x|^2)^{p/2-1}+(\mu/2+b+\beta^{-1}(d+p-2))p(1+|x|^2)^{p/2-1}\nonumber\\
            &\leq (\mu/2)p\left(A(1+|x|^2)^{p/2-1}-(1+|x|^2)^{p/2}\right).\label{lyap_est1}
        \end{align}
        The weighted arithmetic-geometric mean inequality implies $a_1^{w_1}a_2^{w_2}\leq w_1a_1+w_2a_2$, for weights living on a simplex. Using it with exponents $w_1=(p-2)/p$ and $w_2=2/p$, yields
        \begin{align*}
            A(1+|x|^2)^{p/2-1}=A^{p/2\times2/p}(1+|x|^2)^{p/2\times (p-2)/p}\leq \dfrac{2}{p}A^{p/2}+\dfrac{p-2}{p}(1+|x|^2)^{p/2}
        \end{align*}
        In view of \eqref{lyap_est1} one gets
        \begin{align*}
            \mathcal{A}V_p(x)\leq -\mu V_p(x)+\mu A^{p/2}
        \end{align*}
        When $p\in[1,2)$, one immediately gets for any $x\in\mathbb{R}^d$
        \begin{align*}
            \mathcal{A}V_p(x)\leq-(p\mu/2)V_p(x)+(p\mu/2)(1+2(b+d\beta^{-1})/\mu).
        \end{align*}
        \end{proof}
\begin{lemma}
        Let $n_0\in\mathbb{N}$, $\epsilon>0$, and $q\in(0,1)$. If a non-negative sequence $\{a_n\}_{n\in\mathbb{N}}$ satisfies
        \begin{align*}
            a_n\leq \epsilon+qa_{n-n_0},\ \forall n\geq n_0,
        \end{align*}
        then 
        \begin{align*}
            a_n\leq \dfrac{\epsilon}{1-q}+\max_{i\in\{0,\ldots,n_0-1\}}\{a_i\} q^{n/n_0-1},\ \forall n\in\mathbb{N}_0.
        \end{align*}
    \end{lemma}
\begin{proof}
        Write $n=kn_0+i$ with $k=\lfloor n/n_0\rfloor$ and $i\in\{0,\ldots,n_0-1\}$. Iterating the
        hypothesis $k$ times gives
        $a_n\leq \epsilon(1+q+\cdots+q^{k-1})+q^k a_i\leq \epsilon/(1-q)+q^{k}\max_{i\in\{0,\ldots,n_0-1\}}a_i$,
        and $q^k\leq q^{n/n_0-1}$ since $k\geq n/n_0-1$.
    \end{proof}
\begin{lemma}\label{lemma_crude}
        Let Assumptions \hyperlink{A2}{A2}--\hyperlink{A4}{A4} hold and $\lambda_0\in(0,1/(4\mu))$.
        Then, for every $\lambda\in(0,\lambda_0)$ and $k\in\mathbb{N}_0$,
        \begin{align*}
            W_2\left(\mathcal{L}(\theta_k^{\lambda}),\pi_{\beta}\right)\leq \left(2C_{B_1}+2A(2)\right)^{1/2}.
        \end{align*}
    \end{lemma}
    \begin{proof}
    Let $\theta_{\infty}\sim\pi_{\beta}$ be independent of $\theta_k^{\lambda}$. The independent
    coupling and Lemma \ref{lemma_MB} give
    \begin{align}
        W^2_2(\mathcal{L}(\theta_k^{\lambda}),\pi_{\beta})&\leq\mathbb{E}\left[|\theta_k^{\lambda}-\theta_{\infty}|^2\right]\leq2\mathbb{E}\left[|\theta_k^{\lambda}|^2\right]+2\int_{\mathbb{R}^d}|x|^2\,\pi_{\beta}(dx)\nonumber\\
        &\leq 2C_{B_1}+2\int_{\mathbb{R}^d}V_2(x)\,\pi_{\beta}(dx).\nonumber
        \end{align}
        Since $\pi_{\beta}$ is the invariant measure of the SDE \eqref{SDE_1}, one has
        $\int_{\mathbb{R}^d}\mathcal{A}V_2\,d\pi_{\beta}=0$ (see the proof of
        Lemma~\ref{lemmaER1}), so Lemma~\ref{lemmaB2} at $p=2$ yields
        \begin{align*}
            \int_{\mathbb{R}^d}V_2\,d\pi_{\beta}\leq \dfrac{C_4(2)}{C_3(2)}=A(2),
        \end{align*}
        which concludes the proof.
    \end{proof}
\begin{lemma}\label{lem:proxy}
The strongly convex proxy.\\
Let $\pi_{\beta}\propto e^{-\beta u}$ and $W:=\beta u$. Suppose $u$ satisfies
\hyperlink{A1}{A1}, \hyperlink{A2}{A2} and \hyperlink{A3}{A3}. Then there exists
$W_{0}:\mathbb{R}^{d}\to\mathbb{R}$, $\beta\mu$-strongly convex on $\mathbb{R}^{d}$, such that
$W_{0}=W$ on $\mathcal{B}^{c}(0,R)$. Writing $\tilde{W}:=W-\tfrac{\beta\mu}{2}|\cdot|^{2}$
and $W_{0}=:V+\tfrac{\beta\mu}{2}|\cdot|^{2}$, the convex part satisfies
\[
\inf_{|y|=R}\tilde{W}(y)\;\le\;V(x)\;\le\;\sup_{|y|=R}\tilde{W}(y),
\qquad \forall x\in\mathcal{B}(0,R),
\]
and $\psi:=W-W_{0}$ is supported in $\overline{\mathcal{B}}(0,R)$.
\end{lemma}

\begin{proof}
As $W=\beta u$ is $\beta\mu$-strongly convex on $\mathcal{B}^{c}(0,R)$, the function
$\tilde{W}=W-\tfrac{\beta\mu}{2}|\cdot|^{2}$ is convex on $\mathcal{B}^{c}(0,R)$. Let $V$ be the convex hull of
$\tilde{W}|_{\mathcal{B}^{c}(0,R)}$, the largest convex minorant on $\mathbb{R}^{d}$,
\[
V(x)=\inf\Big\{\textstyle\sum_{i}\lambda_{i}\tilde{W}(x_{i})\;:\;\lambda_{i}\ge0,\
\sum_{i}\lambda_{i}=1,\ \sum_{i}\lambda_{i}x_{i}=x,\ x_{i}\in\mathcal{B}^{c}(0,R)\Big\}
\]
(see Proposition~2.13 in \citealp{Rockafellar_variational}), finite by the coercivity implied
by Remark~\ref{remark_diss}. Then $V$ is convex on $\mathbb{R}^{d}$ and $V=\tilde{W}$ on
$\mathcal{B}^{c}(0,R)$. Moreover, Lemma~B.2 in \cite{Flammarion_supp}, sharpens the bound trivial bounds from $|y|\ge R$ to the
sphere $|y|=R$. Define $W_{0}:=V+\tfrac{\beta\mu}{2}|\cdot|^{2}$. Since
$V=W_{0}-\tfrac{\beta\mu}{2}|\cdot|^{2}$ is convex by construction, $W_{0}$ is $\beta\mu$-strongly convex
on $\mathbb{R}^{d}$ and we have $W_{0}=\tilde{W}+\tfrac{\beta\mu}{2}|\cdot|^{2}=W$ on
$\mathcal{B}^{c}(0,R)$. Hence $\psi=W-W_{0}=\tilde{W}-V$ vanishes on $\mathcal{B}^{c}(0,R)$.
\end{proof}

\begin{lemma}\label{lem:lsi}
The oscillation bound.\\
Under the hypotheses of Lemma~\ref{lem:proxy}, the perturbation $\psi=W-W_{0}$ satisfies
\[
\operatorname{osc}(\psi):=\sup_{\mathbb{R}^{d}}\psi-\inf_{\mathbb{R}^{d}}\psi
\;\le\;\beta\Big[\,4R\,(m+KR^{p})+\tfrac{\mu}{2}R^{2}\,\Big].
\]
Consequently $\pi_{\beta}\propto e^{-\beta U}$ satisfies the log-Sobolev inequality \eqref{eq_lsi} with
\[
\rho(\beta)\;\le\;\frac{2}{\beta\mu}\,
\exp\!\Big(\beta\big[\,4R\,(m+KR^{p})+\tfrac{\mu}{2}R^{2}\,\big]\Big).
\]
\end{lemma}

\begin{proof}
Fix $x,y\in\mathcal{B}(0,R)$ and
$\gamma_{t}=y+t(x-y)\in\mathcal{B}(0,R)$. By \hyperlink{A1}{A1} the function
$\Phi:=W+\tfrac{\beta L}{2}|\cdot|^{2}$ is convex, so the convex fundamental theorem of calculus \citep[Theorem~D.2.3.4]{HUL}
gives $\Phi(x)-\Phi(y)=\int_{0}^{1}\langle g_{t},x-y\rangle\,dt$ with
$g_{t}\in\partial\Phi(\gamma_{t})=\beta\,\partial u(\gamma_{t})+\beta L\gamma_{t}$. Choosing
$g_{t}=\beta h_{t}+\beta L\gamma_{t}$, $h_{t}\in\partial u(\gamma_{t})$, and using
$\int_{0}^{1}\langle\beta L\gamma_{t},x-y\rangle\,dt=\tfrac{\beta L}{2}(|x|^{2}-|y|^{2})$,
the semi-convexity quadratic cancels, hence
\[
W(x)-W(y)=\beta\!\int_{0}^{1}\langle h_{t},x-y\rangle\,dt .
\]
Since $|h_{t}|\le m+K|\gamma_{t}|^{p}\le m+KR^{p}$ by \hyperlink{A3}{A3} and $|x-y|\le 2R$,
\[
\operatorname{OSC}_{R}(W):=\sup_{x,y\in\mathcal{B}(0,R)}|W(x)-W(y)|\;\le\;2R\beta\,(m+KR^{p}).
\]
As $\psi=\tilde{W}-V$ vanishes on
$\mathcal{B}^{c}(0,R)$, its oscillation is attained on $\mathcal{B}(0,R)$. The bound of
Lemma~\ref{lem:proxy} gives 
\[
\sup_{\mathcal{B}(0,R)}\psi\le\sup_{\mathcal{B}(0,R)}\tilde{W}-\inf_{|y|=R}\tilde{W},
\qquad
\inf_{\mathcal{B}(0,R)}\psi\ge\inf_{\mathcal{B}(0,R)}\tilde{W}-\sup_{|y|=R}\tilde{W},
\]
hence $\operatorname{osc}(\psi)\le\operatorname{OSC}_{R}(\tilde{W})
+\operatorname{OSC}_{\partial\mathcal{B}}(\tilde{W})$, with
$\operatorname{OSC}_{\partial\mathcal{B}}$ the oscillation over $\{|y|=R\}$. On that sphere
$\tilde{W}=W-\tfrac{\beta\mu}{2}R^{2}$ differs from $W$ by a constant, so
$\operatorname{OSC}_{\partial\mathcal{B}}(\tilde{W})=\operatorname{OSC}_{\partial\mathcal{B}}(W)
\le\operatorname{OSC}_{R}(W)$, while
$\operatorname{OSC}_{R}(\tilde{W})\le\operatorname{OSC}_{R}(W)+\tfrac{\beta\mu}{2}R^{2}$.
Therefore
\[
\operatorname{osc}(\psi)\;\le\;2\operatorname{OSC}_{R}(W)+\tfrac{\beta\mu}{2}R^{2}
\;\le\;\beta\Big[\,4R\,(m+KR^{p})+\tfrac{\mu}{2}R^{2}\,\Big].
\]
The reference $\pi_{0}\propto e^{-W_{0}}$ has $\beta\mu$-strongly
convex potential. For $t,s>0$, $t+s=1$, the identity
$t|x|^{2}+s|y|^{2}-|tx+sy|^{2}=ts|x-y|^{2}$ and convexity of
$V=W_{0}-\tfrac{\beta\mu}{2}|\cdot|^{2}$ give
\[
tW_{0}(x)+sW_{0}(y)-W_{0}(tx+sy)\;\ge\;\frac{\beta\mu\,ts}{2}\,|x-y|^{2},
\]
which is condition (3.1) of \cite{BobkovLedoux} with $c=\beta\mu$. By Proposition~3.1 therein,
 without requiring additional smoothness of the potential, guarantees that $\pi_{0}$ satisfies the log-Sobolev inequality with
$\rho_{0}\le\tfrac{2}{\beta\mu}$. Since $\tfrac{d\pi_{\beta}}{d\pi_{0}}\propto e^{-\psi}$ with
$\psi$ bounded, the Holley--Stroock perturbation principle (Theorem~1, Chapter~5 in
\citealp{Bakry_ref}) yields
\[
\rho(\beta)\;\le\;\rho_{0}\,e^{\operatorname{osc}(\psi)}
\;\le\;\frac{2}{\beta\mu}\,
\exp\!\Big(\beta\big[\,4R\,(m+KR^{p})+\tfrac{\mu}{2}R^{2}\,\big]\Big),
\]
which is \eqref{eq_lsi} with the stated constant.
\end{proof}
\begin{remark}\label{expo_lsi}
The exponential dependence on $\beta$ in Lemma~[13] is optimal. For a target that
is nonconvex on a compact set, such as a mixture of Gaussians or a double-well
potential, the logarithmic Sobolev and Poincar\'e constants grow
like $e^{\beta\Delta}$, where $\Delta$ is the energy barrier separating the wells. This constant was shown to be sharp in \cite{MenzSchlichting2014}.
\end{remark}
\section{Nanochat Architecture \& Bounds}\label{sec:nanochat}

The potentials arising in the pretraining of language models are simultaneously
non-convex, non-smooth and of superlinear gradient growth, which is precisely
the regime this work addresses. It is nevertheless one thing to assert that a
class of potentials falls within the scope of
\hyperlink{A1}{A1}--\hyperlink{A3}{A3}, and quite another to verify it for a
fully specified architecture. The purpose of this section is to carry out that
verification for one such architecture, and to exhibit the resulting constants
explicitly in terms of its depth, width and vocabulary.
We work with the architecture of \emph{nanochat}, a compact version of ChatGPT2 created by Karpathy.

\medskip
\noindent The section is organized as follows. Sections~\ref{sub:nc-arch} and
\ref{sub:nc-prim} lay out the architecture and gather the exact bounds on the
scalar and vector nonlinearities the model uses, all of which are uniform in the
parameter. Section~\ref{sub:nc-forward} writes the forward map and the potential
$u$, and Section~\ref{sub:nc-hess} the curvature of $u$. Sections~\ref{sub:nc-calc}
and \ref{sub:nc-rec} introduce an exponent tracking notation for polynomially
bounded maps, which enables us in
Section~\ref{sub:nc-bound}, to obtain the curvature bound
\begin{align*}
  \nabla^2_{\theta}u(\theta)\ \succeq\ -c^{\star}\bigl(1+|\theta|\bigr)^{12\mathsf{L}+2}\,\Idm .
\end{align*}
Section~\ref{sub:nc-reg} then adds a higher-order regularization term and
recovers \hyperlink{A1}{A1}--\hyperlink{A3}{A3} for the regularized potential.

\medskip
\noindent Throughout this section the depth of the network is denoted
$\mathsf{L}$, as $L$ already denotes the semi-convexity constant of \hyperlink{A1}{A1}.

\subsection{Architecture and Parameters}\label{sub:nc-arch}

The architecture introduces the following constants
\begin{align*}
\begin{array}{llll}
 \mathsf{L} & \text{number of blocks (depth)}, & \dm & \text{model width (token space)},\\
 H & \text{number of heads}, & \dhead=\dm/H & \text{head width},\\
 V & \text{vocabulary size}, & \dff=4\dm & \text{width of the perceptron},\\
 T & \text{context length}, & \varepsilon>0 & \text{normalisation floor},\\
 \tau>0 & \text{logit softcap}, & c_{qk}>0 & \text{query--key scale}.
\end{array}
\end{align*}
\medskip
\noindent The parameter vector is the concatenation
\begin{align}\label{eq:nc-theta}
  \theta=\Bigl(E,\;\bigl\{W_Q^{\ell},W_K^{\ell},W_V^{\ell},W_O^{\ell},
  W_1^{\ell},W_2^{\ell}\bigr\}_{\ell=1}^{\mathsf{L}},\;W_{\mathrm{lm}}\Bigr)
  \in\mathbb{R}^{\dpar},
\end{align}
with
\begin{gather*}
  E\in\mathbb{R}^{V\times \dm},\quad
  W_Q^\ell,W_K^\ell,W_V^\ell,W_O^\ell\in\mathbb{R}^{\dm\times \dm},\quad
  W_1^\ell\in\mathbb{R}^{\dff\times \dm},\\
  W_2^\ell\in\mathbb{R}^{\dm\times \dff},\quad
  W_{\mathrm{lm}}\in\mathbb{R}^{V\times \dm},
\end{gather*}
so that the parameter space has dimension
\begin{align}\label{eq:nc-dim}
  \dpar=2V\dm+12\mathsf{L}\dm^{2},\qquad\text{and}\qquad
  |\theta|^{2}=\Fn{E}^{2}+\sum_{\ell=1}^{\mathsf{L}}\bigl(\Fn{W_Q^\ell}^{2}+\cdots
  +\Fn{W_2^\ell}^{2}\bigr)+\Fn{W_{\mathrm{lm}}}^{2}.
\end{align}
We write $W_{Q,i}^\ell\in\mathbb{R}^{\dhead\times\dm}$ for the $i$-th head slice
of $W_Q^\ell$, and likewise for $K$ and $V$. Notice that \emph{nanochat} doesn't use grouped query attention, in which
$W_{K}^\ell,W_{V}^\ell\in\mathbb{R}^{(H_{\mathrm{kv}}\dhead)\times\dm}$ with
$H_{\mathrm{kv}}\mid H$, but since the analysis below uses nothing about the
weight blocks beyond
\begin{align}\label{eq:nc-opF}
  \opn{W}\le\Fn{W}\le|\theta|,
\end{align}
the proof's argument go through for more general LLMs.

\subsection{Additional Notation}
For $j\in\{1,\dots,n\}$ we
write $e_j\in\mathbb{R}^{n}$ for the $j$-th standard basis vector and
$\mathbf{1}\in\mathbb{R}^{n}$ for the vector all of whose entries equal $1$. Therefore $E^{\!\top}e_{x_t}$ as in
\eqref{eq:nc-emb} is understood as the $x_t$-th row of $E$, and
$\langle e_{y_t},\xi\rangle=\xi_{y_t}$. Let
$F:\mathbb{R}^{\dpar}\to\mathbb{R}^{N}$ and let $v,w\in\mathbb{R}^{\dpar}$
directions. For $N=1$ we write $\nabla_\theta F(\theta)\in\mathbb{R}^{\dpar}$
for the gradient and
$\nabla^{2}_{\theta}F(\theta)\in\mathbb{R}^{\dpar\times\dpar}$ for the Hessian
matrix, and denote their action on those directions by
\begin{align}\label{eq:nc-act}
  \nabla_\theta F[v]:=v^{\!\top}\nabla_\theta F,
  \qquad
  \nabla^{2}_{\theta}F[v,w]:=v^{\!\top}\bigl(\nabla^{2}_{\theta}F\bigr)w ,
\end{align}
that is the directional derivative of $F$ along $v$ and the bilinear form of the
Hessian. For $N>1$ we abuse the notation in the following sense,
\begin{align}\label{eq:nc-actN}
  \nabla_\theta F[v]:=\bigl(\nabla_\theta F_i[v]\bigr)_{i=1}^{N}
  ,
  \qquad
  \nabla^{2}_{\theta}F[v,w]:=\bigl(\nabla^{2}_{\theta}F_i[v,w]\bigr)_{i=1}^{N},
\end{align}
both elements of $\mathbb{R}^{N}$. Essentially $\nabla_\theta F$ denotes the
$N\times\dpar$ Jacobian, whereas $\nabla^{2}_{\theta}F$ is a bilinear map into
$\mathbb{R}^{N}$. The associated operator norms are
\begin{align}\label{eq:nc-normorms}
  \opn{\nabla_\theta F}:=\sup_{|v|=1}\bigl|\nabla_\theta F[v]\bigr|,
  \qquad
  \opn{\nabla^2_\theta F}:=\sup_{|v|=|w|=1}\bigl|\nabla^2_\theta F[v,w]\bigr|,
\end{align}
so that for $N=1$ they reduce to the Euclidean norm of the gradient and the
spectral norm of the Hessian, while for $N>1$ the second is the corresponding
tensor norm. Derivatives in a variable other than $\theta$ always carry that variable
as a subscript.

\subsection{Nonlinearities}\label{sub:nc-prim}

We record exact bounds for the nonlinearities of the model,
including their first two derivatives. Within this subsection $k$ denotes a
fixed width, which appears either as $k=\dm$ or
$k=\dhead$.

\begin{definition}[Normalisation]\label{def:nc-N}
For $x\in\mathbb{R}^{k}$ and $\varepsilon>0$,
\begin{align}\label{eq:nc-norm}
  \Nrm_k(x):=\frac{x}{s_k(x)},\qquad
  s_k(x):=\Bigl(\tfrac{1}{k}|x|^{2}+\varepsilon\Bigr)^{1/2}.
\end{align}
\end{definition}

\begin{definition}[Softmax, softcap, ReLU-squared, sigmoid]\label{def:nc-prim}
For $z\in\mathbb{R}^{n}$ let $\smx(z)_i:=e^{z_i}/\sum_{j}e^{z_j}$, and for
$z\in\mathbb{R}$
\begin{align*}
  \chi(z):=\tau\tanh(z/\tau),\qquad
  \varphi(z):=\max(z,0)^{2},\qquad
  \varsigma(z):=(1+e^{-z})^{-1},
\end{align*}
each extended to vectors componentwise.
\end{definition}

\noindent Since $\chi,\varphi$ and $\varsigma$ act componentwise, their
differentials are diagonal, i.e., $\nabla\chi(z)=\mathrm{diag}(\chi'(z_j))$ and
$\nabla^2\chi(z)$ is the diagonal tensor with entries $\chi''(z_j)$. Hence
$\opn{\nabla\chi(z)}\le\sup_{t\in\mathbb{R}}|\chi'(t)|$ and
$\opn{\nabla^2\chi(z)}\le\sup_{t\in\mathbb{R}}|\chi''(t)|$, and likewise for
$\varphi$ and $\varsigma$. It therefore suffices to bound the one dimensional
analogue.

\begin{lemma}\label{lem:nc-prim}
Write $s=s_k(x)$ as in \eqref{eq:nc-norm} and $\bar\sigma=\smx(z)$.
\begin{itemize}\itemsep2pt
\item[\rm(i)] $\displaystyle |\Nrm_k(x)|\le\sqrt{k}$,\qquad
 $\displaystyle \opn{\nabla \Nrm_k(x)}\le\frac1s\le\varepsilon^{-1/2}$,\qquad
 $\displaystyle \opn{\nabla^{2}\Nrm_k(x)}\le\frac{6}{\sqrt{k}\,s^{2}}
 \le\frac{6}{\sqrt{k}\,\varepsilon}$.
\item[\rm(ii)] $\nabla\smx(z)=\mathrm{diag}(\bar\sigma)-\bar\sigma\bar\sigma^{\!\top}\succeq0$ and
 \begin{align*}
  \opn{\nabla\smx(z)}\le\tfrac12,\quad
  \opn{\nabla^{2}\smx(z)}\le1,\qquad
  \opn{\nabla\smx(z)}_{\infty\to1}\le1,\quad
  \opn{\nabla^{2}\smx(z)}_{\infty\to1}\le2 .
 \end{align*}
\item[\rm(iii)] $|\chi(z)|\le\tau$,\qquad $|\chi'(z)|\le1$,\qquad
 $\displaystyle |\chi''(z)|\le\frac{4}{3\sqrt3\,\tau}=\frac{4\sqrt3}{9\tau}$.
\item[\rm(iv)] $|\varphi(z)|\le|z|^{2}$,\qquad
 $\opn{\nabla\varphi(z)}\le2|z|$,\qquad $\opn{\nabla^{2}\varphi}\le2$ and
 for $\varphi_1(z)=\max(z,0)$ one has $|\varphi_1(z)|\le|z|$,
 $\opn{\nabla\varphi_1}\le1$ and $\nabla^{2}\varphi_1=0$ a.e.
\item[\rm(v)] $|\varsigma|\le1$,\qquad $|\varsigma'|\le\tfrac14$,\qquad
 $\displaystyle|\varsigma''|\le\frac{1}{6\sqrt3}$.
\end{itemize}
Here the mixed norm is defined as$\opn{M}_{\infty\to1}:=\sup_{|v|_\infty\le1}|Mv|_1$.
\end{lemma}

\begin{proof}
(i) Since $ks^{2}=|x|^{2}+k\varepsilon\ge|x|^{2}$, one obtains
\begin{align}\label{eq:nc-xs}
  |x|\le\sqrt{k}\,s ,
\end{align}
whence $|\Nrm_k(x)|=|x|/s\le\sqrt k$. Differentiating \eqref{eq:nc-norm},
\begin{align}\label{eq:nc-DN}
  \nabla \Nrm_k(x)=\frac1s\Bigl(\mathrm{I}-\frac{xx^{\!\top}}{k s^{2}}\Bigr).
\end{align}
The matrix $P:=xx^{\!\top}/(ks^{2})$ is symmetric positive semi-definite with
$\opn{P}=|x|^{2}/(ks^{2})\le1$ by \eqref{eq:nc-xs}, so $\mathrm{I}-P$ has
spectrum in $[0,1]$ and $\opn{\nabla \Nrm_k}\le1/s\le\varepsilon^{-1/2}$.
Differentiating \eqref{eq:nc-DN} once more,
\begin{align*}
  \partial_l\partial_j (\Nrm_k)_i
 =-\frac{\delta_{ij}x_l+\delta_{il}x_j+\delta_{jl}x_i}{k s^{3}}
 +\frac{3\,x_ix_jx_l}{k^{2}s^{5}} ,
\end{align*}
so that, for unit vectors $v,w$,
\begin{align*}
  \nabla^{2}\Nrm_k(x)[v,w]
 =-\frac{v\,\langle x,w\rangle+w\,\langle x,v\rangle+x\,\langle v,w\rangle}{ks^{3}}
 +\frac{3\,x\,\langle x,v\rangle\langle x,w\rangle}{k^{2}s^{5}} .
\end{align*}
In view of the Cauchy--Schwarz inequality and \eqref{eq:nc-xs}, one calculates
\begin{align*}
  \bigl|\nabla^{2}\Nrm_k(x)[v,w]\bigr|
  \le\frac{3|x|}{ks^{3}}+\frac{3|x|^{3}}{k^{2}s^{5}}
  \le\frac{3\sqrt k}{ks^{2}}+\frac{3k\sqrt k}{k^{2}s^{2}}
  =\frac{6}{\sqrt k\,s^{2}}\le\frac{6}{\sqrt{k}\,\varepsilon}.
\end{align*}

(ii) From $\partial_j\bar\sigma_i=\bar\sigma_i(\delta_{ij}-\bar\sigma_j)$ one writes
$\nabla\smx(z)=\mathrm{diag}(\bar\sigma)-\bar\sigma\bar\sigma^{\!\top}$, and therefore, for any
$v\in\mathbb{R}^{n}$,
\begin{align}\label{eq:nc-var}
  v^{\!\top}\nabla\smx(z)\,v=\sum_i\bar\sigma_iv_i^{2}-\Bigl(\sum_i\bar\sigma_iv_i\Bigr)^{2}
  =\sum_i\bar\sigma_i\tilde v_i^{2},\qquad
  \tilde v:=v-\langle\bar\sigma,v\rangle\mathbf{1}.
\end{align}
Applying the Cauchy--Schwarz inequality to $(\sqrt{\bar\sigma_i})_i$ and
$(\sqrt{\bar\sigma_i}v_i)_i$, and using $\sum_i\bar\sigma_i=1$, gives
$(\sum_i\bar\sigma_iv_i)^{2}\le\sum_i\bar\sigma_iv_i^{2}$, so $v^{\!\top}\nabla\smx(z)\,v\ge0$ and thus
$\nabla\smx(z)\succeq0$. Notice next that $\nabla\smx(z)\mathbf{1}=\bar\sigma-\bar\sigma=0$,
so that $v^{\!\top}\nabla\smx(z)\,v=(v-c\mathbf{1})^{\!\top}\nabla\smx(z)\,(v-c\mathbf{1})$ for any $c\in\mathbb{R}$. Taking
$c=\tfrac12(\max_iv_i+\min_iv_i)$ and discarding the (nonnegative) second term
in \eqref{eq:nc-var},
\begin{align}\label{eq:nc-pop}
  v^{\!\top}\nabla\smx(z)\,v\le\sum_i\bar\sigma_i(v_i-c)^{2}\le\max_i(v_i-c)^{2}
  =\Bigl(\frac{\max_iv_i-\min_iv_i}{2}\Bigr)^{2}.
\end{align}
Since $(\max_iv_i-\min_iv_i)^{2}\le2\bigl(\max_iv_i^{2}+\min_iv_i^{2}\bigr)
\le2|v|^{2}$, \eqref{eq:nc-pop} yields $\Var_{\bar\sigma}(v)\le\tfrac12|v|^{2}$, that is
$\opn{\nabla\smx(z)}\le\tfrac12$; while for $|v|_\infty\le1$ the range is at most
$2$ and \eqref{eq:nc-pop} yields $v^{\!\top}\nabla\smx(z)\,v\le1$. Furthermore, by
Cauchy--Schwarz,
\begin{align*}
  \bigl|\nabla\smx(z)v\bigr|_{1}=\sum_i\bar\sigma_i|\tilde v_i|
  \le\Bigl(\sum_i\bar\sigma_i\tilde v_i^{2}\Bigr)^{1/2}=\left(v^{\!\top}\nabla\smx(z)\,v\right)^{1/2}\le1
  \qquad\text{when }|v|_\infty\le1,
\end{align*}
which is the third bound. For the second differential, differentiating
$\partial_j\bar\sigma_i$ once more gives
\begin{align*}
  \partial_l\partial_j\bar\sigma_i
  =\bar\sigma_i(\delta_{il}-\bar\sigma_l)(\delta_{ij}-\bar\sigma_j)-\bar\sigma_i\bar\sigma_j(\delta_{jl}-\bar\sigma_l),
\end{align*}
that is, in the notation of \eqref{eq:nc-var},
\begin{align}\label{eq:nc-D2sm}
  \nabla^{2}\smx(z)[v,w]_i=\bar\sigma_i\bigl(\tilde v_i\tilde w_i-\sum_i\bar\sigma_i\tilde v_i\tilde w_i\bigr).
\end{align}
Hence one calculates,
\begin{align}
  \bigl|\nabla^{2}\smx(z)[v,w]\bigr|
  &\le\bigl|\nabla^{2}\smx(z)[v,w]\bigr|_{1}
  \le2\sum_i\bar\sigma_i|\tilde v_i\tilde w_i|\nonumber\\
  &\le2\,\left(v^{\!\top}\nabla\smx(z)\,v\right)^{1/2}\left(w^{\!\top}\nabla\smx(z)\,w\right)^{1/2}.\label{eq:nc-D2sm-bd}
\end{align}
Plugging in $v^{\!\top}\nabla\smx(z)\,v\le\tfrac12|u|^{2}$ gives $\opn{\nabla^{2}\smx}\le1$, and
inserting $v^{\!\top}\nabla\smx(z)\,v\leq1$ for $|v|_\infty,|w|_\infty\le1$ gives
$\opn{\nabla^{2}\smx}_{\infty\to1}\le2$.

(iii) One calculates $\chi'(t)=\operatorname{sech}^{2}(t/\tau)\in(0,1]$ and
$\chi''(t)=-\tfrac2\tau\operatorname{sech}^{2}(t/\tau)\tanh(t/\tau)$. Writing
$y=\tanh(t/\tau)\in(-1,1)$ one obtains
$|\chi''(t)|=\tfrac2\tau|y(1-y^{2})|$, which attains its
maximum at $y=1/\sqrt3$ with value $2/(3\sqrt3)$.

(iv) $|\varphi(z)|^{2}=\sum_j\max(z_j,0)^{4}\le(\sum_jz_j^{2})^{2}=|z|^{4}$,
$\nabla\varphi(z)=\mathrm{diag}\bigl(2\max(z_j,0)\bigr)$ has norm at most
$2\max_j|z_j|\le2|z|$ and $\nabla^{2}\varphi(z)$ is diagonal with entries
$2\cdot\mathbf{1}_{\{z_j>0\}}$. For $\varphi_1$, one has
$\nabla\varphi_1(z)=\mathrm{diag}(\mathbf{1}_{\{z_j>0\}})$, whence
$|\nabla\varphi_1(z)v|^{2}=\sum_j\mathbf{1}_{\{z_j>0\}}v_j^{2}\le|v|^{2}$, and
$\varphi_1''=0$ a.e.

(v) $\varsigma'=\varsigma(1-\varsigma)\le\tfrac14$ and
$\varsigma''=\varsigma(1-\varsigma)(1-2\varsigma)$. Rewriting it as
$q=\varsigma-\tfrac12$ one gets $|\varsigma''|=2|q|(\tfrac14-q^{2})$, maximal at
$q=1/(2\sqrt3)$ with value $1/(6\sqrt3)$.
\end{proof}

\subsection{The Forward Map and the Potential}\label{sub:nc-forward}

Fix a token sequence $x=(x_0,\dots,x_{T-1})\in\{1,\dots,V\}^{T}$ with targets
$y_t=x_{t+1}$, let $R_t\in\mathbb{R}^{\dhead\times\dhead}$ denote the
parameter-free orthogonal rotation applied at position $t$, and let
$\mathcal{S}_\ell(t)\subseteq\{0,\dots,t\}$ be the attention window of layer
$\ell$ at position $t$. For every $t$ define the following.

\medskip\noindent\emph{Embedding.}
\begin{align}\label{eq:nc-emb}
  f^{0}_{t}\;=\;E^{\!\top}e_{x_t}\;\in\;\mathbb{R}^{\dm}
  \qquad(\text{the }x_t\text{-th row of }E).
\end{align}
For $\ell\in\{1,\dots,\mathsf{L}\}$.\\
\medskip\noindent\emph{Block $\ell$: Attention.}
\begin{align}
 \bar g^{\ell}_{t}&=\Nrm_{\dm}\bigl(f^{\ell-1}_{t}\bigr)\in\mathbb{R}^{\dm}
   &&\text{(normalise)} \label{eq:nc-E1}\\[2pt]
 q^{\ell}_{t,i}&=R_t\,\Nrm_{\dhead}\bigl(W^{\ell}_{Q,i}\,\bar g^{\ell}_{t}\bigr),
 \qquad
 k^{\ell}_{s,i}=R_s\,\Nrm_{\dhead}\bigl(W^{\ell}_{K,i}\,\bar g^{\ell}_{s}\bigr)
   \in\mathbb{R}^{\dhead} &&\text{(queries and keys)}\label{eq:nc-E2}\\[2pt]
 a^{\ell}_{t,s,i}&=\frac{c_{qk}}{\sqrt{\dhead}}
   \bigl\langle q^{\ell}_{t,i},\,k^{\ell}_{s,i}\bigr\rangle,
   \qquad s\in\mathcal{S}_\ell(t) &&\text{(scores)}\label{eq:nc-E3}\\[2pt]
 p^{\ell}_{t,\cdot,i}&=\smx\bigl(a^{\ell}_{t,\cdot,i}\bigr)
   \in\mathbb{R}^{|\mathcal{S}_\ell(t)|} &&\text{(attention weights)}\label{eq:nc-E4}\\[2pt]
 v^{\ell}_{s,i}&=W^{\ell}_{V,i}\,\bar g^{\ell}_{s}\in\mathbb{R}^{\dhead}
   &&\text{(values)}\label{eq:nc-E5}\\[2pt]
 o^{\ell}_{t}&=W^{\ell}_{O}\Bigl(\bigoplus_{i=1}^{H}
   \sum_{s\in\mathcal{S}_\ell(t)}p^{\ell}_{t,s,i}\,v^{\ell}_{s,i}\Bigr)
   \in\mathbb{R}^{\dm} &&\text{(output projection)}\label{eq:nc-E6}\\[2pt]
 f^{\ell-\frac12}_{t}&=f^{\ell-1}_{t}+o^{\ell}_{t} &&\text{(update)}\label{eq:nc-E7}
\end{align}

\medskip\noindent\emph{Block $\ell$: MLP.}
\begin{align}
 \bar b^{\ell}_{t}&=\Nrm_{\dm}\bigl(f^{\ell-\frac12}_{t}\bigr) &&\text{(normalise)}\label{eq:nc-E8}\\[2pt]
 m^{\ell}_{t}&=W^{\ell}_{2}\,\varphi\bigl(W^{\ell}_{1}\,\bar b^{\ell}_{t}\bigr)
   \in\mathbb{R}^{\dm} &&\text{(perceptron)}\label{eq:nc-E9}\\[2pt]
 f^{\ell}_{t}&=f^{\ell-\frac12}_{t}+m^{\ell}_{t} &&\text{(update)}\label{eq:nc-E10}
\end{align}

\medskip\noindent\emph{Head.}
\begin{align}
 z_t&=\Nrm_{\dm}\bigl(f^{\mathsf{L}}_{t}\bigr)\in\mathbb{R}^{\dm}, \label{eq:nc-E11}\\
 \zeta_t&=W_{\mathrm{lm}}\,z_t\in\mathbb{R}^{V} &&\text{(unembedding)}\label{eq:nc-E12}\\
 \hat\zeta_t&=\chi(\zeta_t)\in\mathbb{R}^{V} &&\text{(softcap)}\label{eq:nc-E13}
\end{align}

\begin{definition}(Cross-Entropy)\label{def:nc-u}
Let $\mathcal{D}$ denote the distribution of token sequences. The pretraining
potential is
\begin{align}\label{eq:nc-u}
  u(\theta)\;=\;\mathbb{E}_{x\sim\mathcal{D}}
  \Bigl[\frac1T\sum_{t=0}^{T-1}\CE_t(\theta)\Bigr],
  \qquad
  \CE_t(\theta)=\log\sum_{j=1}^{V}e^{\hat\zeta_{t,j}(\theta)}
  -\hat\zeta_{t,y_t}(\theta),
\end{align}
with $\hat\zeta_t$ given by \eqref{eq:nc-emb}--\eqref{eq:nc-E13}.
\end{definition}

\noindent Writing $\CE:\mathbb{R}^{V}\to\mathbb{R}$,
$\CE(\xi)=\log\sum_je^{\xi_j}-\langle e_{y_t},\xi\rangle$, the target $y_t$ is a
fixed index and \eqref{eq:nc-u} reads $u=\mathbb{E}[\tfrac1T\sum_t
\CE\circ\hat\zeta_t]$, a composition of a fixed convex function on
$\mathbb{R}^{V}$ with the $\theta$-dependent map $\hat\zeta_t$.

\begin{remark}\label{rem:nc-data}
Since $\mathcal{D}$ is supported on the finite set $\{1,\dots,V\}^{T}$ and every
derived bound is uniform in the token sequence $x$, a lower bound of the form
$\nabla^{2}_{\theta}u_x\succeq-c\,\mathrm{I}$ would hold for any data sequence.
We therefore fix $x$ and suppress it.
\end{remark}

\subsection{Curvature}\label{sub:nc-hess}

Fix $t$ and let $\xi\in\mathbb{R}^{V}$, so that
$\CE_t(\xi)=\log\sum_{j=1}^{V}e^{\xi_j}-\xi_{y_t}$. Differentiating in $\xi_i$,
\begin{align}\label{eq:nc-ce1}
  \frac{\partial\CE_t}{\partial\xi_i}
  =\frac{e^{\xi_i}}{\sum_{j}e^{\xi_j}}-\delta_{i,y_t}
  =\smx(\xi)_i-\delta_{i,y_t},
\end{align}
and,
\begin{align}\label{eq:nc-ce2}
  \frac{\partial^{2}\CE_t}{\partial\xi_i\,\partial\xi_j}
  =\frac{\partial\,\smx(\xi)_i}{\partial\xi_j}
  =\smx(\xi)_i\bigl(\delta_{ij}-\smx(\xi)_j\bigr).
\end{align}
Evaluating \eqref{eq:nc-ce1} and \eqref{eq:nc-ce2} at $\xi=\hat\zeta_t$ and
collecting the entries into a vector and a matrix respectively, we set
\begin{align}\label{eq:nc-defs}
  \bar\sigma_t:=\smx(\hat\zeta_t),\qquad
  g_t:=\nabla_{\hat\zeta_t}\CE_t=\bar\sigma_t-e_{y_t},\qquad
  \Lambda_t:=\nabla^{2}_{\hat\zeta_t}\CE_t
  =\mathrm{diag}(\bar\sigma_t)-\bar\sigma_t\bar\sigma_t^{\!\top}.
\end{align}
By Lemma~\ref{lem:nc-prim}(ii) the matrix $\Lambda_t$ is positive
semi-definite, that is $\CE_t$ is convex in $\xi$. 

\medskip
\noindent Applying the chain rule to $u=\tfrac1T\sum_t\CE\circ\hat\zeta_t$ once
gives the gradient of the potential,
\begin{align}\label{eq:nc-grad}
  \nabla_\theta u(\theta)
  =\frac1T\sum_{t=0}^{T-1}\bigl(\nabla_\theta\hat\zeta_t\bigr)^{\!\top}
   \bigl(\smx(\hat\zeta_t)-e_{y_t}\bigr)
  =\frac1T\sum_{t=0}^{T-1}\bigl(\nabla_\theta\zeta_t\bigr)^{\!\top}
   \Bigl[\bigl(\smx(\chi(\zeta_t))-e_{y_t}\bigr)\odot\chi'(\zeta_t)\Bigr],
\end{align}
where $\odot$ denotes the Hadamard product and the second equality follows from
$\nabla_\theta\hat\zeta_t
=\mathrm{diag}\bigl(\chi'(\zeta_t)\bigr)\nabla_\theta\zeta_t$. In the notation of
\eqref{eq:nc-act}, $\nabla_\theta\hat\zeta_t$ is the $V\times\dpar$ Jacobian and
$\nabla_\theta u(\theta)$ a vector of $\mathbb{R}^{\dpar}$. Observe that
$\smx(\hat\zeta_t)-e_{y_t}$ is the predicted distribution less the observed one. Differentiating once more yields the exact identity
\begin{align}\label{eq:nc-hess}
  \nabla^{2}_{\theta}u(\theta)=\frac1T\sum_{t=0}^{T-1}\Bigl[
  \underbrace{(\nabla_\theta\hat\zeta_t)^{\!\top}\Lambda_t\,
  (\nabla_\theta\hat\zeta_t)}_{\text{Gauss--Newton term},\ \succeq\,0}
  \;+\;
  \underbrace{\sum_{j=1}^{V}(g_t)_j\;\nabla^{2}_{\theta}\hat\zeta_{t,j}}
  _{\text{residual term}}\Bigr].
\end{align}
Here $\nabla_\theta\hat\zeta_t$ is the $V\times\dpar$ Jacobian and
$\nabla^{2}_{\theta}\hat\zeta_{t,j}$ the $\dpar\times\dpar$ Hessian of the
scalar map $\theta\mapsto\hat\zeta_{t,j}(\theta)$. One differentiates in
$\theta$, while $g_t$ and $\Lambda_t$ are derivatives of $\CE$ in its own
argument, evaluated at $\hat\zeta_t(\theta)$. Since $A^{\!\top}MA$ inherits the sign of $M$, the first term of \eqref{eq:nc-hess}
is positive semi-definite and may be discarded in the analysis.

\begin{proposition}\label{prop:nc-reduce}
For every $\theta\notin\mathcal{N}$, where $\mathcal{N}$ is the Lebesgue-null
set of Remark~\ref{rem:nc-ae},
\begin{align}\label{eq:nc-reduce}
  \nabla^{2}_{\theta}u(\theta)\;\succeq\;
  -\Bigl[\frac{4}{3\sqrt3\,\tau}\max_{t}\opn{\nabla_\theta\zeta_t}^{2}
  \;+\;2\max_{t}\opn{\nabla^{2}_{\theta}\zeta_t}\Bigr]\Idm .
\end{align}
\end{proposition}

\begin{proof}
Discarding the Gauss--Newton term of \eqref{eq:nc-hess} and writing
\begin{align*}
  P:=\frac1T\sum_{t=0}^{T-1}\sum_{j=1}^{V}(g_t)_j\,\nabla^{2}_{\theta}\hat\zeta_{t,j},
\end{align*}
one has $\nabla^{2}_{\theta}u\succeq P\succeq-\opn{P}\Idm$. Observe first that
\begin{align}\label{eq:nc-gt}
  |g_t|_{1}=\bigl(1-\bar\sigma_{t,y_t}\bigr)+\sum_{j\neq y_t}\bar\sigma_{t,j}
  =2\bigl(1-\bar\sigma_{t,y_t}\bigr)\le2,
  \qquad |g_t|_\infty\le1 .
\end{align}
Since $\chi$ acts componentwise, the chain rule gives
\begin{align*}
  \nabla^{2}_{\theta}\hat\zeta_{t,j}
  =\chi'(\zeta_{t,j})\,\nabla^{2}_{\theta}\zeta_{t,j}
  +\chi''(\zeta_{t,j})\,\nabla_\theta\zeta_{t,j}\otimes\nabla_\theta\zeta_{t,j},
\end{align*}
and we bound the two resulting pieces of $P$, the first gives, by $|\chi'|\le1$ and \eqref{eq:nc-gt},
\begin{align*}
  \Bigl\|\sum_{j}(g_t)_j\,\chi'(\zeta_{t,j})\,\nabla^{2}_{\theta}\zeta_{t,j}\Bigr\|
  \;\le\;|g_t|_{1}\max_j\opn{\nabla^{2}_{\theta}\zeta_{t,j}}
  \;\le\;2\opn{\nabla^{2}_{\theta}\zeta_t},
\end{align*}
while the second, being componentwise in $\chi$, reassembles into a quadratic
form,
\begin{align*}
  \sum_{j}(g_t)_j\,\chi''(\zeta_{t,j})\,
  \nabla_\theta\zeta_{t,j}\otimes\nabla_\theta\zeta_{t,j}
  =(\nabla_\theta\zeta_t)^{\!\top}
  \mathrm{diag}\bigl(g_t\odot\chi''(\zeta_t)\bigr)(\nabla_\theta\zeta_t),
\end{align*}
whose norm is at most
$|g_t|_\infty\opn{\chi''}_\infty\opn{\nabla_\theta\zeta_t}^{2}
\le\tfrac{4}{3\sqrt3\tau}\opn{\nabla_\theta\zeta_t}^{2}$ by
Lemma~\ref{lem:nc-prim}(iii) and \eqref{eq:nc-gt}. Averaging over $t$ and using
$\tfrac1T\sum_t\le\max_t$ yields \eqref{eq:nc-reduce}.
\end{proof}

\subsection{Tracking the Growth}\label{sub:nc-calc}

We now introduce the device by which we perform the induction through the LLM's layers.

\begin{definition}\label{def:nc-T}
Let $F:\mathbb{R}^{\dpar}\to\mathbb{R}^{N}$ be twice differentiable a.e. We
write $F\in\Tcl(a_0,a_1,a_2;\,C_0,C_1,C_2)$ if, for all $\theta$,
\begin{align}\label{eq:nc-Tdef}
  |F(\theta)|\le C_0\bigl(1+|\theta|\bigr)^{a_0},\quad
  \opn{\nabla_\theta F(\theta)}\le C_1\bigl(1+|\theta|\bigr)^{a_1},\quad
  \opn{\nabla^{2}_{\theta} F(\theta)}\le C_2\bigl(1+|\theta|\bigr)^{a_2}.
\end{align}
When only the exponents are regarded, we abbreviate $F\in\Tcl(a_0,a_1,a_2)$.
\end{definition}

\begin{lemma}\label{lem:nc-rules}
Let $F\in\Tcl(a_0,a_1,a_2;C_0,C_1,C_2)$ and $G\in\Tcl(b_0,b_1,b_2;D_0,D_1,D_2)$.
\begin{itemize}\itemsep3pt
\item[\rm(i)] \emph{(Addition.)}
 $F+G\in\Tcl\bigl(\max\{a_0,b_0\},\max\{a_1,b_1\},\max\{a_2,b_2\};\,
 C_0+D_0,\,C_1+D_1,\,C_2+D_2\bigr)$.
\item[\rm(ii)] \emph{(Bilinear maps.)} If $\mathcal{B}$ is bilinear with
 $|\mathcal{B}(\xi,\eta)|\le\beta|\xi||\eta|$, then
 $\mathcal{B}(F,G)\in\Tcl(c_0,c_1,c_2;\allowbreak K_0,K_1,K_2)$ with
 \begin{align*}
 \begin{aligned}
 c_0&=a_0+b_0, & K_0&=\beta C_0D_0,\\
 c_1&=\max\{a_0+b_1,\;a_1+b_0\}, & K_1&=\beta( C_0D_1+C_1D_0),\\
 c_2&=\max\{a_0+b_2,\;a_1+b_1,\;a_2+b_0\},
   & K_2&=\beta( C_0D_2+2C_1D_1+C_2D_0).
 \end{aligned}
 \end{align*}
\item[\rm(iii)] \emph{(Composition with a bounded function.)} If
 $g:\mathbb{R}^{N}\to\mathbb{R}^{M}$ satisfies $|g|\le g_0$,
 $\opn{\nabla g}\le g_1$ and $\opn{\nabla^{2}g}\le g_2$ globally, then
 \[
 g\circ F\;\in\;\Tcl\bigl(0,\;a_1,\;\max\{2a_1,a_2\};\;
 g_0,\;g_1C_1,\;g_1C_2+g_2C_1^{2}\bigr).
 \]
\item[\rm(iv)] \emph{(Composition with the activation function.)}
 \[
 \varphi\circ F\in\Tcl\bigl(2a_0,\;a_0+a_1,\;\max\{2a_1,\,a_0+a_2\};\;
 C_0^{2},\;2C_0C_1,\;2C_1^{2}+2C_0C_2\bigr),
 \]
 whereas $\varphi_1\circ F\in\Tcl(a_0,a_1,a_2;\,C_0,C_1,C_2)$.
\item[\rm(v)] \emph{(Coordinate blocks.)} If $F(\theta)=W(\theta)$ is a
 coordinate block of $\theta$, then $F\in\Tcl(1,0,0;\,1,1,0)$.
\item[\rm(vi)] \emph{(Weighted sum by softmax.)} Let $\mathcal{S}$ be a finite
 index set and let the families $\alpha=(\alpha_s)_{s\in\mathcal{S}}$ and
 $w=(w_s)_{s\in\mathcal{S}}$ satisfy the bounds \eqref{eq:nc-Tdef} entrywise,
 that is with $\max_{s}$ in place of the norm, with exponents and constants
 $(a_{\cdot};C_{\cdot})$ and $(b_{\cdot};D_{\cdot})$ respectively. Put
 $p=\smx(\alpha)$ and $y=\sum_{s\in\mathcal{S}}p_sw_s$. Then
 $y\in\Tcl(c_0,c_1,c_2;K_0,K_1,K_2)$ with
 \begin{align*}
 \begin{aligned}
 c_0&=b_0, & K_0&=D_0,\\
 c_1&=\max\{a_1+b_0,\;b_1\}, & K_1&=C_1D_0+D_1,\\
 c_2&=\max\{a_2+b_0,\;2a_1+b_0,\;a_1+b_1,\;b_2\},
   & K_2&=(C_2+2C_1^{2})D_0+2C_1D_1+D_2.
 \end{aligned}
 \end{align*}
\end{itemize}
\end{lemma}

\begin{proof}
(i) One trivially has
$(1+|\theta|)^{a_0}\le(1+|\theta|)^{\max\{a_0,b_0\}}$ and likewise for $b_0$,
whence
$|F+G|\le|F|+|G|\le(C_0+D_0)(1+|\theta|)^{\max\{a_0,b_0\}}$. The same argument
applies to $\nabla_\theta$ and $\nabla^{2}_{\theta}$.

(ii) By the Leibniz rule,
$\nabla_\theta\mathcal{B}(F,G)=\mathcal{B}(\nabla_\theta F,G)
+\mathcal{B}(F,\nabla_\theta G)$ and
$\nabla^{2}_{\theta}\mathcal{B}(F,G)=\mathcal{B}(\nabla^{2}_{\theta}F,G)
+2\mathcal{B}(\nabla_\theta F,\nabla_\theta G)
+\mathcal{B}(F,\nabla^{2}_{\theta}G)$. One multiplies the corresponding bounds and
takes the maximum of the exponents to obtain the claim.

(iii) The chain rule gives $|g\circ F|\le g_0=g_0(1+|\theta|)^{0}$. Also one writes $\nabla_\theta(g\circ F)=\nabla g(F)\,\nabla_\theta F$,
whence $\opn{\nabla_\theta(g\circ F)}\le g_1C_1(1+|\theta|)^{a_1}$ and
\begin{align}\label{eq:nc-chain2}
  \nabla^{2}_{\theta}(g\circ F)[v,w]
  =\nabla^{2}g(F)\bigl[\nabla_\theta F[v],\nabla_\theta F[w]\bigr]
  +\nabla g(F)\bigl[\nabla^{2}_{\theta}F[v,w]\bigr],
\end{align}
so that $\opn{\nabla^{2}_{\theta}(g\circ F)}
\le g_2C_1^{2}(1+|\theta|)^{2a_1}+g_1C_2(1+|\theta|)^{a_2}$.

(iv) One repeats the computations of
Lemma~\ref{lem:nc-prim}(iv) to get : $|\varphi(F)|\le|F|^{2}\le C_0^{2}(1+|\theta|)^{2a_0}$, and
$\opn{\nabla_\theta(\varphi\circ F)}\le2|F|\opn{\nabla_\theta F}
\le2C_0C_1(1+|\theta|)^{a_0+a_1}$. Finally, by \eqref{eq:nc-chain2},
$\opn{\nabla^{2}_{\theta}(\varphi\circ F)}\le2C_1^{2}(1+|\theta|)^{2a_1}
+2C_0C_2(1+|\theta|)^{a_0+a_2}$. For $\varphi_1$ the three bounds of
Lemma~\ref{lem:nc-prim}(iv) are $1$, $1$ and $0$, so that the first term of
\eqref{eq:nc-chain2} vanishes and the rule is transparent in all three
exponents.

(v) $|W(\theta)|=\Fn{W}\le|\theta|\le1+|\theta|$ by \eqref{eq:nc-opF} and $\nabla_\theta W$ is defined a.e. with uniformly bounded norm by $1$. Moreover, $\nabla^{2}_{\theta}W=0$.

(vi) Since $p$ lies in the simplex, $|y|\le\sum_sp_s|w_s|\le\max_s|w_s|$. Next,
$\nabla_\theta y=\sum_s(\nabla_\theta p_s)w_s+\sum_sp_s\,\nabla_\theta w_s$,
the second sum being again a convex combination and bounded by
$\max_s\opn{\nabla_\theta w_s}$. Regarding the first differential, let $v$ be unit vector such that ,
\begin{align*}
  \Bigl|\sum_s \nabla_\theta p_s[v]\,w_s\Bigr|
  &\le\max_s|w_s|\;\bigl|\nabla\smx(\alpha)\,\nabla_\theta\alpha[v]\bigr|_{1}\\
  &\le\max_s|w_s|\;\opn{\nabla\smx}_{\infty\to1}\,
  \bigl|\nabla_\theta\alpha[v]\bigr|_{\infty}
  \le\max_s|w_s|\;\max_s\opn{\nabla_\theta\alpha_s},
\end{align*}
by Lemma~\ref{lem:nc-prim}(ii). For the second differential, expand
\begin{align*}
  \nabla^{2}_{\theta}y=\sum_s(\nabla^{2}_{\theta}p_s)w_s
  +2\sum_s(\nabla_\theta p_s)\otimes(\nabla_\theta w_s)
  +\sum_sp_s\,\nabla^{2}_{\theta}w_s ,
\end{align*}
the last sum being bounded by $\max_s\opn{\nabla^{2}_{\theta}w_s}$ and the
middle one by $2\max_s\opn{\nabla_\theta w_s}\max_s\opn{\nabla_\theta\alpha_s}$
as above. For the first, \eqref{eq:nc-chain2} applied to $p=\smx\circ\alpha$
gives
\begin{align*}
  \bigl|\nabla^{2}_{\theta}p[v,v]\bigr|_{1}
  &\le\opn{\nabla^{2}\smx}_{\infty\to1}
  \bigl|\nabla_\theta\alpha[v]\bigr|_{\infty}^{2}
  +\opn{\nabla\smx}_{\infty\to1}
  \bigl|\nabla^{2}_{\theta}\alpha[v,v]\bigr|_{\infty}\\
  &\le2\max_s\opn{\nabla_\theta\alpha_s}^{2}
  +\max_s\opn{\nabla^{2}_{\theta}\alpha_s},
\end{align*}
which upon multiplication by $\max_s|w_s|$ yieldsthe exponents $2a_1+b_0$
and $a_2+b_0$ and the constants $2C_1^{2}D_0$ and $C_2D_0$.
\end{proof}

\begin{corollary}\label{cor:nc-reset}
Let $F\in\Tcl(a_0,a_1,a_2;C_0,C_1,C_2)$ take values in $\mathbb{R}^{k}$. Then
\begin{align*}
  \Nrm_k\circ F\;\in\;\Tcl\Bigl(0,\;a_1,\;\max\{2a_1,a_2\};\;
  \sqrt{k},\;\varepsilon^{-1/2}C_1,\;
  \varepsilon^{-1/2}C_2+\tfrac{6}{\sqrt k\,\varepsilon}C_1^{2}\Bigr).
\end{align*}
\end{corollary}

\begin{proof}
Apply Lemma~\ref{lem:nc-rules}(iii) with $g=\Nrm_k$ and the constants
$g_0=\sqrt k$, $g_1=\varepsilon^{-1/2}$ and $g_2=6/(\sqrt k\varepsilon)$ of
Lemma~\ref{lem:nc-prim}(i), all of which are finite because $\varepsilon>0$.
\end{proof}

\subsection{The Layer Recursion}\label{sub:nc-rec}

We track, for the residual stream at token $t$,
\begin{align}\label{eq:nc-track}
  f^{\ell}_t\in\Tcl\bigl(A_\ell,\,B_\ell,\,C_\ell\bigr),
\end{align}
suppressing the constants, which are collected later in Section [X].

\begin{lemma}[Base]\label{lem:nc-base}
$f^{0}_t\in\Tcl(1,0,0)$, that is $A_0=1$, $B_0=0$ and $C_0=0$.
\end{lemma}
\begin{proof}
By \eqref{eq:nc-emb}, $f^0_t$ is a coordinate slice of $E$; apply
Lemma~\ref{lem:nc-rules}(v).
\end{proof}

\begin{lemma}[Attention sublayer]\label{lem:nc-attn}
Suppose $f^{\ell-1}_t\in\Tcl(A,B,C)$. Then
\begin{align}\label{eq:nc-o}
  o^{\ell}_t\in\Tcl\bigl(2,\;3+B,\;\max\{4+2B,\;3+C\}\bigr),
\end{align}
and consequently, by \eqref{eq:nc-E7} and Lemma~\ref{lem:nc-rules}(i),
\begin{align}\label{eq:nc-fhalf}
  f^{\ell-\frac12}_t\in\Tcl\bigl(\max\{A,2\},\;3+B,\;\max\{4+2B,\,3+C\}\bigr).
\end{align}
\end{lemma}

\begin{proof}
Write $B^{\star}:=\max\{2B,C\}$. By Corollary~\ref{cor:nc-reset} applied to
\eqref{eq:nc-E1},
\begin{align*}
  \bar g^{\ell}_t=\Nrm_{\dm}\bigl(f^{\ell-1}_t\bigr)\in\Tcl(0,\,B,\,B^{\star}).
\end{align*}

\emph{Queries and keys.} By Lemma~\ref{lem:nc-rules}(v),
$W^{\ell}_{Q,i}\in\Tcl(1,0,0)$, so that Lemma~\ref{lem:nc-rules}(ii) applied to
the matrix--vector product, for which $\beta=1$ since
$|Wx|\le\Fn{W}|x|$, gives
\begin{align*}
  W^{\ell}_{Q,i}\bar g^{\ell}_t
  \in\Tcl\bigl(1,\;\max\{1+B,\,0\},\;\max\{1+B^{\star},\,B,\,0\}\bigr)
  =\Tcl\bigl(1,\,1+B,\,1+B^{\star}\bigr),
\end{align*}
the collapse of the two maxima following from $B^{\star}\ge2B\ge B$. Composing with $\Nrm_{\dhead}$ through
Corollary~\ref{cor:nc-reset}, and with the orthogonal $R_t$, which is invariant,
\begin{align}\label{eq:nc-q}
  q^{\ell}_{t,i}\in\Tcl\bigl(0,\;1+B,\;\max\{2+2B,\,1+B^{\star}\}\bigr),
\end{align}
and identically for $k^{\ell}_{s,i}$.

\emph{Scores and weights.} Equation \eqref{eq:nc-E3} is a bilinear form with
$\beta=c_{qk}/\sqrt{\dhead}$, so that Lemma~\ref{lem:nc-rules}(ii) applied to
two factors of the class \eqref{eq:nc-q} returns that same class,
\begin{align*}
  a^{\ell}_{t,s,i}\in\Tcl\bigl(0,\;1+B,\;\alpha_2\bigr),
  \qquad \alpha_2:=\max\{2+2B,\,1+B^{\star}\}.
\end{align*}

\emph{Values.} Alike the queries,
$v^{\ell}_{s,i}=W^{\ell}_{V,i}\bar g^{\ell}_s\in\Tcl(1,\,1+B,\,1+B^{\star})$,
entrywise in $s$.

\emph{Contraction.} The weighted sum $\sum_{s}p^{\ell}_{t,s,i}v^{\ell}_{s,i}$ is
now handled by Lemma~\ref{lem:nc-rules}(vi), with
$(a_0,a_1,a_2)=(0,1+B,\alpha_2)$ and $(b_0,b_1,b_2)=(1,1+B,1+B^{\star})$:
\begin{align*}
  c_0&=1,\qquad c_1=\max\{(1+B)+1,\;1+B\}=2+B,\\
  c_2&=\max\bigl\{\alpha_2+1,\;2(1+B)+1,\;(1+B)+(1+B),\;1+B^{\star}\bigr\}
  =\max\{3+2B,\;2+B^{\star}\},
\end{align*}
where the last equality uses $\alpha_2+1=\max\{3+2B,2+B^{\star}\}$. Thus
$\sum_{s}p^{\ell}_{t,s,i}v^{\ell}_{s,i}
\in\Tcl(1,\,2+B,\,\max\{3+2B,\,2+B^{\star}\})$ for each head $i$, and the
concatenation over $i$ leaves these exponents unchanged, see
Remark~\ref{rem:nc-heads}.

\emph{Output projection.} Applying $W^{\ell}_O\in\Tcl(1,0,0)$ through
Lemma~\ref{lem:nc-rules}(ii),
\begin{align*}
  o^{\ell}_t&\in\Tcl\bigl(2,\;\max\{3+B,\,1\},\;
  \max\{1+\max\{3+2B,2+B^{\star}\},\;2+B,\;1\}\bigr)\\
  &=\Tcl\bigl(2,\;3+B,\;\max\{4+2B,\,3+B^{\star}\}\bigr).
\end{align*}
Substituting $B^{\star}=\max\{2B,C\}$ gives
$3+B^{\star}=\max\{3+2B,3+C\}$ and hence
$\max\{4+2B,3+B^{\star}\}=\max\{4+2B,3+C\}$, which is \eqref{eq:nc-o}. Finally
\eqref{eq:nc-fhalf} follows from Lemma~\ref{lem:nc-rules}(i) applied to
\eqref{eq:nc-E7}, using $B\le3+B$ and $C\le\max\{4+2B,3+C\}$.
\end{proof}

\begin{remark}[The head concatenation]\label{rem:nc-heads}
The bounds preceding \eqref{eq:nc-o} are per head, whereas \eqref{eq:nc-E6}
concatenates. Write $y_i:=\sum_sp^{\ell}_{t,s,i}v^{\ell}_{s,i}$ and using
$|y_i|\le\max_s|W^{\ell}_{V,i}\bar g^{\ell}_s|\le\Fn{W^{\ell}_{V,i}}\sqrt{\dm}$, then one obtains
\begin{align*}
  \Bigl|\bigoplus_{i=1}^{H}y_i\Bigr|^{2}=\sum_{i=1}^{H}|y_i|^{2}
  \le\dm\sum_{i=1}^{H}\Fn{W^{\ell}_{V,i}}^{2}
  =\dm\,\Fn{W^{\ell}_{V}}^{2}\le\dm\bigl(1+|\theta|\bigr)^{2},
\end{align*}
the middle equality holding because the head slices partition the rows of
$W^{\ell}_V$. The same argument applies verbatim to the first and second
differentials.
\end{remark}

\begin{lemma}[Perceptron sublayer]\label{lem:nc-mlp}
Suppose $f^{\ell-\frac12}_t\in\Tcl(A',B',C')$. Then
\begin{align}\label{eq:nc-m}
  m^{\ell}_t\in\Tcl\bigl(3,\;3+B',\;\max\{3+2B',\;3+C'\}\bigr),
\end{align}
and consequently
$f^{\ell}_t\in\Tcl\bigl(\max\{A',3\},\,3+B',\,\max\{3+2B',3+C'\}\bigr)$.
\end{lemma}

\begin{proof}
Write $B'^{\star}:=\max\{2B',C'\}$. By Corollary~\ref{cor:nc-reset} applied to
\eqref{eq:nc-E8}, $\bar b^{\ell}_t\in\Tcl(0,B',B'^{\star})$, and by
Lemma~\ref{lem:nc-rules}(v) and (ii),
\begin{align*}
  W^{\ell}_1\bar b^{\ell}_t\in\Tcl\bigl(1,\,1+B',\,1+B'^{\star}\bigr).
\end{align*}
Applying Lemma~\ref{lem:nc-rules}(iv),
\begin{align*}
  \varphi\bigl(W^{\ell}_1\bar b^{\ell}_t\bigr)
  &\in\Tcl\bigl(2,\;1+(1+B'),\;\max\{2(1+B'),\,1+(1+B'^{\star})\}\bigr)\\
  &=\Tcl\bigl(2,\;2+B',\;\max\{2+2B',\,2+B'^{\star}\}\bigr),
\end{align*}
and then $W^{\ell}_2\in\Tcl(1,0,0)$ through Lemma~\ref{lem:nc-rules}(ii),
\begin{align*}
  m^{\ell}_t&\in\Tcl\bigl(3,\;\max\{3+B',\,2\},\;
  \max\{1+\max\{2+2B',2+B'^{\star}\},\;2+B',\;2\}\bigr)\\
  &=\Tcl\bigl(3,\,3+B',\,\max\{3+2B',\,3+B'^{\star}\}\bigr).
\end{align*}
Substituting $B'^{\star}=\max\{2B',C'\}$ gives
$3+B'^{\star}=\max\{3+2B',3+C'\}$, whence \eqref{eq:nc-m}. The statement for
$f^{\ell}_t$ follows from Lemma~\ref{lem:nc-rules}(i) applied to
\eqref{eq:nc-E10}, using $B'\le3+B'$ and $C'\le\max\{3+2B',3+C'\}$.
\end{proof}

\begin{proposition}[Induction]\label{prop:nc-exp}
For every $\ell\in\{0,1,\dots,\mathsf{L}\}$,
\begin{align}\label{eq:nc-ABC}
  A_\ell=\begin{cases}1,&\ell=0\\ 3,&\ell\ge1,\end{cases}
  \qquad
  B_\ell=6\ell,
  \qquad
  C_\ell=\begin{cases}0,&\ell=0\\ 12\ell-3,&\ell\ge1.\end{cases}
\end{align}
\end{proposition}

\begin{proof}
Composing Lemmas~\ref{lem:nc-attn} and \ref{lem:nc-mlp}, that is inserting
$A'=\max\{A_{\ell-1},2\}$, $B'=3+B_{\ell-1}$ and
$C'=\max\{4+2B_{\ell-1},3+C_{\ell-1}\}$ into Lemma~\ref{lem:nc-mlp}, one obtains
the recursions
\begin{gather*}
  A_\ell=\max\{A_{\ell-1},3\},\qquad B_\ell=6+B_{\ell-1},\\
  C_\ell=\max\bigl\{9+2B_{\ell-1},\;3+\max\{4+2B_{\ell-1},\,3+C_{\ell-1}\}\bigr\},
\end{gather*}
the first because $\max\{\max\{A_{\ell-1},2\},3\}=\max\{A_{\ell-1},3\}$. With
$A_0=1$ and $=B_0=C_0=0$, from Lemma~\ref{lem:nc-base}, the first two are
immediate. For the third, substituting $B_{\ell-1}=6\ell-6$,
\begin{align*}
  C_\ell&=\max\bigl\{12\ell-3,\;3+\max\{12\ell-8,\;3+C_{\ell-1}\}\bigr\}\\
  &=\max\bigl\{12\ell-3,\;12\ell-5,\;6+C_{\ell-1}\bigr\}
  =\max\bigl\{12\ell-3,\;6+C_{\ell-1}\bigr\}.
\end{align*}
Thus $C_1=\max\{9,6\}=9=12-3$, and if $C_{\ell-1}=12(\ell-1)-3$ then
$6+C_{\ell-1}=12\ell-9<12\ell-3$, so that $C_\ell=12\ell-3$.
\end{proof}

\begin{remark}\label{rem:nc-Aconst}
The exponent $A_\ell=3$ does not depend on $\mathsf{L}$ since the
normalisations \eqref{eq:nc-E1} and \eqref{eq:nc-E8} discard the size of the
incoming stream and nothing accumulates across the depth.
\end{remark}

\begin{corollary}[The logit map]\label{cor:nc-zeta}
$f^{\mathsf{L}}_t\in\Tcl(3,\,6\mathsf{L},\,12\mathsf{L}-3)$, and consequently
\begin{align}\label{eq:nc-zeta}
  z_t\in\Tcl\bigl(0,\;6\mathsf{L},\;12\mathsf{L}\bigr),
  \qquad
  \zeta_t\in\Tcl\bigl(1,\;6\mathsf{L}+1,\;12\mathsf{L}+1\bigr).
\end{align}
\end{corollary}

\begin{proof}
By Corollary~\ref{cor:nc-reset} applied to \eqref{eq:nc-E11},
$z_t\in\Tcl(0,6\mathsf{L},\max\{12\mathsf{L},12\mathsf{L}-3\})=\Tcl(0,6\mathsf{L},12\mathsf{L})$. By
Lemma~\ref{lem:nc-rules}(v) and (ii) applied to \eqref{eq:nc-E12},
\begin{align*}
  \zeta_t\in\Tcl\bigl(1+0,\;\max\{1+6\mathsf{L},\,0\},\;\max\{1+12\mathsf{L},\;6\mathsf{L},\;0\}\bigr)
  =\Tcl(1,\,6\mathsf{L}+1,\,12\mathsf{L}+1). 
\end{align*}
\end{proof}

\subsection{The Curvature Bound}\label{sub:nc-bound}

\begin{theorem}\label{thm:nc-main}
There exists a constant $c^{\star}=c^{\star}(\dm,\dhead,H,\varepsilon,\tau,c_{qk},\mathsf{L})$,
independent of $\theta$, of $T$ and of the token sequence, such that for a.e.\
$\theta\in\mathbb{R}^{\dpar}$
\begin{align}\label{eq:nc-main}
  \boxed{\;\nabla^{2}_{\theta}u(\theta)\;\succeq\;
  -c^{\star}\bigl(1+|\theta|\bigr)^{12\mathsf{L}+2}\,\Idm\;}
\end{align}
The constant is computed in Section [X].
\end{theorem}

\begin{proof}
By Corollary~\ref{cor:nc-zeta} there are constants $\Gamma_1,\Gamma_2$ with
\begin{align*}
  \opn{\nabla_\theta\zeta_t}\le\Gamma_1\bigl(1+|\theta|\bigr)^{6\mathsf{L}+1},
  \qquad
  \opn{\nabla^{2}_{\theta}\zeta_t}\le\Gamma_2\bigl(1+|\theta|\bigr)^{12\mathsf{L}+1},
\end{align*}
uniformly in $t$. Inserting these into Proposition~\ref{prop:nc-reduce},
\begin{align*}
  \nabla^{2}_{\theta}u(\theta)\;\succeq\;
  -\Bigl[\frac{4\,\Gamma_1^{2}}{3\sqrt3\,\tau}\bigl(1+|\theta|\bigr)^{12\mathsf{L}+2}
  +2\Gamma_2\bigl(1+|\theta|\bigr)^{12\mathsf{L}+1}\Bigr]\Idm
  \;\succeq\;-c^{\star}\bigl(1+|\theta|\bigr)^{12\mathsf{L}+2}\,\Idm,
\end{align*}
with $c^{\star}=4\Gamma_1^{2}/(3\sqrt3\tau)+2\Gamma_2$, the last step using
$1+|\theta|\ge1$ to absorb the lower exponent into the higher.
\end{proof}

\begin{remark}\label{rem:nc-ae}
Every function of Lemma~\ref{lem:nc-prim} is $C^{\infty}$, with the exception of Relu-squared, $\varphi$, whose second derivative
$\varphi''=2\cdot\mathbf{1}_{\{\,\cdot\,>0\}}$ jumps at the origin. The
exceptional set is therefore the finite union
\begin{align}\label{eq:nc-normullset}
  \mathcal{N}=\bigcup_{\ell,j,t}\Bigl\{\theta:\;
  \bigl(W_1^{\ell}\,\Nrm_{\dm}(f^{\ell-\frac12}_t)\bigr)_j=0\Bigr\},
\end{align}
and each constituent is a $C^{1}$ hypersurface, hence Lebesgue-null.
\end{remark}

\begin{remark}\label{rem:nc-notis}
The softcap gives $|\hat\zeta_{t,j}|\le\tau$ for every $t,j$ and every $\theta$,
so that by \eqref{eq:nc-u}
\begin{align}\label{eq:nc-bounded}
  0\le u(\theta)\le\log V+2\tau\qquad\text{for all }\theta\in\mathbb{R}^{\dpar}.
\end{align}
Hence, $e^{-\beta u}$ is not integrable on
$\mathbb{R}^{\dpar}$, i.e., no Gibbs measure can be associated with $u$ without proper regularization.
\end{remark}

\subsection{Regularization}\label{sub:nc-reg}

Following \cite{Lovas_TUSLA_2023} and \cite{Lim_Sabanis_2024}, we add a
higher-order term to the potential. Fix $\eta>0$ and set
\begin{align}\label{eq:nc-ureg}
  \hat u(\theta):=u(\theta)+\eta\,|\theta|^{12\mathsf{L}+5},
  \qquad\theta\in\mathbb{R}^{\dpar}.
\end{align}
One calculates
\begin{align}
  \nabla^{2}_{\theta}\bigl(|\theta|^{12\mathsf{L}+5}\bigr)
  &=(12\mathsf{L}+5)|\theta|^{12\mathsf{L}+3}\,\Idm
  +(12\mathsf{L}+5)(12\mathsf{L}+3)\,|\theta|^{12\mathsf{L}+1}\,\theta\theta^{\!\top}\nonumber\\
  &\;\succeq\;(12\mathsf{L}+5)|\theta|^{12\mathsf{L}+3}\,\Idm ,\label{eq:nc-hesspen}
\end{align}
the discarded term being positive semi-definite. Throughout this subsection we
abbreviate
\begin{align}\label{eq:nc-cbar}
  \bar c^{\star}:=2^{12\mathsf{L}+1}c^{\star},\qquad\text{so that}\qquad
  c^{\star}\bigl(1+|\theta|\bigr)^{12\mathsf{L}+2}\le \bar c^{\star}\bigl(1+|\theta|^{12\mathsf{L}+2}\bigr).
\end{align}

\begin{proposition}[\hyperlink{A1}{A1} and \hyperlink{A2}{A2} hold for
$\hat u$]\label{prop:nc-A1A2}
Let $\eta>0$ and let $\hat u$ be given by \eqref{eq:nc-ureg}. Then, for a.e.\
$\theta\in\mathbb{R}^{\dpar}$,
\begin{align}\label{eq:nc-brack}
  \nabla^{2}_{\theta}\hat u(\theta)\;\succeq\;
  \Bigl[\bigl(\eta(12\mathsf{L}+5)|\theta|-\bar c^{\star}\bigr)|\theta|^{12\mathsf{L}+2}
  -\bar c^{\star}\Bigr]\Idm .
\end{align}
Consequently:
\begin{itemize}\itemsep2pt
\item[\rm(i)] $\hat u$ satisfies \hyperlink{A1}{A1} with the explicit constant
\begin{align}\label{eq:nc-Lconst}
  L=\bar c^{\star}\left[1+\frac{1}{12\mathsf{L}+3}
  \left(\frac{(12\mathsf{L}+2)\,\bar c^{\star}}
  {(12\mathsf{L}+3)\,\eta\,(12\mathsf{L}+5)}\right)^{12\mathsf{L}+2}\right];
\end{align}
\item[\rm(ii)] for every $\mu>0$, $\nabla^{2}_{\theta}\hat u(\theta)\succeq
\mu\,\Idm$ a.e.\ on $\{|\theta|\ge\varrho(\mu)\}$, where
\begin{align}\label{eq:nc-rho}
  \varrho(\mu)=\max\left\{\frac{2\bar c^{\star}}{\eta(12\mathsf{L}+5)},\;
  \Bigl(\frac{2(\bar c^{\star}+\mu)}{\eta(12\mathsf{L}+5)}\Bigr)^{\!1/(12\mathsf{L}+3)}\right\};
\end{align}
\item[\rm(iii)] $\hat u$ satisfies \hyperlink{A2}{A2} with constant $\mu/2$ and
radius $R=4\varrho(\mu)\bigl(\mu+L\bigr)/\mu$.
\end{itemize}
\end{proposition}

\begin{proof}
Combining Theorem~\ref{thm:nc-main}, \eqref{eq:nc-hesspen} and
\eqref{eq:nc-cbar},
\begin{align*}
  \nabla^{2}_{\theta}\hat u(\theta)
  &\succeq\Bigl[\eta(12\mathsf{L}+5)|\theta|^{12\mathsf{L}+3}
  -c^{\star}\bigl(1+|\theta|\bigr)^{12\mathsf{L}+2}\Bigr]\Idm\\
  &\succeq\Bigl[\eta(12\mathsf{L}+5)|\theta|^{12\mathsf{L}+3}-\bar c^{\star}|\theta|^{12\mathsf{L}+2}
  -\bar c^{\star}\Bigr]\Idm,
\end{align*}
which is \eqref{eq:nc-brack} upon factoring $|\theta|^{12\mathsf{L}+2}$.

(i) Write $q:=12\mathsf{L}+2$, $\kappa:=\eta(12\mathsf{L}+5)$, and
\begin{align*}
  \psi(s):=\kappa s^{q+1}-\bar c^{\star}s^{q}-\bar c^{\star},\qquad s\ge0,
\end{align*}
for the scalar appearing in \eqref{eq:nc-brack}, so that
$\nabla^{2}_{\theta}\hat u\succeq-L\Idm$ a.e.\ with
$L=\sup_{s\ge0}\bigl(-\psi(s)\bigr)$. One calculates
\begin{align*}
  \psi'(s)=-s^{q-1}\bigl(\bar c^{\star}q-\kappa(q+1)s\bigr),
\end{align*}
which is positive on $(0,s_\star)$ and negative on $(s_\star,\infty)$, with
$s_\star:=\bar c^{\star}q/\bigl(\kappa(q+1)\bigr)$. Hence the supremum is
attained at $s_\star$, and since
$\bar c^{\star}-\kappa s_\star=\bar c^{\star}/(q+1)$,
\begin{align*}
  L=-\psi(s_\star)
  =\bar c^{\star}+s_\star^{q}\bigl(\bar c^{\star}-\kappa s_\star\bigr)
  =\bar c^{\star}\Bigl(1+\frac{s_\star^{q}}{q+1}\Bigr),
\end{align*}
which is \eqref{eq:nc-Lconst}. Moreover, since $-\psi(0)=\bar c^{\star}$,
we have that $L\ge\bar c^{\star}>0$.
\eqref{eq:nc-Lconst}

(ii) Suppose first 
$\eta(12\mathsf{L}+5)|\theta|\ge2\bar c^{\star}$, which implies
$\eta(12\mathsf{L}+5)|\theta|^{12\mathsf{L}+3}-\bar c^{\star}|\theta|^{12\mathsf{L}+2}
\ge\tfrac12\eta(12\mathsf{L}+5)|\theta|^{12\mathsf{L}+3}$ and hence
\begin{align*}
  \psi(|\theta|)\ \ge\ \tfrac12\eta(12\mathsf{L}+5)|\theta|^{12\mathsf{L}+3}-\bar c^{\star} .
\end{align*}
Suppose in addition $|\theta|\ge(2(\bar c^{\star}+\mu)/(\eta(12\mathsf{L}+5)))^{1/(12\mathsf{L}+3)}$,
then $\tfrac12\eta(12\mathsf{L}+5)|\theta|^{12\mathsf{L}+3}\ge\bar c^{\star}+\mu$ and therefore
$\psi(|\theta|)\ge\mu$.

(iii) By (i) the potential $\hat u$ is $L$-semi-convex, and by (ii) its
Alexandrov Hessian satisfies $\nabla^{2}\hat u\succeq\mu\,\mathrm{I}$ a.e.\ on
$\{|\theta|\ge\varrho(\mu)\}$. Lemma~\ref{lem:pr-chord} applies with
$\mu_0=\mu$ and $\rho=\varrho(\mu)$ and returns \hyperlink{A2}{A2} with
constant $\mu/2$ and radius $4\varrho(\mu)(\mu+L)/\mu$.
\end{proof}
\clearpage
\section{Table of Constants}
Table~\ref{tab:constants} collects the analytic expressions of the constants used
throughout the paper, together with their order in the dimension-to-temperature ratio.

\begin{center}
\[
\begin{array}{|c|c|l|}
\hline
\textbf{No.} & \textbf{Constants} & \textbf{order}\\
\hline
1 & \displaystyle b=\max\left\{m^2/(2\mu),mR+(L+\mu/2)R^2\right\} & \mathcal{O}(1) \\
\hline
2 & \displaystyle C_1=4\max\left\{2m^2,\mu^2,2K^2\right\} & \mathcal{O}(1) \\
\hline
3 & \displaystyle C_2=p(b+\beta^{-1}(d+p-2)) & \mathcal{O}(d/\beta)\\
\hline
4 & \displaystyle A(p)=1+2(b+\beta^{-1}(d+(p-2)^{+}))/\mu& \mathcal{O}(d/\beta) \\
\hline
5 & \displaystyle \dfrac{C_4(p)}{C_3(p)}\leq A(p)^{\max\{1,\,p/2\}}& \mathcal{O}((d/\beta)^{\max\{1,\,p/2\}})\\
\hline
6 & \displaystyle C_{B_1}=\mathbb{E}|\theta_0|^2+4(1/\mu)(1+b+d\beta^{-1})& \mathcal{O}(d/\beta) \\
\hline
7 & \displaystyle \Lambda_p=p(p+1)2^{3p-3}\sqrt{D_p} & \mathcal{O}(1)\\
\hline
8 & \displaystyle M(p)=2^{3p+1}\Lambda_p^{\,p}  & \mathcal{O}(1)\\
\hline
9 & \displaystyle C_{B_2}=\mathbb{E}|\theta_0|^{2p}+M(p)(1/\mu)^p(1+b+d\beta^{-1})^p & \mathcal{O}((d/\beta)^{p})\\
\hline
10 & \displaystyle C_{B_3}=2^{p-1}(C_{B_2}+1)+A(2p)^{p} & \mathcal{O}((d/\beta)^{p}) \\
\hline
11 & \displaystyle C_{B_4}=\mu C_{B_1}+4+4d\beta^{-1} & \mathcal{O}(d/\beta) \\
\hline
12 & \displaystyle \rho(\beta)\leq \frac{2}{\beta\mu}\exp\!\Big(\beta\big[\,4R(m+KR^{p})+\tfrac{\mu}{2}R^{2}\,\big]\Big) & \mathcal{O}(1)\\
\hline
13 & \displaystyle C_r=\dfrac{2}{\beta\rho(\beta)} & \mathcal{O}(1) \\
\hline
14 & \displaystyle C_{w}=\Big(1+\tfrac{\beta L\rho(\beta)}{2}\Big)^{\frac12+\frac{1}{\beta L\rho(\beta)}}\ \leq\ \sqrt{e\Big(1+\tfrac{\beta L\rho(\beta)}{2}\Big)} & \mathcal{O}(1)\\
\hline
15 & \displaystyle T_0=\dfrac{\ln C_w+1}{C_r}+\lambda_0 & \mathcal{O}(1)\\
\hline
16 & \begin{array}{@{}l@{}}\displaystyle C^2_{E}=e^{5LT_0}\,T_0\Big(\tfrac{2LC_{B_4}}{\sqrt{\mu}}+\tfrac{C_1(1+C_{B_2})}{L}\\ \displaystyle\qquad+4\big(\max{\{2m^2,K^2\}}(1+C_{B_2}+C_{B_3})C_{B_4}\big)^{1/2}\Big)\end{array} & \mathcal{O}((d/\beta)^{p})\\
\hline
17 & \displaystyle C_{D}=\left(2C_{B_1}+2A(2)\right)^{1/2} & \mathcal{O}((d/\beta)^{1/2})\\
\hline
18 & \displaystyle C_{\mathcal{T}_1}=m+\dfrac{2^{p-1}K}{p+1}\left(\sqrt{C_{B_2}}+A(2p)^{p/2}\right) & \mathcal{O}((d/\beta)^{p/2}) \\
\hline
19 & \displaystyle  J=m+K2^{2p-2}(1+(2b/\mu)^{p/2})+K2^{p-1}/(p+1) & \mathcal{O}(1)\\
\hline
20 & \displaystyle   C_{\mathcal{T}_2}=\dfrac{d}{2\beta}\log\left(\dfrac{2(b+d/\beta)\beta^2J^2}{\mu}\right)+\dfrac{1}{2\beta}\log(\pi d)+\dfrac{13}{6\beta} & \mathcal{O}(d/\beta\log(d\beta))\\
\hline
21 & \displaystyle C_{B_1,c}=\mathbb{E}|\theta_0|^2+4(1/\mu)d(1+b+\beta^{-1})& \mathcal{O}(d) \\
\hline
22 & \displaystyle C_{B_2,c}=\mathbb{E}|\theta_0|^{2p}+M(p)(1/\mu)^pd^p(1+b+\beta^{-1})^p & \mathcal{O}(d^{p})\\
\hline
23 & \displaystyle C_{B_4,c}=\mu C_{B_1,c}+4d+4d\beta^{-1} & \mathcal{O}(d)\\
\hline
\end{array}
\]
\captionof{table}{Analytic expressions of constants. The last column reports the order in
powers of $(d/\beta)$, suppressing logarithmic factors and treating
$m,K,L,\mu,R,\lambda_0$ and $\mathbb{E}|\theta_0|^{2p}/(d/\beta)^{p}$ as
$\mathcal{O}(1)$. The constants $\rho(\beta),C_r,C_w,T_0$ are dimension free and
are treated as $\mathcal{O}(1)$ here, but depend exponentially on $\beta$ in
general, see Remark~\ref{remark_beta_poly}. For the coordinate-wise variants
$C_{B_1,c},C_{B_2,c},C_{B_4,c}$ the dependence on $d$ is decoupled from $\beta$.
We write $D_p:=(2p-1)!!$.}\label{tab:constants}
\end{center}
\end{document}